\documentclass{article}

\usepackage{microtype}
\usepackage{graphicx}
\usepackage{subfigure}
\usepackage{booktabs} 

\usepackage{wrapfig}

\usepackage{enumitem}
\usepackage[textsize=tiny]{todonotes}
\presetkeys%
    {todonotes}%
    {inline}
    {}

\usepackage[accepted]{icml2021}

\icmltitlerunning{Non-uniform Analysis for Non-convex Optimization}

\usepackage{amsmath,amsfonts,bm}
\usepackage{amssymb,amsthm}
\usepackage{mathtools}
\usepackage{algorithm,algorithmic}
\usepackage{natbib}
\usepackage{xcolor}
\usepackage{hyperref}
\usepackage{url}
\usepackage{cancel}
\usepackage[capitalize]{cleveref}

\renewcommand{\algorithmicrequire}{\textbf{Input:}}
\renewcommand{\algorithmicensure}{\textbf{Output:}}

\crefname{proposition}{Proposition}{Propositions}
\crefname{theorem}{Theorem}{Theorems}
\crefname{lemma}{Lemma}{Lemmas}
\crefname{update_rule}{Update}{Updates}
\crefname{algorithm}{Algorithm}{Algorithms}
\crefname{figure}{Figure}{Figures}

\def\eqref#1{equation~\ref{#1}}

\def\1{\bm{1}}

\DeclareMathAlphabet{\mathsfit}{\encodingdefault}{\sfdefault}{m}{sl}
\SetMathAlphabet{\mathsfit}{bold}{\encodingdefault}{\sfdefault}{bx}{n}

\def\gA{{\mathcal{A}}}

\def\gD{{\mathcal{D}}}

\def\gL{{\mathcal{L}}}
\def\gM{{\mathcal{M}}}
\def\gN{{\mathcal{N}}}

\def\gP{{\mathcal{P}}}

\def\gR{{\mathcal{R}}}
\def\gS{{\mathcal{S}}}

\def\gW{{\mathcal{W}}}
\def\gX{{\mathcal{X}}}
\def\gY{{\mathcal{Y}}}
\def\gZ{{\mathcal{Z}}}

\def\sR{{\mathbb{R}}}

\newcommand{\R}{\mathbb{R}}

\newcommand{\softmax}{\mathrm{softmax}}
\newcommand{\sigmoid}{\sigma}

\newcommand{\KL}{D_{\mathrm{KL}}}
\newcommand{\Var}{\mathrm{Var}}

\DeclareMathOperator*{\argmax}{arg\,max}
\DeclareMathOperator*{\argmin}{arg\,min}

\DeclareMathOperator{\sign}{sign}

\newtheorem{theorem}{Theorem}
\newtheorem{lemma}{Lemma}
\newtheorem{definition}{Definition}
\newtheorem{proposition}{Proposition}
\newtheorem{remark}{Remark}

\newtheorem{assumption}{Assumption}

\newtheorem{claim}[theorem]{Claim}

\DeclareMathOperator*{\expectation}{\mathbb{E}}

\def\rvone{{\mathbf{1}}}
\def\rvzero{{\mathbf{0}}}

\def\identitymatrix{\mathbf{Id}}
\def\diagonalmatrix{\text{diag}}

\DeclareMathOperator*{\probability}{Pr}

\newcommand{\cS}{\mathcal{S}}

\begin{document}

\twocolumn[
\icmltitle{Leveraging Non-uniformity in First-order Non-convex Optimization}



\icmlsetsymbol{equal}{*}

\begin{icmlauthorlist}
\icmlauthor{Jincheng Mei}{ualberta,google_brain,equal}
\icmlauthor{Yue Gao}{ualberta,equal}
\icmlauthor{Bo Dai}{google_brain}
\icmlauthor{Csaba Szepesv{\'a}ri}{deepmind,ualberta}
\icmlauthor{Dale Schuurmans}{google_brain,ualberta}
\end{icmlauthorlist}

\icmlaffiliation{ualberta}{University of Alberta}
\icmlaffiliation{deepmind}{DeepMind}
\icmlaffiliation{google_brain}{Google Research, Brain Team}

\icmlcorrespondingauthor{Jincheng Mei}{jmei2@ualberta.ca} \icmlcorrespondingauthor{Yue Gao}{gao12@ualberta.ca}

\icmlkeywords{Machine Learning, ICML}

\vskip 0.3in
]



\printAffiliationsAndNotice{\textsuperscript{*}Equal contribution} 

\begin{abstract}
Classical global convergence results for first-order methods 
rely on uniform smoothness and the
\L{}ojasiewicz inequality.
Motivated by properties of objective functions that arise in machine learning,
we propose a non-uniform refinement of these notions, leading to
\emph{Non-uniform Smoothness} (NS)
and 
\emph{Non-uniform \L{}ojasiewicz inequality} (N\L{}).
The new definitions
inspire new 
geometry-aware first-order methods that are able to
converge to global optimality faster than
the classical $\Omega(1/t^2)$ lower bounds.
%
To illustrate the power of these geometry-aware methods
and their corresponding non-uniform analysis,
we consider two important problems in machine learning:
policy gradient optimization in reinforcement learning (PG),
and generalized linear model training in supervised learning (GLM).
For PG, 
we find that normalizing the gradient ascent method
can accelerate convergence to $O(e^{- c \cdot t})$ (where $c > 0$)
while incurring less overhead than existing algorithms.
For GLM, 
we show that geometry-aware normalized gradient descent can also achieve a linear
convergence rate, which significantly improves the best known results.
We additionally show that the proposed geometry-aware gradient descent methods
escape landscape plateaus faster than standard gradient descent.
Experimental results are used to illustrate and complement the theoretical findings.

\end{abstract}

\setlength{\abovedisplayskip}{3pt}
\setlength{\abovedisplayshortskip}{3pt}
\setlength{\belowdisplayskip}{3pt}
\setlength{\belowdisplayshortskip}{3pt}
\setlength{\jot}{2pt}

\setlength{\floatsep}{2ex}
\setlength{\textfloatsep}{2ex}


\section{Introduction}
\label{sec:introduction}

The optimization of non-convex objective functions
is a topic of key interest in modern-day machine learning.
Recent, intriguing results show that
simple gradient-based optimization can achieve \emph{globally} optimal solutions in certain 
non-convex problems arising in machine learning,
such as in reinforcement learning~(RL) \citep{agarwal2020optimality}, supervised learning~(SL) \citep{hazan2015beyond}, and deep learning \citep{allen2019convergence}.
While gradient-based algorithms remain the method of choice in machine learning, 
the convergence of such algorithms to global minimizers has still only been established in restrictive settings where one can assert two strong assumptions about the objective function:
(1) that the objective is smooth,
and (2) that the objective satisfies a gradient dominance over sub-optimality such as the \L{}ojasiewicz inequality. We will find it beneficial to recall the definitions of these properties. For the remainder of this paper let $\Theta = \R^d$. 

\begin{definition}[Smoothness]
\label{def:smoothness}
The function $f: \Theta \to \sR$ is $\beta$-smooth ($\beta > 0$) if it is differentiable and 
for all $\theta,\theta^\prime \in \Theta$,
\begin{align}
    \left| f(\theta^\prime) - f(\theta) - \Big\langle \frac{d f(\theta)}{d \theta}, \theta^\prime - \theta \Big\rangle \right| \le \frac{\beta}{2} \cdot \| \theta^\prime - \theta \|_2^2.
    \label{eq:uniform smoothness}
\end{align}
\end{definition}

\begin{definition}[]
\citep{lojasiewicz1963propriete,polyak1963gradient,kurdyka1998gradients} 
\label{def:lojasiewicz}
The differentiable function $f: \Theta \to \sR$ satisfies the $(C, \xi)$-\L{}ojasiewicz inequality if for all $\theta \in  \Theta$,
\begin{align}
    \left\| \frac{d f(\theta)}{d \theta} \right\|_2 \ge C \cdot \Big( f(\theta) - \inf_{\theta \in \Theta}{ f(\theta) } \Big)^{1 - \xi},
    \label{eq:lojasiewicz}
\end{align}
where $C > 0$ and $\xi \in [0, 1]$.
\end{definition}

In particular, if an objective function $f$ satisfies both assumptions, gradient-based optimization can be shown to converge to a global minimizer by noting first that uniform smoothness \cref{eq:uniform smoothness} ensures the gradient updates achieve monotonic improvement with an appropriate step size (i.e., $f(\theta_{t+1}) \le f(\theta_t) - \frac{1}{2 \beta} \cdot \left\| \nabla f(\theta_t) \right\|_2^2$, if $\theta_{t+1} \gets \theta_t - \frac{1}{\beta} \cdot \nabla f(\theta_t)$),
while the \L{}ojasiewicz inequality \cref{eq:lojasiewicz} ensures the gradient does not vanish before a global minimizer is reached.
Several global convergence results have recently been achieved in the machine learning literature by exploiting assumptions of this kind. 
For example, in reinforcement learning it has recently been shown that policy gradient (PG) methods converge to a globally optimal policy \citep{agarwal2020optimality,mei2020global};
in supervised learning it has been shown that gradient descent (GD) methods converge to global minimizers of certain non-convex problems \citep{hazan2015beyond};
and in deep learning theory it has been shown that (stochastic) GD can converge to a global minimizer with an over-parameterized neural network \citep{allen2019convergence}. 

However, previous work that relies on the two \emph{uniform} conditions in \cref{def:smoothness,def:lojasiewicz} assumes \emph{universal constants} $\beta$ and $C$, which ignores important problem structure and limits both the applicability of the results and the strength of the results that can be obtained.


In this paper, we expand the class of problems for which gradient-based optimization is globally convergent, 
develop novel gradient-based methods that better exploit local structure,
and improve the convergence rate analysis. 
We achieve these results by first defining then investigating a new set of \emph{non-uniform}
%
smoothness and \L{o}jasiewicz inequalities,
which generalize the classical definitions and allow a refined characterization of the space of objectives.
Given these refined notions,
we then tailor novel gradient-based algorithms that improve previous methods for these new problem classes,
and extend the analysis to exploit these new forms of non-uniformity, achieving significantly stronger convergence rates in many cases.
Importantly, these improvements are achieved in non-convex optimization problems that arise in relevant machine learning problems.

The remainder of the paper is organized as follows.
First, in \cref{sec:motivation} we illustrate how natural optimization problems, including those in machine learning, exhibit interesting \emph{local} structure that cannot be adequately captured by the
uniform smoothness and \L{}ojasiewicz inequalities.
Then, \cref{sec:non_uniform_propterties} introduces the 
the Non-uniform Smoothness (NS) property and the Non-uniform \L{}ojasiewicz (N\L{}) inequality,
based on which \cref{sec:non_uniform_analysis} provides non-uniform analyses. \cref{sec:policy_gradient,sec:generalized_linear_model} then present new results for policy gradient and generalized linear models respectively. \cref{sec:conclusions_future_work} concludes the paper and discusses some future directions. 
Due to space limits, we relegate most of the proofs to the appendix. 

\section{Motivation}
\label{sec:motivation}

To illustrate the significance of non-uniformity in 
machine learning problems,
we consider examples motivated by recent theoretical \citep{zhang2019gradient,wilson2019accelerating,mei2020global} and empirical studies \citep{cohen2021gradient}.

Regarding smoothness, it is clear that a uniform smoothness constant $\beta$ cannot always adequately characterize an objective over its entire domain.
For example, the convex function $f:x\mapsto x^4$ cannot be informatively characterized by a uniform smoothness constant $\beta$
because its Hessian $f^{\prime\prime}:x \mapsto 12 \cdot x^2$ has the property that
$f^{\prime\prime}(x)\to\infty$ as $|x|\to\infty$, and $f^{\prime\prime}(x)\to0$ as $|x|\to0$. 
Varying smoothness of this kind has motivated the study of alternative definitions to explain, for example, the effectiveness of gradient clipping in training neural networks and normalization in optimization \citep{zhang2019gradient,wilson2019accelerating}.
Meanwhile \citet{cohen2021gradient} present neural network training results that cannot be well explained using the standard smoothness condition of \cref{def:smoothness}. 

Regarding the \L{}ojasiewicz inequality, a recent study of policy gradient optimization in reinforcement learning has shown that,
with the standard softmax parameterization, the expected return objective cannot satisfy \emph{any} \L{}ojasiewicz inequality with a universal constant $C$ \citep{mei2020global}, which 
removes the possibility of using \cref{def:lojasiewicz} to prove convergence.
By introducing a \emph{non-uniform} version of the \L{ojasiewicz} inequality, the same authors were able to show a global convergence rate for the same problem.

\begin{figure}[ht]
\centering
\includegraphics[width=1.0\linewidth]{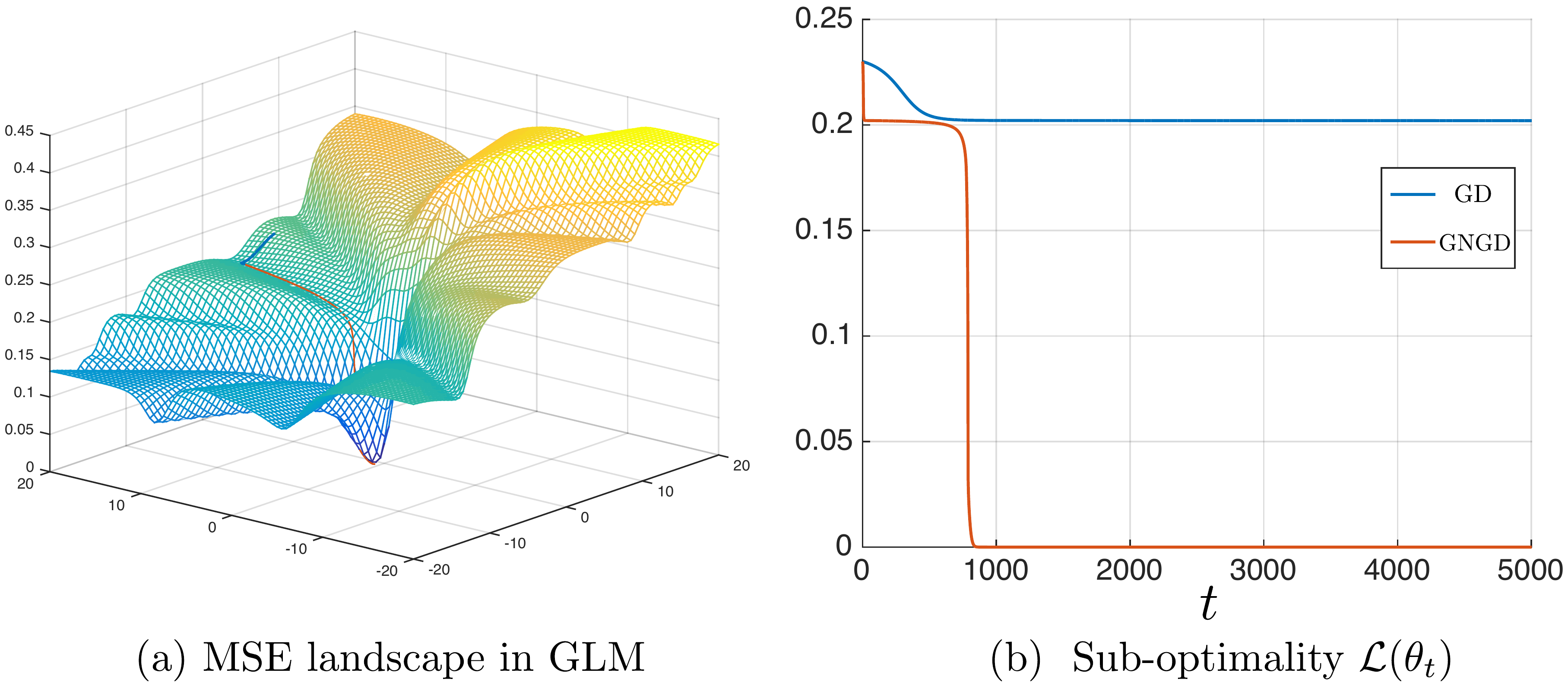}
\caption{Non-uniform landscape of non-convex function.}
\label{fig:example_glm_sigmoid}
\end{figure}

\cref{fig:example_glm_sigmoid} illustrates another example of a  non-convex objective, which arises in supervised learning. Subfigure 1(a) visualizes the mean squared error (MSE) of a generalized linear model (GLM) \citep{hazan2015beyond}, which is not only non-convex but also highly non-uniform. 
As a ``teaser'', subfigure 1(b) compares the convergence behavior of two algorithms: standard gradient descent (GD), which suffers from slow convergence on the plateaus due to the non-uniformity of the objective, and an alternative algorithm (GNGD), soon to be introduced. 
This figure previews how proper handling of non-uniformity in the optimization landscape can enable significant acceleration of optimization progress, including
a quick escape from plateaus.

\section{Non-uniform Properties and Algorithms}
\label{sec:non_uniform_propterties}

The main results in this paper depend on
two core concepts, Non-uniform Smoothness (NS) and Non-uniform \L{}ojasiewicz (N\L{}) inequality.
The NS property is a new, intuitive generalization of smoothness. 
The N\L{} inequality is a recent proposal of \citet{mei2020global} and generalizes previous \L{}ojasiewicz inequalities.
Our key contribution is to show that the \emph{combination} of these two non-uniform concepts is particularly powerful, applicable to important non-convex objectives in machine learning, and allows the development of improved algorithms and analysis. 

\subsection{Non-uniform Smoothness}

The first main concept we leverage is a generalized notion of smoothness that depends on the parameters non-uniformly.

\begin{definition}[Non-uniform Smoothness (NS)]
\label{def:non_uniform_smoothness}
The function $f: \Theta \to \sR$ satisfies $\beta(\theta)$ non-uniform smoothness if $f$ is differentiable and for all 
$\theta,\theta'\in\Theta$,
\begin{align*}
    \left| f(\theta^\prime) - f(\theta) - \Big\langle \frac{d f(\theta)}{d \theta}, \theta^\prime - \theta \Big\rangle \right| \le \frac{{\color{red}\beta(\theta)}}{2} \cdot \| \theta^\prime - \theta \|_2^2,
\end{align*}
where $\beta$ is a positive valued function: $\beta: \Theta \to (0,\infty)$.
\end{definition}

We will refer to $\beta(\theta)$ in \cref{def:non_uniform_smoothness}  as the \textit{NS coefficient}. 
This alternative definition generalizes and unifies several smoothness concepts from the recent literature.
First, NS clearly reduces to \cref{eq:uniform smoothness} with $\beta(\theta) = \beta$. 
However, NS also generalizes the notion of $(L_0, L_1)$ smoothness from \citet{zhang2019gradient} by using $\beta(\theta) = L_0 + L_1 \cdot \left\| \nabla f(\theta) \right\|_2$. 
By using $\beta(\theta) = c \cdot \left\| \nabla f(\theta) \right\|_2^{\frac{p-2}{p-1}}$, NS also reduces to the notion of strong smoothness of order $p$ proposed in \citet{wilson2019accelerating}. 
Finally, with $\beta(\theta) = c / \left\| \theta \right\|_p^2$, NS reduces to a special form of non-uniform smoothness considered in \citet{mei2020escaping}. 
We will show later that NS also covers other previously unstudied smoothness variants. 
Below we will demonstrate the key benefits of \cref{def:non_uniform_smoothness} in terms of its \emph{generality}, \emph{better convergence results}, and \emph{practical implications}
in conjunction with the N\L{} inequality.

\subsection{Non-uniform \L{}ojasiewicz Inequality}

The second main concept we leverage is a generalized \L{}ojasiewicz inequality introduced by \citet{mei2020global}:

\begin{definition}[Non-uniform \L{}ojasiewicz (N\L{})]
\label{def:non_uniform_lojasiewicz}
The differentiable function $f : \Theta \to \sR$ satisfies the $(C(\theta), \xi)$ non-uniform \L{}ojasiewicz inequality if for all $\theta \in \Theta$,
\begin{align}
    \left\| \frac{ d f(\theta) }{d \theta} \right\|_2 \ge {\color{red}C(\theta)} \cdot \left| f(\theta) - f(\theta^*) \right|^{1-\xi},
\end{align}
where $\xi \in (-\infty, 1]$, and $C(\theta): \Theta \to \sR > 0$ holds for all $\theta \in \Theta$.
In this definition, either $\theta^*=\arg\min_{\theta\in\Theta}f(\theta)$,
or
$f(\theta^*)$ is replaced with $\inf_{\theta}{ f(\theta)}$ 
if the global optimum is not achieved within the domain $\Theta$. 
\end{definition}
\cref{def:non_uniform_lojasiewicz} extends the classical ``uniform'' \L{}ojasiewicz inequalities in optimization literature, such as the Polyak-\L{}ojasiewicz (P\L{}) inequality with $C(\theta) = C > 0$ and $\xi = 1/2$ \citep{lojasiewicz1963propriete,polyak1963gradient}; and the Kurdyka-\L{}ojasiewicz (K\L{}) inequality\footnote{The K\L{} inequality is violated at bad local optima, since vanishing gradient norm cannot dominate non-zero sub-optimality gap. Therefore \cref{def:non_uniform_lojasiewicz} actually recovers global K\L{} inequality.} by setting $C(\theta) = C > 0$ \citep{kurdyka1998gradients}.
Following \citep{mei2020global,mei2020escaping}, we refer to $\xi$ as the \textit{N\L{} degree} and $C(\theta)$ as the \textit{N\L{} coefficient}.
Generally speaking, a larger N\L{} degree $\xi$ and N\L{} coefficient $C(\theta)$ indicate faster convergence for gradient based algorithms. 
\cref{sec:examples_non_convex_nl_inequality} provides an overview of remarkable non-convex functions that satisfy the N\L{} inequality for various $\xi$ and $C(\theta)$.
As stated, our main contribution in this paper is to show how,
when \emph{combined} with NS, N\L{} becomes a powerful tool for both algorithm design and analysis, which is a novel direction of investigation.



\subsection{Geometry-aware Gradient Descent}

A key benefit of the non-uniform definitions is that we can introduce stepsize 
rules that make gradient descent adapt to the local ``geometry'' of the optimization objective.
First consider the classical gradient decent update.

\begin{definition}[Gradient Descent (GD)]
\label{def:gd}
\begin{align}
\theta_{t+1} \gets \theta_t - \eta \cdot \nabla f(\theta_t).
\end{align}
\end{definition}

The key challenge with deploying GD is choosing the step size $\eta$;
if $\eta$ is too large, instability ensues,
if too small, progress becomes slow.
Recall from the introduction that $\eta=1/\beta$ is a canonical choice for assuring convergence in \emph{uniformly} $\beta$ smooth objectives.
This suggests that in the presence of \emph{non-uniform} smoothness $\beta(\theta)$ given in NS,
the stepsize should be adapted to $1/\beta(\theta)$. This leads to
 a new variant of normalized gradient descent.

\begin{definition}[Geometry-aware Normalized GD (GNGD)]
\label{def:gngd}
\begin{align}
\theta_{t+1} \gets \theta_t - \eta \cdot \frac{ \nabla f(\theta_t) }{ \color{red}{\beta(\theta_t)} }.
\end{align}
\end{definition}

Key to making this approach practical will be efficient ways to measure (or bound) $\beta(\theta)$.
Below we will show how in the context of NS and N\L{} properties, GNGD can be made both practical and extremely efficient at solving various 
global optimization problems in machine learning.
These results also broaden our fundamental knowledge of the set of objectives that admit efficient global optimization.


\section{Non-uniform Analysis}
\label{sec:non_uniform_analysis}

\begin{figure*}[ht]
\centering
\includegraphics[width=0.8\linewidth]{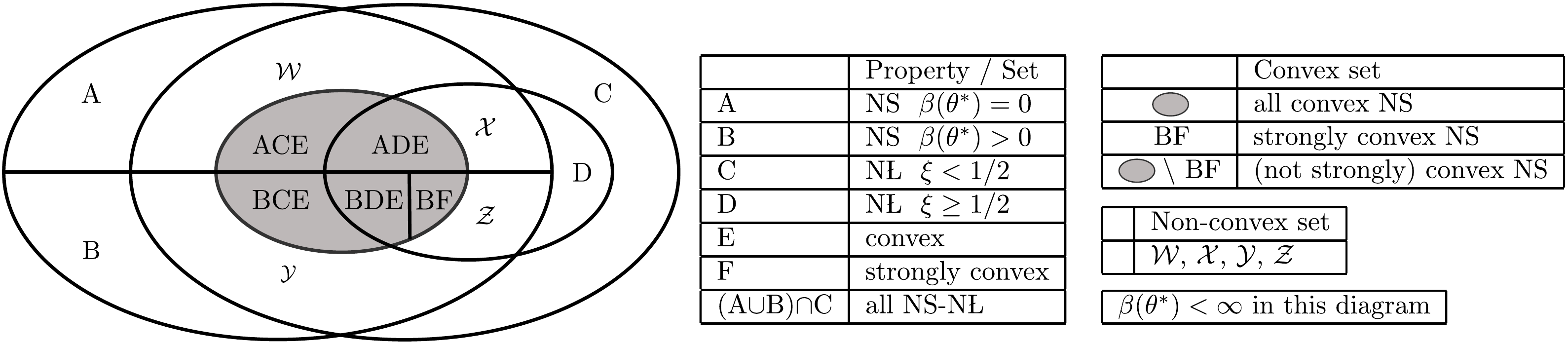}
\caption{Different function classes for $\beta(\theta^*) < \infty$. We use a label notation where, e.g., C denotes the set of all functions that satisfy property C, and $\text{ACE} \coloneqq \text{A} \cap \text{C} \cap \text{E}$. The two largest ellipsoids correspond to $\text{A} \cup \text{B}$ and C. We study the following four non-convex function classes in $( \text{A} \cup \text{B} ) \cap C$, i.e., $\gW \coloneqq \text{AC} \setminus ( \text{AD} \cup \text{ACE} )$, $\gX \coloneqq \text{AD} \setminus \text{ADE}$, $\gY \coloneqq \text{BC} \setminus ( \text{BD} \cup \text{BCE} )$, and $\gZ \coloneqq \text{BD} \setminus ( \text{BDE} \cup \text{BF} )$.
} 
\label{fig:non_uniform_function_classes_diagram_table}
\end{figure*}

Our first main contribution is an analysis for GD and GNGD based in the presence of non-uniform properties. For minimization problems, we assume $\inf_{\theta}{ f(\theta) } > - \infty$ ($\sup_{\theta}{ f(\theta) } < \infty$ for maximization problems).
\begin{theorem}
\label{thm:general_optimization_main_result_1}
Suppose $f: \Theta \to \sR$ satisfies NS with $\beta(\theta)$ and 
the N\L{} inequality with $(C(\theta), \xi)$. 
Suppose $C \coloneqq \inf_{t \ge 1}{ C(\theta_t) } > 0$ for GD and GNGD. Let $\delta(\theta) \coloneqq f(\theta) - f(\theta^*)$ be the sub-optimality gap. The following hold:

\begin{description}[topsep=0pt,parsep=0pt]

\item[(1a)] if $\beta(\theta) \le c \cdot  \delta(\theta)^{1-2\xi}$ with $\xi \in (- \infty, 1/2)$, then the conclusions of (1b) hold;
\item[(1b)] if $\beta(\theta) \le c \cdot \left\| \nabla f(\theta) \right\|_2^{\frac{1 - 2 \xi}{1 - \xi}}$ with $\xi \in (- \infty, 1/2)$, then GD with $\eta \in O(1)$ achieves $\delta(\theta_t) \in  \Theta(1/t^{\frac{1}{1 - 2 \xi}})$,
and GNGD achieves $\delta(\theta_t) \in O(e^{- t})$.
%
\item[(2a)] if $\beta(\theta) \le L_0 + L_1 \cdot \left\| \nabla f(\theta) \right\|_2$, then  the conclusions of (2b) hold;
\item[(2b)] if $\beta(\theta) \le L_0 \cdot \frac{\left\| \nabla f(\theta) \right\|^2}{ \delta(\theta)^{2 - 2 \xi} } + L_1 \cdot \left\| \nabla f(\theta) \right\|_2$, then GD and GNGD both achieve $\delta(\theta_t) \in O(1/t^{\frac{1}{1 - 2 \xi}})$ when $\xi \in (- \infty, 1/2)$, and $O(e^{-t})$ when $\xi = 1/2$. GNGD has strictly better constant than GD ($1 > C > C^2$).
%
\item[(3a)] if $\beta(\theta)\!\le\!c\!\cdot\!\left\| \nabla f(\theta) \right\|_2^{\frac{1 - 2 \xi}{1 - \xi}}\!$ with $\xi \in (1/2, 1)$, then  the conclusions of (3b) hold;
\item[(3b)] if $\beta(\theta) \le c \cdot \delta(\theta)^{1 - 2 \xi}$ with $\xi \in (1/2, 1)$, then GD with $\eta \in \Theta(1)$ does not converge, while GNGD achieves $\delta(\theta_t) \in O(e^{-t})$.
\end{description}
\end{theorem}

\begin{remark}
The cases (1)-(3) cover all three possibilities of $\beta(\theta^*)$. Since $\theta^*$ is the global minimum, $\nabla^2 f(\theta^*)$ is positive semi-definite (negative if $\theta^*$ is maximum) if it exists. 
\begin{itemize}[noitemsep, topsep=0pt,parsep=0pt]
	\item[(1)] If $\nabla^2 f(\theta^*) = \rvzero$, then $\beta(\theta) \to 0$ as $\theta, \theta^\prime \to \theta^*$, which means the landscape around $\theta^*$ is flat; 
	\item[(2)] If $\nabla^2 f(\theta^*)$ has at least one strictly positive (negative) eigenvalue, then $\beta(\theta) \to \beta > 0$ as $\theta, \theta^\prime \to \theta^*$. 
\end{itemize}
The cases (1)-(2) also cover the situations where the Hessian $\nabla^2 f(\theta^*)$ does not exist but one can find a finite $\beta(\theta^*) > 0$ to upper bound the l.h.s. of \cref{def:non_uniform_smoothness}. 
\begin{itemize}[noitemsep, topsep=0pt,parsep=0pt]
	\item[(3)] The case (3) is for blow-up type non-existence of $\nabla^2 f(\theta^*)$, where $\beta(\theta^*)$ is unbounded.
\end{itemize}
\end{remark}
\begin{remark}
In \cref{thm:general_optimization_main_result_1}, $C \coloneqq \inf_{t \ge 1}{ C(\theta_t) }$ is related to the early optimization and
plateau escaping behavior studied in \citet{mei2020escaping}. It remains open to study whether GNGD can be combined with different parameterizations in \citet{mei2020escaping} to further improve $C$.
\end{remark}

Note that
(1b) recovers the strong smoothness of order $p$ with $p = 1 / \xi$ in \citet{wilson2019accelerating}, 
and (2a) recovers the $(L_0, L_1)$ smoothness of \citet{zhang2019gradient}. 
The results here consider more general N\L{} functions and establish faster rates of convergence. 
The other cases have not been studied in literature to our knowledge.
%
In \cref{sec:policy_gradient,sec:generalized_linear_model} below we study practical machine learning examples that are covered by cases (1) and (2) in \cref{thm:general_optimization_main_result_1}. 
Other cases of different $\beta(\theta)$ and $\xi$ are discussed in \cref{sec:proofs_non_uniform_analysis} for completeness.  

\subsection{Function Classes and Existing Lower Bounds}

Before applying these results to problems in machine learning, we first provide a refined characterization of function classes organized by their NS and N\L{} properties. 
This also clarifies the relation between the non-uniform properties and standard notions of convexity and smoothness; see \cref{fig:non_uniform_function_classes_diagram_table}. 

\begin{proposition}
\label{prop:function_class_relation}
The following 
hold for an objective $f$: 
\begin{itemize}[leftmargin=*,noitemsep, topsep=0pt,parsep=0pt]
\item[(1)] $\text{D} \subseteq \text{C}$. 
If 
$f$ satisfies N\L{} with degree $\xi$, it satisfies N\L{} with degree $\xi^\prime < \xi$;
\item[(2)] $\text{F}\!\subseteq\!\text{D}$. 
A strongly convex 
$f$ satisfies N\L{} with $\xi\!=\!1/2$; 
\item[(3)] $\text{F} \cap \text{A} = \emptyset$. 
A strongly convex 
$f$ cannot satisfy NS with $\beta(\theta) \to 0$ as $\theta, \theta^\prime \to \theta^*$;
\item[(4)] $\text{E}\!\subseteq\!\text{C}$. 
A (not strongly) convex 
$f$ satisfies N\L{} with $\xi\!=\!0$.
\end{itemize}
\end{proposition}

The next proposition provides concrete examples for each convex function class in $(A \cup B) \cap C$ in \cref{fig:non_uniform_function_classes_diagram_table}.

\begin{proposition}
\label{prop:non_empty_convex_function_classes}
The following results hold: (1) $\text{ACE} \not= \emptyset$. (2) $\text{ADE} \not= \emptyset$. (3) $\text{BCE} \not= \emptyset$. (4) $\text{BDE} \not= \emptyset$. (5) $\text{BF} \not= \emptyset$.
\end{proposition}

A more interesting result considers examples in the classes of non-convex functions $(A \cup B) \cap C$ in \cref{fig:non_uniform_function_classes_diagram_table}. 
The non-uniform analysis above largely still applies to these 
problems, even when standard convex analysis cannot apply.
\begin{proposition}
\label{prop:non_empty_non_convex_function_classes}
The following results hold: (1) $\gW \coloneqq \text{AC} \setminus ( \text{AD} \cup \text{ACE} ) \not= \emptyset$.  (2) $\gX \coloneqq \text{AD} \setminus \text{ADE} \not= \emptyset$. (3) $\gY \coloneqq \text{BC} \setminus ( \text{BD} \cup \text{BCE} ) \not= \emptyset$. (4) $\gZ \coloneqq \text{BD} \setminus ( \text{BDE} \cup \text{BF} ) \not= \emptyset$.
\end{proposition}
We next apply the techniques to a class of convex functions, 
achieving
results that cannot be explained by classical convex-smooth analysis.
\begin{proposition}
\label{prop:absulte_power_p}
    The convex function $f: x \mapsto |x|^p$ with $p > 1$ satisfies the N\L{} inequality with $\xi = 1/p$ and the NS property with $\beta(x) \le c_1 \cdot \delta(x)^{1 - 2 \xi}$.
\end{proposition}

Consider any $p > 2$, such as $p = 4$,
where it follows that $f$ satisfies N\L{} with degree $\xi = 1/4 < 1/2$. 
According to (1a) in \cref{thm:general_optimization_main_result_1}, GD will achieve $\delta(x_t) \in \Theta(1/t^2)$, while GNGD attains $\delta(x_t)\in O(e^{-c \cdot t})$ (where $c > 0$). 
Note that standard convex analysis can only give a $O(1/t)$ rate on (not strongly) convex smooth functions. 
The $\Theta(1/t^2)$ rate for GD here follows from using N\L{} degree $\xi = 1/4$, which improves on $\xi = 0$ from mere convexity 
((4) in \cref{prop:function_class_relation}). 
The $O(e^{-c \cdot t})$ rate has also been observed for this example by exploiting strong smoothness ((1b), as noted) \citep{wilson2019accelerating}.  \cref{fig:non_uniform_function_classes_diagram_table} provides a more general understanding of when this happens.

\paragraph{$\Omega(1/t^2)$ lower bound for convexity-smoothness.} 
Note that
GNGD satisfies $x_{t+1} = x_1 - \sum_{i=1}^{t}{ \frac{\eta}{\beta(x_i)} \cdot \nabla f(x_i) } \in \text{Span} \left\{x_1, \nabla f(x_1), \dots, \nabla f(x_t) \right\}$, which is a first-order oracle \citep{nesterov2003introductory}. 
Thus
there exists a worst-case objective in the convex-smooth class 
that forces
$\delta(x_t)\in\Omega(1/t^2)$ for 
$t \in O(n)$, where $n$ is the parameter dimension \citep{nemirovski1983problem,nesterov2003introductory,bubeck2014convex}. 
This is not a contradiction, since the
lower bound is established by constructing
a convex smooth function with a \emph{constant} $\beta > 0$ \citep{bubeck2014convex}, and $\beta(x) \to \beta > 0$ as $x, x^\prime \to x^*$ in \cref{def:non_uniform_smoothness}. Hence, the $\Omega(1/t^2)$ result covers \textit{some} functions in BCE in \cref{fig:non_uniform_function_classes_diagram_table}. 
Meanwhile $f: x \mapsto |x|^p$ with $p > 2$ satisfies $\beta(x) \to 0$ as $x, x^\prime \to 0$ in \cref{def:non_uniform_smoothness} (ACE in \cref{fig:non_uniform_function_classes_diagram_table}),
which implies that the standard convex-smooth class consists of two subclasses. 
One
subclass (BCE) admits first-order sub-linear lower bounds, while the other 
(ACE) allows linear convergence using first-order methods. 
This illustrates the \emph{necessity} of non-uniformity in 
subdividing the NS 
class as $\text{A} \cup \text{B}$ in \cref{fig:non_uniform_function_classes_diagram_table}.
This partition also inspires 
geometry-aware 
GD.

\paragraph{$\Omega(1/\sqrt{t})$ lower bound for $(L_0, L_1)$-smoothness.} 
For $\beta(\theta) = L_0 + L_1 \cdot \| \nabla f(\theta) \|_2$ ((2a) in \cref{thm:general_optimization_main_result_1}) with $L_0, L_1 \ge 1$, standard normalized GD is subject to a $\Omega(1/\sqrt{t})$ lower bound \citep{zhang2019gradient}. However, in \cref{sec:policy_gradient}, we will show that normalized policy gradient (PG) method achieves a linear rate of $O(e^{- c \cdot t})$. 
Again, this is not a contradiction for similar reasons.
With $L_0 \ge 1$, $\beta(\theta) \to L_0 > 0$ as $\theta, \theta^\prime \to \theta^*$,  
the $\Omega(1/\sqrt{t})$ lower bound will hold for \textit{some} functions in $\text{BCE} \cup \gY$ in \cref{fig:non_uniform_function_classes_diagram_table}.  While in \cref{sec:policy_gradient} the 
objective satisfies $L_0 = 0$ and $L_1 > 0$, hence $\beta(\theta) \to 0$ as $\theta, \theta^\prime \to \theta^*$ ($\text{ACE} \cup \gW$ in \cref{fig:non_uniform_function_classes_diagram_table}). 
This shows a similar separation of rates for first-order methods will also occur based on NS conditions. 
Furthermore, in \cref{sec:generalized_linear_model}, we will show that both GD and GNGD achieve a $O(e^{- c \cdot t})$ rate for GLM,
but here the objective is in $\gZ$ in \cref{fig:non_uniform_function_classes_diagram_table} so the lower bounds do not apply.

\paragraph{Unbounded Hessian.} 
Consider any $p \in (1, 2)$, such as $p = 3/2$
where $f$ satisfies $\xi = 2 / 3$. 
According to \cref{thm:general_optimization_main_result_1}(3a), GD diverges since the Hessian is unbounded near $0$. 
This makes it \textit{necessary} to introduce geometry-aware normalization to ensure convergence, which is verified in \cref{sec:additional_simulations_experiments}.
This has practical implications for RL, 
for example ensuring exploration using state distribution entropy, which has unbounded Hessian near probability simplex boundary \citep{hazan2019provably} and alternative probability transforms \citep{mei2020escaping}.


\section{Policy Gradient}
\label{sec:policy_gradient}

Our second main contribution is to show that the expected return objective considered in direct policy optimization in RL 
falls under the 
function class $\gW$ in \cref{fig:non_uniform_function_classes_diagram_table}, in particular satisfying case (1) of   \cref{thm:general_optimization_main_result_1} with N\L{} degree $\xi = 0$. The key point is that value functions in Markov decision processes (MDPs) satisfy NS properties with coefficient being the PG norm (\cref{lem:non_uniform_smoothness_softmax_special,lem:non_uniform_smoothness_softmax_general}). This novel finding not only provides a much more precise characterization than existing standard smoothness results \citep{agarwal2020optimality,mei2020global}, but also enables PG with normalization to use the N\L{} inequalities (\cref{lem:non_uniform_lojasiewicz_softmax_special,lem:non_uniform_lojasiewicz_softmax_general}) differently than for standard PG ($\left\| \nabla f(\theta) \right\|_2$ vs. $\left\| \nabla f(\theta) \right\|_2^2$), which leads to faster convergence as well as plateau escaping.

\subsection{RL Settings and Notations}

For a finite set $\gN$, let $\Delta(\gN)$ denote the set of all probability distributions on $\gN$. A finite MDP $\gM \coloneqq (\gS, \gA, \gP, r, \gamma)$ is determined by a finite state space $\gS$, a finite action space $\gA$, a transition function $\gP: \gS \times \gA \to \Delta(\gS)$, a scalar reward function $r: \gS \times \gA \to \sR$, and a discount factor $\gamma \in [0 , 1)$. 

In policy-based RL, an agent interacts with the environment, i.e., the MDP $\gM$, using a policy $\pi: \gS \to \Delta(\gA)$. Given a state $s_t$, the agent takes an action $a_t \sim \pi(\cdot | s_t)$, receives a one-step scalar reward $r(s_t, a_t)$ and a next-state $s_{t+1} \sim \gP( \cdot | s_t, a_t)$. The long-term expected reward, also known as the value function of $\pi$ under $s$, is defined as
\begin{align}
\label{eq:state_value_function}
    V^\pi(s) \coloneqq \expectation_{\substack{s_0 = s, a_t \sim \pi(\cdot | s_t), \\ s_{t+1} \sim \gP( \cdot | s_t, a_t)}}{\left[ \sum_{t=0}^{\infty}{\gamma^t r(s_t, a_t)} \right]}.
\end{align}
The state-action value of $\pi$ at $(s,a) \in \gS \times \gA$ is defined as $Q^\pi(s, a) \coloneqq r(s, a) + \gamma \sum_{s^\prime}{ \gP( s^\prime | s, a) V^\pi(s^\prime) }$, and $A^\pi(s,a) \coloneqq Q^\pi(s, a) - V^\pi(s)$ is the advantage function of $\pi$.
The state distribution of $\pi$ is defined as,
\begin{align}
     d_{s_0}^{\pi}(s) \coloneqq (1 - \gamma) \sum_{t=0}^{\infty}{ \gamma^t \probability(s_t = s | s_0, \pi, \gP) }.
\end{align}
Given an initial state distribution $\rho \in \Delta(\gS)$, we denote $V^\pi(\rho) \coloneqq \expectation_{s \sim \rho}{ \left[ V^\pi(s) \right]}$ and $d_{\rho}^{\pi}(s) \coloneqq \expectation_{s_0 \sim \rho}{\left[ d_{s_0}^{\pi}(s) \right]}$. There exists an optimal policy $\pi^*$ such that $V^{\pi^*}(\rho) = \sup_{\pi : \gS \to \Delta(\gA) }{ V^{\pi}(\rho)}$. For convenience, we denote $V^* \coloneqq V^{\pi^*}$. 
Consider a tabular representation, i.e., $\theta(s,a) \in \sR$ for all $(s,a)$, so that the policy $\pi_\theta$ can be parameterized by $\theta$ as $\pi_\theta(\cdot | s) = \softmax(\theta(s, \cdot))$; 
that is, for all $(s,a)$,
\begin{align}
    \pi_\theta(a | s) = \frac{ \exp\{ \theta(s, a) \} }{ \sum_{a^\prime \in \gA}{ \exp\{ \theta(s, a^\prime) \} } }.
\end{align}
When there is only one state the policy $\pi_\theta = \softmax(\theta)$ is defined as $\pi_\theta(a) = \exp\{ \theta(a) \} / \sum_{a^\prime \in \gA}{ \exp\{ \theta(a^\prime) } $. The problem of policy-based RL is then to find a policy $\pi_\theta$ that maximizes the value function, i.e.,
\begin{align}
    \sup_{\theta : \gS \times \gA \to \sR}{ V^{\pi_\theta}(\rho) }. 
\end{align}
For convenience, and without loss of generality, we assume $r(s,a) \in [0, 1]$ for all $(s, a) \in \gS \times \gA$.

\subsection{One-state MDPs}

We first illustrate some key insights for  one-state MDPs with $K$ actions and $\gamma = 0$. The value function \cref{eq:state_value_function} reduces to expected reward $\pi_\theta^\top r$, where $r \in [0 ,1]^K$, $\theta \in \sR^K$, and $\pi_\theta = \softmax(\theta)$. 
\citet{mei2020global} have shown that even though $\max_{\theta}{ \pi_\theta^\top r}$ is a non-concave maximization, global convergence can be achieved with a $O(1/t)$ rate using uniform smoothness and the N\L{} inequality:

\begin{lemma}[N\L{}, \citet{mei2020global}, Lemma 3]
\label{lem:non_uniform_lojasiewicz_softmax_special}
Let $a^*$ be the optimal action.
Denote $\pi^* = \argmax_{\pi \in \Delta}{ \pi^\top r}$. Then, 
\begin{align}
    \left\| \frac{d \pi_\theta^\top r}{d \theta} \right\|_2 \ge \pi_\theta(a^*) \cdot ( \pi^* - \pi_\theta )^\top r.
\end{align}
\end{lemma}

Note that \cref{lem:non_uniform_lojasiewicz_softmax_special} is not improvable in terms of the coefficients $C(\theta) = \pi_\theta(a^*)$ and $\xi = 0$ \citep[Remark 1 and Lemma 17]{mei2020global}. 
However, this result is based on only using a \emph{uniform} smoothness coefficient $\beta = 5/2$ \citep[Lemma 2]{mei2020global}, which even empirical evidence suggests can be significantly refined.
To illustrate, we run standard policy gradient (PG) on a $3$-action one-state MDP. As shown in \cref{fig:expected_reward_non_uniform_smoothness}(a), PG first goes through a long suboptimal plateau, and then eventually escapes to approach $\pi^*$. \cref{fig:expected_reward_non_uniform_smoothness}(b) presents the spectral radius of the Hessian and the PG norm $3 \cdot \Big\| \frac{d \pi_{\theta_t}^\top r}{d \theta_t} \Big\|_2$ as functions of time $t$. 
It is evident that the smoothness behaves non-uniformly: it is close to zero at the suboptimal plateau and near $\pi^*$, highly aligned with the PG norm. 
Compared to any universal constant $\beta$, the PG  norm characterizes the non-uniform landscape information far more precisely. We formalize this observation by proving the following key result:
\begin{lemma}[NS]
\label{lem:non_uniform_smoothness_softmax_special}
Denote $\theta_\zeta\coloneqq \theta + \zeta \cdot (\theta^\prime - \theta)$ with some $\zeta\in [0,1]$. For any $r \in \left[ 0, 1\right]^K$, $\theta \mapsto \pi_\theta^\top r$ satisfies $\beta(\theta_\zeta)$ non-uniform smoothness with $\beta(\theta_\zeta) = 3 \cdot \Big\| \frac{d \pi_{\theta_\zeta}^\top r}{d {\theta_\zeta}} \Big\|_2 $.
\end{lemma}

\begin{figure}[ht]
\centering
\includegraphics[width=1.0\linewidth]{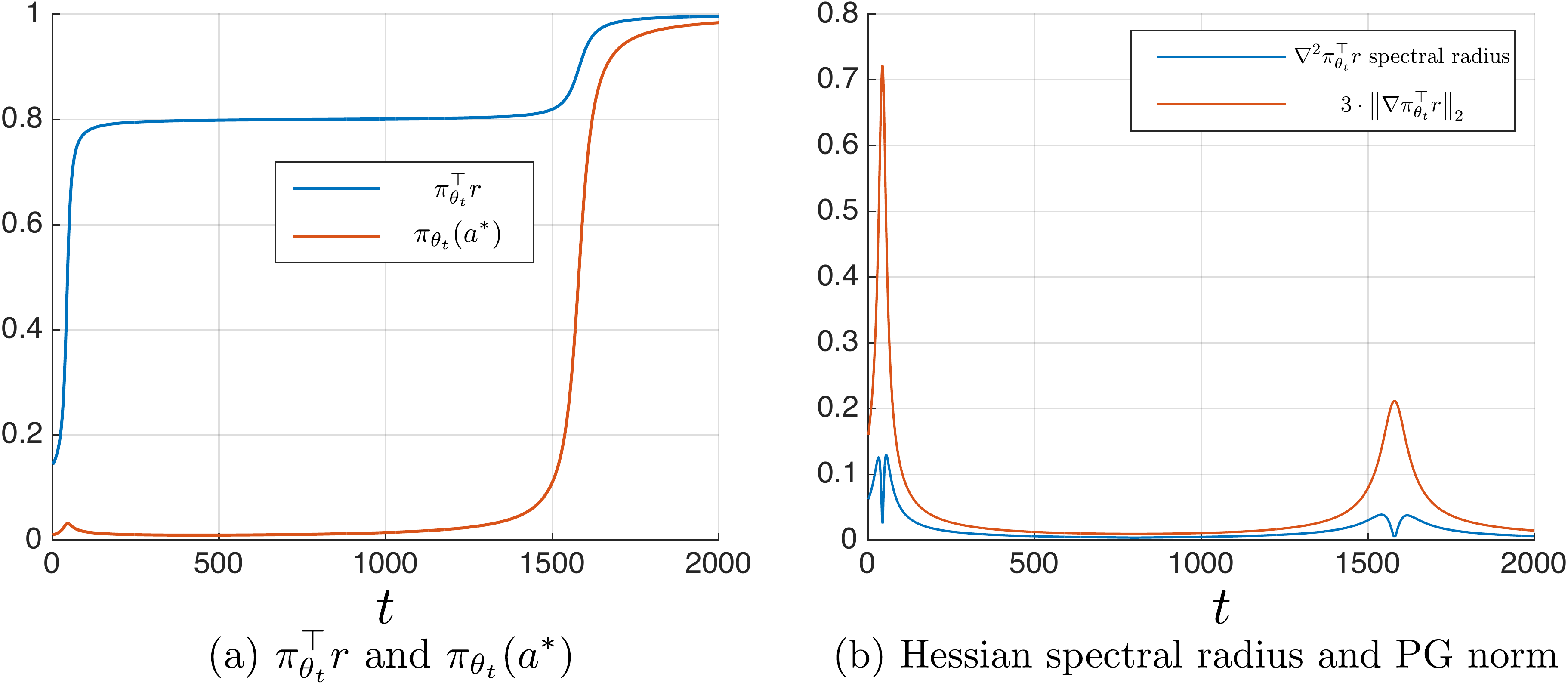}
\caption{PG results on $r = (1.0, 0.8, 0.1)^\top$.} 
\label{fig:expected_reward_non_uniform_smoothness}
\end{figure}

Comparing \cref{lem:non_uniform_smoothness_softmax_special} with (1b) in \cref{thm:general_optimization_main_result_1}, we have $\xi = 0$, and GNGD requires normalizing $\beta(\theta_\zeta)$, which is the PG norm of $\theta_\zeta$ rather than $\theta$. However, $\zeta$ is unknown. Fortunately, the next lemma shows that, if we still normalize the PG norm of $\theta$, the $\beta(\theta_\zeta)$ in \cref{lem:non_uniform_smoothness_softmax_special} can be upper bounded by $\Big\| \frac{d \pi_{\theta}^\top r}{d {\theta}} \Big\|_2$, given the learning rate is small enough:

\begin{lemma}
\label{lem:non_uniform_smoothness_intermediate_policy_gradient_norm_special}
Let $\theta^\prime = \theta + \eta \cdot \frac{d \pi_{\theta}^\top r}{d {\theta}} \Big/ \left\| \frac{d \pi_{\theta}^\top r}{d {\theta}} \right\|_2$. Denote $\theta_\zeta \coloneqq \theta + \zeta \cdot (\theta^\prime - \theta)$ with some $\zeta \in [0,1]$. We have, for all $\eta \in ( 0 , 1/3)$,
\begin{align}
    \bigg\| \frac{d \pi_{\theta_{\zeta}}^\top r}{d {\theta_{\zeta}}} \bigg\|_2 \le \frac{1}{1 - 3 \eta} \cdot \left\| \frac{d \pi_{\theta}^\top r}{d {\theta}} \right\|_2.
\end{align}
\end{lemma}

Next, the N\L{} coefficient $\pi_\theta(a^*)$ is bounded away from $0$, which provides constants in the convergence rate results.

\begin{lemma}[Non-vanishing N\L{} coefficient]
\label{lem:optimal_action_prob_positive}
\label{lem:lower_bound_cT_softmax_special}
Using normalized policy gradient method, we have $\inf_{t\ge 1} \pi_{\theta_t}(a^*) > 0$. 
\end{lemma}

To this point, we demonstrate that the non-concave function $\pi_\theta^\top r$ satisfies (1b) in \cref{thm:general_optimization_main_result_1} with $\xi = 0$ \textit{in each iteration of normalized PG}\footnote{This essentially means we prove that a uniform \L{}ojasiewicz inequality holds for the entire sequence $\left\{ \theta_t \right\}_{t\ge 1}$, but this does not
imply that the N\L{} condition is unnecessary. As shown in \citet[Remark 1]{mei2020global}, \L{}ojasiewicz-type inequalities with constant
$C > 0$ cannot hold. It can only become uniform after specifying an initialization $\theta_1$ and an algorithm (in this case, PG). Otherwise, uniform \L{}ojasiewicz cannot hold since initialization can make the N\L{}
coefficient $\pi_\theta(a^*)$ arbitrarily close to $0$.}: \cref{lem:non_uniform_smoothness_softmax_special,lem:non_uniform_smoothness_intermediate_policy_gradient_norm_special} show that the NS coefficient $\beta(\theta_{\zeta_t}) \le c_1 \cdot \Big\| \frac{d \pi_{\theta_t}^\top r}{d {\theta_t}} \Big\|_2$, while \cref{lem:non_uniform_lojasiewicz_softmax_special,lem:lower_bound_cT_softmax_special} guarantee $\Big\| \frac{d \pi_{\theta_t}^\top r}{d {\theta_t}} \Big\|_2 \ge c_2 \cdot ( \pi^* - \pi_{\theta_t} )^\top r$. Therefore, combining \cref{lem:non_uniform_smoothness_softmax_special,lem:non_uniform_lojasiewicz_softmax_special,lem:optimal_action_prob_positive,lem:non_uniform_smoothness_intermediate_policy_gradient_norm_special}, we prove the global linear convergence rate $O(e^{- c \cdot t})$ of normalized PG:

\begin{theorem}
\label{thm:final_rates_normalized_softmax_pg_special}
Using normalized PG $\theta_{t+1} = \theta_t + \eta \cdot \frac{d \pi_{\theta_t}^\top r}{d {\theta_t}} \Big/ \Big\| \frac{d \pi_{\theta_t}^\top r}{d {\theta_t}} \Big\|_2$, with $\eta = 1/6$, for all $t \ge 1$, we have, 
\begin{align}
    ( \pi^* - \pi_{\theta_t} )^\top r \le e^{ - \frac{  c \cdot (t-1) }{12} } \cdot \left( \pi^* - \pi_{\theta_{1}}\right)^\top r,
\end{align}
where $c = \inf_{t\ge 1} \pi_{\theta_t}(a^*) > 0$ is from \cref{lem:lower_bound_cT_softmax_special}, and $c$ is a constant that depends on $r$ and $\theta_1$, but not on the time $t$.
\end{theorem}

\begin{remark}
If $\pi_{\theta_1}$ is uniform, i.e., $\pi_{\theta_1}(a) = 1/K$, $\forall \ a \in [K]$, then we have $c \ge 1/K$ in \cref{thm:final_rates_normalized_softmax_pg_special}. This can be proved by showing that $\pi_{\theta_{t+1}}(a^*) \ge \pi_{\theta_{t}}(a^*)$, similar to \citet[Proposition 2]{mei2020global}.
\end{remark}

\subsection{Geometry-aware Normalized PG (GNPG)}

Next, we generalize from one-state to finite MDPs, using the GNPG\footnote{We use GNPG as the name of \cref{alg:normalized_policy_gradient_softmax}, since NPG is usually used to refer to the natural PG algorithm in RL literature.} on value function, as shown in \cref{alg:normalized_policy_gradient_softmax}.
\begin{algorithm}[ht]
   \caption{Geometry-aware Normalized Policy Gradient}
   \label{alg:normalized_policy_gradient_softmax}
\begin{algorithmic}
   \STATE {\bfseries Input:} Learning rate $\eta > 0$.
   \STATE Initialize parameter $\theta_1(s,a)$ for all $(s,a)$.
   \WHILE{$t \ge 1$}
   \STATE $\theta_{t+1} \gets \theta_{t} + \eta \cdot \frac{\partial V^{\pi_{\theta_t}}(\mu)}{\partial \theta_t} \Big/ \left\| \frac{\partial V^{\pi_{\theta_t}}(\mu)}{\partial \theta_t} \right\|_2 $.
   \ENDWHILE
\end{algorithmic}
\end{algorithm}

\subsection{General MDPs}

For general finite MDPs, we assume ``sufficient exploration'' for the initial state distribution $\mu$, which is also adapted in literature \citep{agarwal2020optimality,mei2020global}.

\begin{assumption}[Sufficient exploration]\label{ass:posinit}
The initial state distribution satisfies $\min_s \mu(s)>0$.
\end{assumption}

Given \cref{ass:posinit}, \citet{agarwal2020optimality} prove asymptotic global convergence for PG on the non-concave $\max_{\theta}{ V^{\pi_\theta}(\rho)}$ problem,
while \citet{mei2020global} strengthens this to a $O(1/t)$ rate using uniform smoothness and the following N\L{} inequality that generalizes \cref{lem:non_uniform_lojasiewicz_softmax_special}.

\begin{lemma}[N\L{}, \citet{mei2020global}, Lemma 8] Denote $S \coloneqq | \gS| $ as the total number of states. We have, $\forall \ \theta \in \sR^{\gS \times \gA}$,
\label{lem:non_uniform_lojasiewicz_softmax_general}
\begin{align*}
    \left\| \frac{\partial V^{\pi_\theta}(\mu)}{\partial \theta }\right\|_2 \ge \frac{ \min_s{ \pi_\theta(a^*(s)|s) } }{ \sqrt{S} \cdot  \left\| d_{\rho}^{\pi^*} / d_{\mu}^{\pi_\theta} \right\|_\infty } \cdot \left( V^*(\rho) - V^{\pi_\theta}(\rho) \right),
\end{align*}
where $a^*(s)$ is the action that $\pi^*$ selects in state $s$.
\end{lemma}

Here, the N\L{} degree $\xi = 0$ is not improvable \citet[Lemma 28]{mei2020global}. In one-state MDPs with $S = 1$, \cref{lem:non_uniform_lojasiewicz_softmax_general} recovers \cref{lem:non_uniform_lojasiewicz_softmax_special} with the same N\L{} coefficient $C(\theta) = \pi_\theta(a^*)$, indicating that $C(\theta)$ in \cref{lem:non_uniform_lojasiewicz_softmax_general} might also be unimprovable. On the other hand, the uniform smoothness considered in \citep{agarwal2020optimality,mei2020global} $\beta = 8 / (1 - \gamma)^3$ is too conservative, particularly when $\gamma$ is close to $1$. Our next key result shows that the policy value also satisfies a stronger NS property, with the NS coefficient being the PG norm, generalizing \cref{lem:non_uniform_smoothness_softmax_special}:
\begin{lemma}[NS]
\label{lem:non_uniform_smoothness_softmax_general}
Let \cref{ass:posinit} hold and denote $\theta_\zeta \coloneqq \theta + \zeta \cdot (\theta^\prime - \theta)$ with some $\zeta \in [0,1]$. $\theta \mapsto V^{\pi_\theta}(\mu)$ satisfies $\beta(\theta_\zeta)$ non-uniform smoothness with
\begin{align*}
    \beta(\theta_\zeta) = \left[ 3 + \frac{ 4 \cdot \left( C_\infty - (1 - \gamma) \right) }{ 1 - \gamma } \right] \cdot \sqrt{S} \cdot \left\| \frac{\partial V^{\pi_{\theta_\zeta}}(\mu)}{\partial {\theta_\zeta} }\right\|_2,
\end{align*}
where $C_\infty \coloneqq \max_{\pi}{ \left\| \frac{d_{\mu}^{\pi}}{ \mu} \right\|_\infty} \le \frac{1}{ \min_s \mu(s) } < \infty$.
\end{lemma}
In one-state MDPs with $\gamma = 0$ and $S = 1$, we have $C_\infty = 1 - \gamma$. Thus \cref{lem:non_uniform_smoothness_softmax_general} reduces to \cref{lem:non_uniform_smoothness_softmax_special} with the same NS coefficient $\beta(\theta_\zeta) = 3 \cdot \Big\| \frac{d \pi_{\theta_\zeta}^\top r}{d {\theta_\zeta}} \Big\|_2$. Similar to \cref{lem:non_uniform_smoothness_intermediate_policy_gradient_norm_special}, if we use \cref{alg:normalized_policy_gradient_softmax} with small enough learning rate, then $\beta(\theta_\zeta)$ in \cref{lem:non_uniform_smoothness_softmax_general} is upper bounded by the PG norm of $\theta$:
\begin{lemma}
\label{lem:non_uniform_smoothness_intermediate_policy_gradient_norm_general}
Let $\eta = \frac{1 - \gamma }{ 6 \cdot ( 1 - \gamma)  + 8 \cdot \left( C_\infty - (1 - \gamma) \right) } \cdot \frac{ 1}{ \sqrt{S} }$ and $\theta^\prime = \theta + \eta \cdot \frac{\partial V^{\pi_\theta}(\mu)}{\partial \theta} \Big/ \left\| \frac{\partial V^{\pi_\theta}(\mu)}{\partial \theta} \right\|_2$. Denote $\theta_\zeta \coloneqq \theta + \zeta \cdot (\theta^\prime - \theta)$ with some $\zeta \in [0,1]$. We have,
\begin{align}
    \left\| \frac{\partial V^{\pi_{\theta_\zeta}}(\mu)}{\partial \theta_\zeta} \right\|_2 \le 2 \cdot \left\| \frac{\partial V^{\pi_\theta}(\mu)}{\partial \theta} \right\|_2.
\end{align}
\end{lemma}
Next, the N\L{} coefficient $\min_s{ \pi_\theta(a^*(s)|s) }$ in \cref{lem:non_uniform_lojasiewicz_softmax_general} is lower bounded away from $0$:
\begin{lemma}[Non-vanishing N\L{} coefficient]
\label{lem:lower_bound_cT_softmax_general}
Let \cref{ass:posinit} hold. We have,
$c:=\inf_{s\in \cS,t\ge 1} \pi_{\theta_t}(a^*(s)|s) > 0$, where $\left\{ \theta_t \right\}_{t\ge 1}$ is generated by \cref{alg:normalized_policy_gradient_softmax}.
\end{lemma}
Now we have the non-concave function $V^{\pi_\theta}(\rho)$ satisfies (1b) in \cref{thm:general_optimization_main_result_1} with $\xi = 0$ \textit{in each iteration of \cref{alg:normalized_policy_gradient_softmax}}. Therefore, combining \cref{lem:non_uniform_smoothness_softmax_general,lem:non_uniform_lojasiewicz_softmax_general,lem:lower_bound_cT_softmax_general,lem:non_uniform_smoothness_intermediate_policy_gradient_norm_general}, we prove the global linear convergence rate $O(e^{-c \cdot t})$ of \cref{alg:normalized_policy_gradient_softmax}:
\begin{theorem}
\label{thm:final_rates_normalized_softmax_pg_general}
Let \cref{ass:posinit} hold and
let $\left\{ \theta_t \right\}_{t\ge 1}$ be generated using  \cref{alg:normalized_policy_gradient_softmax} with $\eta = \frac{ 1 - \gamma }{ 6 \cdot ( 1 - \gamma)  + 8 \cdot \left( C_\infty - (1 - \gamma) \right) } \cdot \frac{ 1}{ \sqrt{S} }$, where $C_\infty \coloneqq \max_{\pi}{ \left\| \frac{d_{\mu}^{\pi}}{ \mu} \right\|_\infty}$. Denote $C_\infty^\prime \coloneqq \max_{\pi}{ \left\| \frac{d_{\rho}^{\pi}}{ \mu} \right\|_\infty}$. Let $c$ be the positive constant from \cref{lem:lower_bound_cT_softmax_general}.
We have, for all $t\ge 1$,
\begin{align*}
    V^*(\rho) - V^{\pi_{\theta_t}}(\rho) \le \frac{ \left( V^*(\mu) - V^{\pi_{\theta_1}}(\mu) \right) \cdot C_\infty^\prime }{ 1 - \gamma} \cdot e^{ - C \cdot (t-1)},
\end{align*}
where $C = \frac{ ( 1 - \gamma)^2 \cdot c }{ 12 \cdot ( 1 - \gamma)  + 16 \cdot \left( C_\infty - (1 - \gamma) \right) } \cdot \frac{ 1}{ S } \cdot \Big\| \frac{d_{\mu}^{\pi^*}}{\mu} \Big\|_\infty^{-1} $.
\end{theorem}
Not only the $O(e^{-c \cdot t})$ rate in \cref{thm:final_rates_normalized_softmax_pg_general} is faster than $O(1/t)$ for standard PG without normalization, but also the constant is better than \citet[Theorem 4]{mei2020global}. The strictly better dependence $c$ ($\gg c^2$ in PG) is related to faster escaping plateaus as shown later \citep{mei2020escaping}. 
\begin{remark}
The conclusion of GNPG has better constants than PG ($c \gg c^2$) arises from upper bounds (\cref{thm:final_rates_normalized_softmax_pg_general} and \citet[Theorem 4]{mei2020global}), which is also supported by empirical evidence. In fact, there exists a lower bound that shows $c$ cannot be removed for PG \citep[Theorem 1]{mei2020escaping} under one-state MDP settings. For finite MDPs, very recently, \citet{li2021softmax} show that for softmax PG (without normalization), $c$ can be very small in terms of the number of states. It remains open to consider whether $c$ is reasonably large for GNPG.
\end{remark}
\begin{remark}
To our knowledge, existing PG variants can achieve linear convergence $O(e^{-c \cdot t})$ only if using at least one of the following techniques: (a) \textbf{regularization}; \citet{mei2020global} prove that entropy regularized PG enjoys $O(e^{-c \cdot t})$ convergence toward the regularized optimal policy. (b) \textbf{natural gradient}; \citet{cen2020fast} prove that entropy regularized natural PG achieves linear convergence. (c) \textbf{exact line-search}; \citet{bhandari2020note} prove that without parameterization, PG variants with exact line-search achieve linear rates by approximating policy iteration.

Among the above techniques, regularization changes the problem to regularized MDPs, while natural PG and line-search require solving expensive optimization problems to do updates, since each update is an $\argmax$.
\end{remark}
On the contrary, \cref{alg:normalized_policy_gradient_softmax} enjoys global $O(e^{-c \cdot t})$ rate \textit{(i)} without using regularization, since \cref{alg:normalized_policy_gradient_softmax} directly works on the original MDPs; \textit{(ii)} without solving optimization problems in each iteration, and the normalized PG update is cheap. The strong results rely on the NS and N\L{} properties, and also the geometry-aware normalization that takes advantage of the non-uniform properties.
\begin{remark}
Standard softmax PG with bounded learning rate follows $\Omega(1/t)$ lower bound \citep{mei2020global}, which is consistent with the case (1) in \cref{thm:general_optimization_main_result_1}. \cref{alg:normalized_policy_gradient_softmax} achieves faster  linear convergence rates, indicating that the adaptive update stepsize $\eta / \left\| \nabla V^{\pi_{\theta_t}}(\rho) \right\|_2$ is asymptotically unbounded, since $\left\| \nabla V^{\pi_{\theta_t}}(\rho) \right\|_2 \to 0$ as $t \to \infty$.
\end{remark}
\begin{figure}[ht]
\centering
\includegraphics[width=1.0\linewidth]{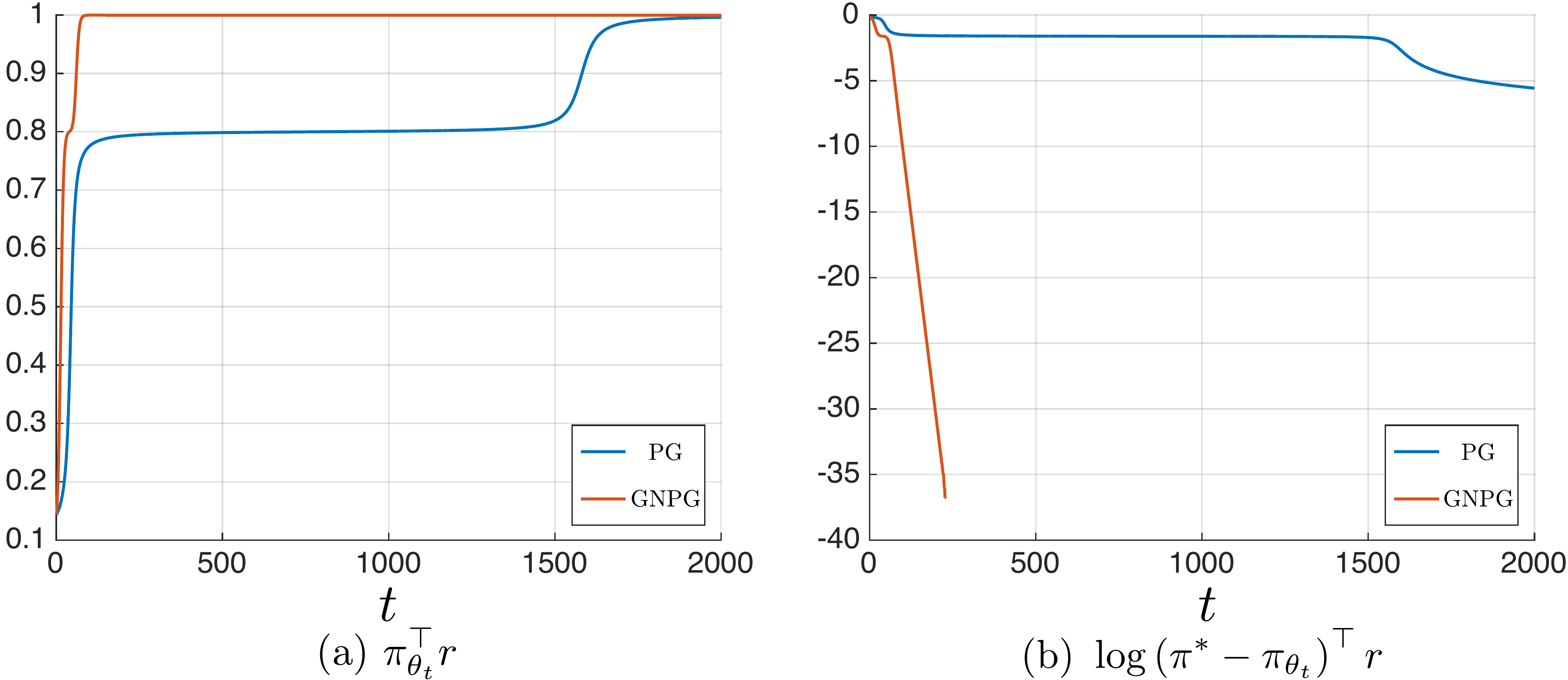}
\caption{PG and GNPG on $r = (1.0, 0.8, 0.1)^\top$.} 
\label{fig:pg_npg_expected_reward_sub_optimality}
\end{figure}
We compare PG and GNPG on the one-state MDP problem as shown in \cref{fig:pg_npg_expected_reward_sub_optimality}. \cref{fig:pg_npg_expected_reward_sub_optimality}(a) shows that GNPG escapes from the sub-optimal plateau significantly faster than PG, while \cref{fig:pg_npg_expected_reward_sub_optimality}(b) shows that GNPG follows linear convergence $O(e^{-c \cdot t})$ of sub-optmality, verifying the theoretical results. Similar experimental results on synthetic tree MDPs with multiple states are presented in \cref{sec:additional_simulations_experiments}.


\section{Generalized Linear Models}
\label{sec:generalized_linear_model}

Next, we investigate the generalized linear model (GLM) with quasi-maximum likelihood estimate (quasi-MLE), which applied widely in supervised learning. We show that the mean squared error (MSE) of GLM \citep{hazan2015beyond} is in the non-convex function class $\gZ$ in \cref{fig:non_uniform_function_classes_diagram_table}, and it satisfies the case (2) in \cref{thm:general_optimization_main_result_1} with $\xi = 1/2$. As a result, both GD and GNGD achieve global linear convergence rates $O(e^{-c \cdot t})$, significantly improving the best existing results of $O(1 / \sqrt{t})$ \citep{hazan2015beyond}. We also provide new understandings of using normalization in GLM based on our non-uniform analysis.

\subsection{Basic Settings and Notations}

Given a training data set $\gD = \{ ( x_i, y_i ) \}_{ i \in [N] }$, which consists of $N$ data points, there is a feature map $x_i \mapsto \phi(x_i) \in \sR^d$ for each pair $(x_i, y_i) \in \gD$. We denote $\phi_i \coloneqq \phi(x_i)$ for conciseness. For each data point $x_i$, we have $y_i \in [0, 1]$ as the ground truth likelihood. Following \citet{hazan2015beyond}, our model is parameterized by a weight vector $\theta \in \sR^d$ as ,
\begin{align}
    \pi_i = \sigmoid(\phi_i^\top \theta) = \frac{1}{1 + \exp\{ - \phi_i^\top \theta \}},
\end{align}
where $\sigmoid: \sR \to (0,1)$ is the sigmoid activation. The problem is to minimize the mean squared error (MSE),
\begin{align}
\label{eq:mean_squared_error_sigmoid}
    \min_{\theta} \gL(\theta) = \min_{\theta \in \sR^d} \frac{1}{N} \cdot \sum_{i=1}^{N}{ (\pi_i - y_i)^2 }.
\end{align}
We assume $y_i = \pi_i^* \coloneqq \sigmoid(\phi_i^\top \theta^*)$, where $\theta^* \in \sR^d$, and $\left\| \theta^* \right\|_2 < \infty$, which means the target $y_i$ is realizable and non-deterministic. According to \citet{hazan2015beyond}, the MSE in \cref{eq:mean_squared_error_sigmoid} is not quasi-convex (thus not convex). Fortunately, \citet{hazan2015beyond} manage to show that \cref{eq:mean_squared_error_sigmoid} satisfies a weaker Strictly-Locally-Quasi-Convex (SLQC) property, based on which they prove the following result:
\begin{theorem}[\citet{hazan2015beyond}]
\label{thm:ngd_glm_sigmoid_convergence_rate}
With diminishing learning rate $\eta_t \in \Theta(1/\sqrt{t})$, the normalized gradient descent (NGD) update $\theta_{t+1} \gets \theta_t - \eta_t \cdot \frac{\partial \gL(\theta_t)}{\partial \theta_t} \Big/ \left\| \frac{\partial \gL(\theta_t)}{\partial \theta_t} \right\|_2$ satisfies,
\begin{align}
    \delta(\theta_t) \coloneqq \gL(\theta_t) - \gL(\theta^*) \in O(1/\sqrt{t}),
\end{align}
where $\theta^* \coloneqq \argmin_{\theta}{ \gL(\theta) }$ is the global optimal solution.
\end{theorem}

\subsection{Fast Convergence using Non-uniform Analysis}

Based on the $O(1/\sqrt{t})$ rate for NGD in \cref{thm:ngd_glm_sigmoid_convergence_rate}, \citet{hazan2015beyond} propose to normalize gradient norm in MSE minimization. However, there is no lower bound for other methods including GD on GLM, and thus it is not clear if there exists a faster rate for GLM optimization.

Surprisingly, we prove that both GD and GNGD actually achieve much faster rates of $O(e^{-c \cdot t})$ using the non-uniform analysis. Our first key finding is to show that the MSE in GLM satisfies a new N\L{} inequality with $\xi = 1/2$:
\begin{lemma}[N\L{}]
\label{lem:non_uniform_lojasiewicz_glm_sigmoid_realizable}
Denote $u(\theta) \coloneqq \min_{i \in [N]}{ \left\{ \pi_i \cdot \left( 1 - \pi_i \right) \right\} }$, and $v \coloneqq \min_{i \in [N]}{ \left\{ \pi_i^*\cdot \left( 1 - \pi_i^* \right) \right\} }$. We have,
\begin{align}
    \left\| \frac{\partial \gL(\theta)}{\partial \theta} \right\|_2 \ge C(\theta, \phi) \cdot \left[ \frac{1}{N} \cdot \sum_{i=1}^{N}{ \left( \pi_i - \pi_i^* \right)^2 } \right]^{\frac{1}{2}},
\end{align}
holds for all $\theta \in \sR^d$, where 
\begin{align}
    C(\theta, \phi) = 8 \cdot u(\theta) \cdot \min\left\{ u(\theta), v \right\} \cdot \sqrt{\lambda_\phi},
\end{align}
and $\lambda_\phi$ is the smallest positive eigenvalue of $\frac{1}{N} \cdot \sum_{i=1}^{N}{ \phi_i \phi_i^\top}$.
\end{lemma}
\begin{remark}
It is not clear if results similar to \cref{lem:non_uniform_lojasiewicz_glm_sigmoid_realizable} hold without assuming: \textit{(i)} realizable optimal prediction $y_i = \pi_i^* \coloneqq \sigmoid(\phi_i^\top \theta^*)$; \textit{(ii)}  non-deterministic optimal prediction $\left\| \theta^* \right\|_2 < \infty$. We leave it as an open question to study non-uniformity of GLM without the above assumptions.
\end{remark}
In \cref{lem:non_uniform_lojasiewicz_glm_sigmoid_realizable}, $\lambda_\phi$ is determined by the feature $\phi$, and $u(\theta)$ shows that the gradient is vanishing when $\pi_i$ is near deterministic, which is consistent with the fact that the sigmoid saturates and provides uninformative gradient as the parameter magnitude becomes large.
\begin{figure}[ht]
\centering
\includegraphics[width=1.0\linewidth]{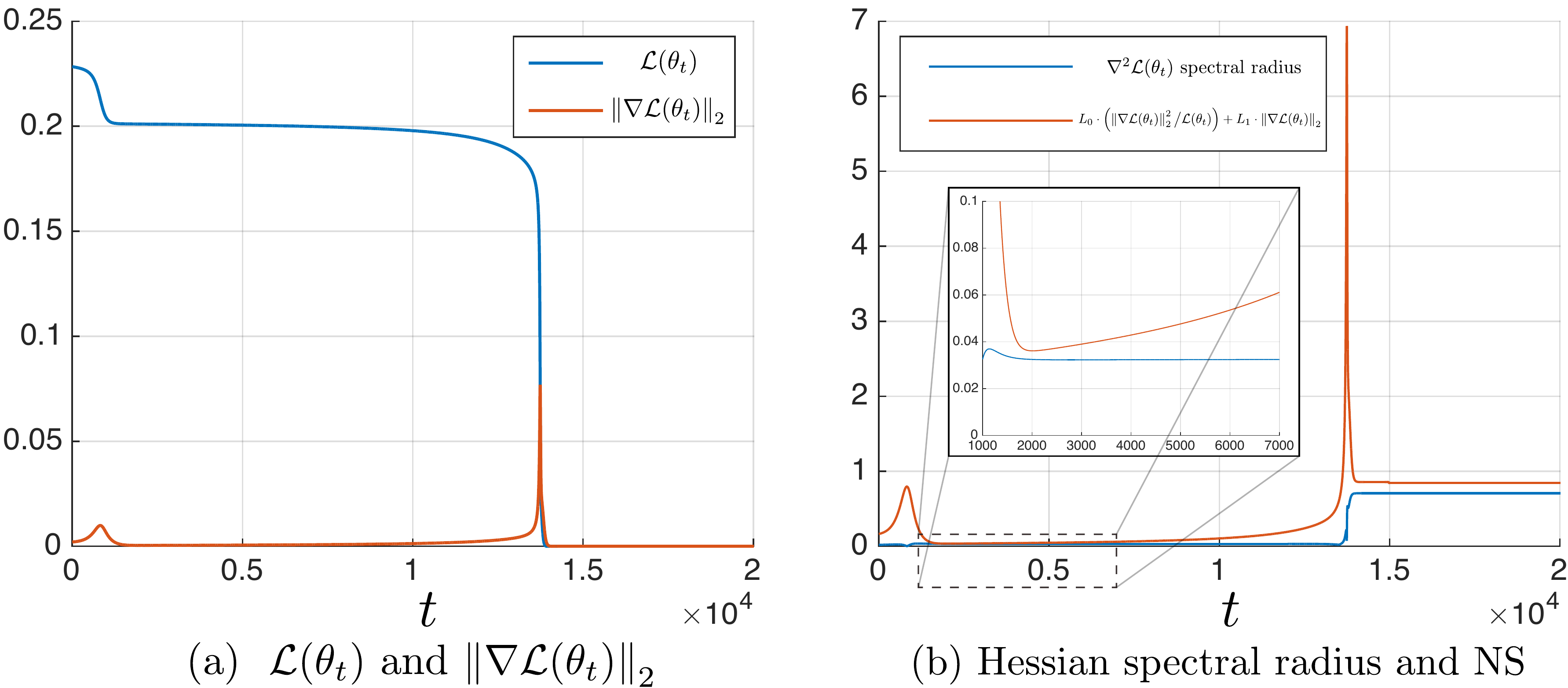}
\caption{Experiments on GLM using GD.} 
\label{fig:glm_sigmoid_non_uniform_smoothness}
\end{figure}

We run GD on one example with $N = 10$ and $d = 2$. As shown in \cref{fig:glm_sigmoid_non_uniform_smoothness}, the gradient norm $\left\| \nabla \gL(\theta_t) \right\|_2$ is close to zero at plateaus and near optimum. However, unlike the PG, the spectral radius of the Hessian $\nabla^2 \gL(\theta_t)$ is only close to zero at plateaus, while it approaches positive constant near optimum. This indicates a different NS condition other than \cref{lem:non_uniform_smoothness_softmax_special,lem:non_uniform_smoothness_softmax_general} is needed, since only gradient norm $\to 0$ cannot upper bound the spectral radius of Hessian $\to c > 0$. With some calculations, we prove the following key results:
\begin{lemma}[Smoothness and NS]
\label{lem:non_uniform_smoothness_glm_sigmoid_realizable}
$\gL(\theta)$ satisfies $\beta$ smoothness with 
\begin{align}
\beta = \frac{3}{8} \cdot \max_{i \in [N]}{ \left\| \phi_i \right\|_2^2 },   
\end{align}
and $\beta(\theta)$ NS with
\begin{align*}
    \beta(\theta) = L_1 \cdot \left\| \frac{\partial \gL(\theta)}{\partial \theta} \right\|_2 + L_0 \cdot \left( \left\| \frac{\partial \gL(\theta)}{\partial \theta} \right\|_2^2 \Big/ \gL(\theta) \right).
\end{align*}
\end{lemma}
At the optimal solution $\theta^*$, the spectral radius of the Hessian $\frac{\partial^2 \gL(\theta^*)}{\partial (\theta^*)^2}$ is strictly positive. Therefore, the MSE objective of \cref{eq:mean_squared_error_sigmoid} is in the non-convex function class $\gZ$ in \cref{fig:non_uniform_function_classes_diagram_table}, and it satisfies the case (2) in \cref{thm:general_optimization_main_result_1} with $\xi = 1/2$. Combining \cref{lem:non_uniform_lojasiewicz_glm_sigmoid_realizable,lem:non_uniform_smoothness_glm_sigmoid_realizable} and applying \cref{thm:general_optimization_main_result_1}, we have the following global linear convergence result:
\begin{theorem}
\label{thm:final_rates_normalized_glm_sigmoid_realizable}
With $\eta = 1/ \beta $, GD update satisfies for all $t \ge 1$, $\gL(\theta_t) \le \gL(\theta_1) \cdot e^{- C^2 \cdot (t-1)}$.
With $\eta \in \Theta(1)$, GNGD update satisfies for all $t \ge 1$, $\gL(\theta_t) \le \gL(\theta_1) \cdot e^{- C \cdot (t-1)}$, where $C \in (0,1)$, i.e., GNGD is strictly faster than GD.
\end{theorem}
\cref{thm:final_rates_normalized_glm_sigmoid_realizable} significantly improves the $O(1/\sqrt{t})$ rate in \cref{thm:ngd_glm_sigmoid_convergence_rate}. The key difference is that we discovered a new N\L{} inequality of \cref{lem:non_uniform_lojasiewicz_glm_sigmoid_realizable} that is satisfied by GLMs. The linear convergence rates are verified in \cref{sec:additional_simulations_experiments}.

In \cref{thm:final_rates_normalized_glm_sigmoid_realizable}, we have $C = \inf_{t \ge 1}{ C(\theta_t, \phi)}$, which is very close to zero if $\pi_i$ is near deterministic, and GD suffers sub-optimality plateaus as shown in \cref{fig:example_glm_sigmoid}. GNGD has strictly (orders of magnitudes) better constant dependence $C \gg C^2$, and escapes plateaus significantly faster than GD. Intuitively, for the GLM in \cref{fig:example_glm_sigmoid}, $C$ in \cref{thm:general_optimization_main_result_1} is lower bounded reasonably if $\theta_1$ is initialized within some finite distance of the central valley containing $\theta^*$.

Combining the N\L{} and NS properties (\cref{lem:non_uniform_lojasiewicz_glm_sigmoid_realizable,lem:non_uniform_smoothness_glm_sigmoid_realizable}), we provide new understandings of using normalization in GLM: \textit{(i)} \textbf{First}, using standard NGD \citep{hazan2015beyond} for all $t \ge 1$ is not a good choice. By examining the asymptotic behaviour as $\theta \to \theta^*$, we have $\beta(\theta) \to \beta > 0$. However, the normalization $\Big\| \frac{\partial \gL(\theta)}{\partial \theta} \Big\|_2$ in standard NGD gives incremental updates with adaptive stepsize $\to \infty$. To guarantee convergence, it is necessary to use $\eta_t \to 0$, which counteracts normalization and slows down the learning, since it could be not easy to find a learning rate scheme. This is consistent with the $O(e^{-c \cdot t})$ result for GD with $\eta > 0$ and without normalization in  \cref{thm:final_rates_normalized_glm_sigmoid_realizable}. \textit{(ii)} \textbf{Second}, using geometry-aware normalization $\beta(\theta_t)$ is a better choice than normalizing the gradient norm $\left\| \nabla \gL(\theta_t) \right\|_2$. We elaborate this point by investigating \textit{both the asymptotic and the early-stage behaviours} using NS-N\L{}. Since $\beta(\theta_t) \to \beta > 0$ asymptotically, GNGD is approaching GD as $\theta_t \to \theta^*$, which makes GNGD enjoy the same $O(e^{- c \cdot t})$ rate. On the other hand, at early-stage optimization (e.g., close to initialization in \cref{fig:example_glm_sigmoid}), when $\theta_t$ is far from $\theta^*$, we have thus $\beta(\theta_t) \le c \cdot \left\| \frac{\partial \gL(\theta_t)}{\partial \theta_t} \right\|_2$. Then GNGD is close to NGD, which guarantees strictly better progresses than GD. This is because of the progress of GNGD in each iteration at this time is about $\left\| \nabla \gL(\theta_t) \right\|_2$, while the progress of GD is $\left\| \nabla \gL(\theta_t) \right\|_2^2$, and the gradient norm is close to $0$ on plateaus. Using N\L{} of \cref{lem:non_uniform_lojasiewicz_glm_sigmoid_realizable}, GNGD will have strictly better constant dependence $C$ than $C^2$ in GD.

\section{Conclusions and Future Work}
\label{sec:conclusions_future_work}

The main contributions of this paper concern a general characterization and analysis based on non-uniform properties, which are not only sufficiently general to cover concrete examples, but also significantly improve convergence rates over previous work and even over classical lower bounds. The most exciting part is the techniques apply to important applications in machine learning that involve non-convex optimization problems. One direction is to further push the analysis to other domains with more complex function approximators, including neural networks \citep{allen2019convergence}. Another valuable future work is to incorporate stochastic gradient \citep{karimi2016linear} and other adaptive gradient-based methods \citep{kingma2014adam} in the analysis. Finally, applying other cases of non-uniform properties beyond those mentioned in this paper would be interesting.

\section*{Acknowledgements}

The authors would like to thank anonymous reviewers for their valuable comments. Jincheng Mei would like to thank Fuheng Cui for spotting a mistake in \cref{fig:example_power_1p5}(c) in a previous verion. Jincheng Mei would like to thank Ziyi Chen for pointing that N\L{} recovers global K\L{} property. Jincheng Mei and Bo Dai would like to thank Nicolas Le Roux for providing feedback on a draft of
this manuscript.
Csaba Szepesv\'ari and Dale Schuurmans gratefully acknowledge funding from the Canada CIFAR AI Chairs Program, Amii and NSERC.

\bibliography{all}

\begin{thebibliography}{22}
\providecommand{\natexlab}[1]{#1}
\providecommand{\url}[1]{\texttt{#1}}
\expandafter\ifx\csname urlstyle\endcsname\relax
  \providecommand{\doi}[1]{doi: #1}\else
  \providecommand{\doi}{doi: \begingroup \urlstyle{rm}\Url}\fi

\bibitem[Agarwal et~al.(2020)Agarwal, Kakade, Lee, and
  Mahajan]{agarwal2020optimality}
Agarwal, A., Kakade, S.~M., Lee, J.~D., and Mahajan, G.
\newblock Optimality and approximation with policy gradient methods in markov
  decision processes.
\newblock In \emph{Conference on Learning Theory}, pp.\  64--66. PMLR, 2020.

\bibitem[Allen-Zhu et~al.(2019)Allen-Zhu, Li, and Song]{allen2019convergence}
Allen-Zhu, Z., Li, Y., and Song, Z.
\newblock A convergence theory for deep learning via over-parameterization.
\newblock In \emph{International Conference on Machine Learning}, pp.\
  242--252. PMLR, 2019.

\bibitem[Bhandari \& Russo(2020)Bhandari and Russo]{bhandari2020note}
Bhandari, J. and Russo, D.
\newblock A note on the linear convergence of policy gradient methods.
\newblock \emph{arXiv preprint arXiv:2007.11120}, 2020.

\bibitem[Bubeck(2014)]{bubeck2014convex}
Bubeck, S.
\newblock Convex optimization: Algorithms and complexity.
\newblock \emph{arXiv preprint arXiv:1405.4980}, 2014.

\bibitem[Cen et~al.(2020)Cen, Cheng, Chen, Wei, and Chi]{cen2020fast}
Cen, S., Cheng, C., Chen, Y., Wei, Y., and Chi, Y.
\newblock Fast global convergence of natural policy gradient methods with
  entropy regularization.
\newblock \emph{arXiv preprint arXiv:2007.06558}, 2020.

\bibitem[Cohen et~al.(2021)Cohen, Kaur, Li, Kolter, and
  Talwalkar]{cohen2021gradient}
Cohen, J., Kaur, S., Li, Y., Kolter, J.~Z., and Talwalkar, A.
\newblock Gradient descent on neural networks typically occurs at the edge of
  stability.
\newblock In \emph{International Conference on Learning Representations}, 2021.

\bibitem[Hazan et~al.(2015)Hazan, Levy, and Shalev-Shwartz]{hazan2015beyond}
Hazan, E., Levy, K., and Shalev-Shwartz, S.
\newblock Beyond convexity: Stochastic quasi-convex optimization.
\newblock \emph{Advances in neural information processing systems},
  28:\penalty0 1594--1602, 2015.

\bibitem[Hazan et~al.(2019)Hazan, Kakade, Singh, and
  Van~Soest]{hazan2019provably}
Hazan, E., Kakade, S., Singh, K., and Van~Soest, A.
\newblock Provably efficient maximum entropy exploration.
\newblock In \emph{International Conference on Machine Learning}, pp.\
  2681--2691. PMLR, 2019.

\bibitem[Kakade \& Langford(2002)Kakade and Langford]{kakade2002approximately}
Kakade, S. and Langford, J.
\newblock Approximately optimal approximate reinforcement learning.
\newblock In \emph{ICML}, volume~2, pp.\  267--274, 2002.

\bibitem[Karimi et~al.(2016)Karimi, Nutini, and Schmidt]{karimi2016linear}
Karimi, H., Nutini, J., and Schmidt, M.
\newblock Linear convergence of gradient and proximal-gradient methods under
  the polyak-{\l}ojasiewicz condition.
\newblock In \emph{Joint European Conference on Machine Learning and Knowledge
  Discovery in Databases}, pp.\  795--811. Springer, 2016.

\bibitem[Kingma \& Ba(2014)Kingma and Ba]{kingma2014adam}
Kingma, D.~P. and Ba, J.
\newblock Adam: A method for stochastic optimization.
\newblock \emph{arXiv preprint arXiv:1412.6980}, 2014.

\bibitem[Kurdyka(1998)]{kurdyka1998gradients}
Kurdyka, K.
\newblock On gradients of functions definable in o-minimal structures.
\newblock In \emph{Annales de l'institut Fourier}, volume~48, pp.\  769--783,
  1998.

\bibitem[Li et~al.(2021)Li, Wei, Chi, Gu, and Chen]{li2021softmax}
Li, G., Wei, Y., Chi, Y., Gu, Y., and Chen, Y.
\newblock Softmax policy gradient methods can take exponential time to
  converge.
\newblock \emph{arXiv preprint arXiv:2102.11270}, 2021.

\bibitem[\L{}ojasiewicz(1963)]{lojasiewicz1963propriete}
\L{}ojasiewicz, S.
\newblock Une propri{\'e}t{\'e} topologique des sous-ensembles analytiques
  r{\'e}els.
\newblock \emph{Les {\'e}quations aux d{\'e}riv{\'e}es partielles},
  117:\penalty0 87--89, 1963.

\bibitem[Mei et~al.(2020{\natexlab{a}})Mei, Xiao, Dai, Li, Szepesv{\'a}ri, and
  Schuurmans]{mei2020escaping}
Mei, J., Xiao, C., Dai, B., Li, L., Szepesv{\'a}ri, C., and Schuurmans, D.
\newblock Escaping the gravitational pull of softmax.
\newblock \emph{Advances in Neural Information Processing Systems}, 33,
  2020{\natexlab{a}}.

\bibitem[Mei et~al.(2020{\natexlab{b}})Mei, Xiao, Szepesvari, and
  Schuurmans]{mei2020global}
Mei, J., Xiao, C., Szepesvari, C., and Schuurmans, D.
\newblock On the global convergence rates of softmax policy gradient methods.
\newblock In \emph{International Conference on Machine Learning}, pp.\
  6820--6829. PMLR, 2020{\natexlab{b}}.

\bibitem[Nemirovski \& Yudin(1983)Nemirovski and Yudin]{nemirovski1983problem}
Nemirovski, A.~S. and Yudin, D.~B.
\newblock Problem complexity and method efficiency in optimization.
\newblock 1983.

\bibitem[Nesterov(2003)]{nesterov2003introductory}
Nesterov, Y.
\newblock \emph{Introductory lectures on convex optimization: A basic course},
  volume~87.
\newblock Springer Science \& Business Media, 2003.

\bibitem[Polyak(1963)]{polyak1963gradient}
Polyak, B.~T.
\newblock Gradient methods for minimizing functionals.
\newblock \emph{Zhurnal Vychislitel'noi Matematiki i Matematicheskoi Fiziki},
  3\penalty0 (4):\penalty0 643--653, 1963.

\bibitem[Sutton et~al.(2000)Sutton, McAllester, Singh, and
  Mansour]{sutton2000policy}
Sutton, R.~S., McAllester, D.~A., Singh, S.~P., and Mansour, Y.
\newblock Policy gradient methods for reinforcement learning with function
  approximation.
\newblock In \emph{Advances in neural information processing systems}, pp.\
  1057--1063, 2000.

\bibitem[Wilson et~al.(2019)Wilson, Mackey, and
  Wibisono]{wilson2019accelerating}
Wilson, A., Mackey, L., and Wibisono, A.
\newblock Accelerating rescaled gradient descent: Fast optimization of smooth
  functions.
\newblock \emph{arXiv preprint arXiv:1902.08825}, 2019.

\bibitem[Zhang et~al.(2019)Zhang, He, Sra, and Jadbabaie]{zhang2019gradient}
Zhang, J., He, T., Sra, S., and Jadbabaie, A.
\newblock Why gradient clipping accelerates training: A theoretical
  justification for adaptivity.
\newblock In \emph{International Conference on Learning Representations}, 2019.

\end{thebibliography}
\bibliographystyle{icml2021}

\appendix
\onecolumn

\begin{center}
\LARGE \textbf{Appendix}
\end{center}

The appendix is organized as follows.
\begin{itemize}
    \item \cref{sec:proofs_non_uniform_analysis}: proofs for general optimization results in \cref{sec:non_uniform_analysis}.
    \begin{itemize}
        \item \cref{sec:proofs_general_optimization_main_result}: proofs for the main \cref{thm:general_optimization_main_result_1}.
        \item \cref{sec:general_optimization_function_class_relation}: proofs for the function classes, i.e., \cref{prop:function_class_relation,prop:non_empty_convex_function_classes,prop:non_empty_non_convex_function_classes,prop:absulte_power_p}.
    \end{itemize}
    \item \cref{sec:proofs_policy_gradient}: proofs for policy gradient in \cref{sec:policy_gradient}.
    \begin{itemize}
        \item \cref{sec:proofs_policy_gradient_one_state_mdps}: proofs for one-state MDPs.
        \item \cref{sec:proofs_policy_gradient_general_mdps}: proofs for general MDPs.
    \end{itemize}
    \item \cref{sec:proofs_generalized_linear_model}: proofs for generalized linear model in \cref{sec:generalized_linear_model}.
    \item \cref{sec:supporting_lemmas}: miscellaneous extra supporting results those are not mentioned in the main paper.
    \item \cref{sec:examples_non_convex_nl_inequality}: non-convex (non-concave) examples for the N\L{} inequality in literature.
    \item \cref{sec:additional_simulations_experiments}: additional simulation results which are not in the main paper.
\end{itemize}

\section{Proofs for \cref{sec:non_uniform_analysis}}
\label{sec:proofs_non_uniform_analysis}

\subsection{Main \cref{thm:general_optimization_main_result_1}}
\label{sec:proofs_general_optimization_main_result}

\textbf{\cref{thm:general_optimization_main_result_1}.}
Suppose $f: \Theta \to \sR$ satisfies NS with $\beta(\theta)$ and 
the N\L{} inequality with $(C(\theta), \xi)$. 
Suppose $C \coloneqq \inf_{t \ge 1}{ C(\theta_t) } > 0$ for GD and GNGD. Let $\delta(\theta) \coloneqq f(\theta) - f(\theta^*)$ be the sub-optimality gap. The following hold:

\begin{description}[topsep=0pt,parsep=0pt]

\item[(1a)] if $\beta(\theta) \le c \cdot  \delta(\theta)^{1-2\xi}$ with $\xi \in (- \infty, 1/2)$, then the conclusions of (1b) hold;
\item[(1b)] if $\beta(\theta) \le c \cdot \left\| \nabla f(\theta) \right\|_2^{\frac{1 - 2 \xi}{1 - \xi}}$ with $\xi \in (- \infty, 1/2)$, then GD with $\eta \in O(1)$ achieves $\delta(\theta_t) \in  \Theta(1/t^{\frac{1}{1 - 2 \xi}})$,
and GNGD achieves $\delta(\theta_t) \in O(e^{-t})$.
%
\item[(2a)] if $\beta(\theta) \le L_0 + L_1 \cdot \left\| \nabla f(\theta) \right\|_2$, then  the conclusions of (2b) hold;
\item[(2b)] if $\beta(\theta) \le L_0 \cdot \frac{\left\| \nabla f(\theta) \right\|^2}{ \delta(\theta)^{2 - 2 \xi} } + L_1 \cdot \left\| \nabla f(\theta) \right\|_2$, then GD and GNGD both achieve $\delta(\theta_t) \in O(1/t^{\frac{1}{1 - 2 \xi}})$ when $\xi \in (- \infty, 1/2)$, and $O(e^{-t})$ when $\xi = 1/2$. GNGD has strictly better constant than GD ($1 \ge C \ge C^2$).
%
\item[(3a)] if $\beta(\theta)\!\le\!c\!\cdot\!\left\| \nabla f(\theta) \right\|_2^{\frac{1 - 2 \xi}{1 - \xi}}\!$ with $\xi \in (1/2, 1)$, then  the conclusions of (3b) hold;
\item[(3b)] if $\beta(\theta) \le c \cdot \delta(\theta)^{1 - 2 \xi}$ with $\xi \in (1/2, 1)$, then GD with $\eta \in \Theta(1)$ does not converge, while GNGD achieves $\delta(\theta_t) \in O(e^{-t})$.
\end{description}

\begin{proof}

\textbf{(1a) First part:} \textit{$O(1/t^{\frac{1}{1 - 2 \xi}})$ upper bound for GD update $\theta_{t+1} \gets \theta_t - \eta \cdot \nabla f(\theta_t) $ with $\eta \in O(1)$.}

We show that using GD with learning rate $\eta = \frac{1}{ c \cdot \delta(\theta_1)^{1 - 2 \xi} }$, the sub-optimality $\delta(\theta_t)$ is monotonically decreasing. And thus there exists a universal constant $\beta > 0$ such that $\beta(\theta_t) \le \beta$, for all $t \ge 1$.

Denote $\beta \coloneqq c \cdot \delta(\theta_1)^{1 - 2 \xi} $. We have $\beta \in (0, \infty)$, since $f(\theta^*) > - \infty$, and $f(\theta^*) < f(\theta_1) < \infty$. By assumption, we have $\beta(\theta_1) \le \beta$. According to \cref{lem:descent_lemma_smooth_function}, using GD with $\eta = \frac{1}{ \beta }$, we have,
\begin{align}
\label{eq:general_optimization_main_result_1_1a_intermediate_a}
    \delta(\theta_2) - \delta(\theta_1) = f(\theta_2) - f(\theta_1) \le 0.
\end{align}
Therefore, we have,
\begin{align}
\label{eq:general_optimization_main_result_1_1a_intermediate_b}
    \beta(\theta_2) &\le c \cdot \delta(\theta_2)^{1 - 2 \xi} \qquad \left( \text{by assumption} \right) \\
    &\le  c \cdot \delta(\theta_1)^{1 - 2 \xi} \qquad \left( 0 < \delta(\theta_2) \le \delta(\theta_1) \text{ and } \xi < 1/2 \right) \\
    &= \beta.
\end{align}
Repeating similar arguments of \cref{eq:general_optimization_main_result_1_1a_intermediate_a,eq:general_optimization_main_result_1_1a_intermediate_b}, we have, for all $t \ge 1$, $\beta(\theta_t) \le \beta$ and,
\begin{align}
\label{eq:general_optimization_main_result_1_1a_intermediate_c}
    0 < \delta(\theta_{t+1}) \le \delta(\theta_{t}).
\end{align}
Therefore, we have, for all $t \ge 1$ (or using \cref{lem:descent_lemma_smooth_function}),
\begin{align}
\label{eq:general_optimization_main_result_1_1a_intermediate_1}
    \delta(\theta_{t+1}) - \delta(\theta_t) &= f(\theta_{t+1}) - f(\theta_t) \\
    &\le \nabla f(\theta_t)^\top \left( \theta_{t+1} - \theta_t \right) + \frac{\beta(\theta_t)}{2} \cdot \left\| \theta_{t+1} - \theta_t \right\|_2^2 \qquad \left( \text{NS} \right) \\
    &\le \nabla f(\theta_t)^\top \left( \theta_{t+1} - \theta_t \right) + \frac{\beta}{2} \cdot \left\| \theta_{t+1} - \theta_t \right\|_2^2 \qquad \left( \beta(\theta_t) \le \beta \right) \\
\label{eq:smoothness_progress}
    &= - \frac{1}{ 2 \beta} \cdot \left\| \nabla f(\theta_t) \right\|_2^2 \qquad \left( \theta_{t+1} \gets \theta_t - \frac{1}{ \beta } \cdot \nabla f(\theta_t)\right)  \\
    &\le - \frac{1}{ 2 \beta} \cdot C(\theta_t)^2 \cdot \delta(\theta_t)^{2 - 2 \xi} \qquad \left( \text{N\L{}} \right) \\
    &\le - \frac{1}{ 2 \beta} \cdot C^2 \cdot \delta(\theta_t)^{2 - 2 \xi}. \qquad \left( C \coloneqq \inf_{t \ge 1}{ C(\theta_t) } > 0 \right)
\end{align}
According to \cref{lem:auxiliary_lemma_1}, given any $\alpha > 0$, we have, for all $x \in [0, 1]$,
\begin{align}
    \frac{1}{\alpha} \cdot (1 - x^\alpha) \ge x^\alpha \cdot \left( 1 - x \right).
\end{align}
Let $\alpha = 1 - 2 \xi > 0$, since $\xi < 1/2$. Also let $x = \frac{\delta(\theta_{t+1})}{\delta(\theta_{t})} \in (0, 1]$ due to \cref{eq:general_optimization_main_result_1_1a_intermediate_c}. We have,
\begin{align}
\label{eq:general_optimization_main_result_1_1a_intermediate_2}
    \frac{1}{ 1 - 2 \xi } \cdot \left[ 1 - \frac{\delta(\theta_{t+1})^{1 - 2 \xi}}{\delta(\theta_{t})^{1 - 2 \xi}} \right] \ge \frac{\delta(\theta_{t+1})^{1 - 2 \xi}}{\delta(\theta_{t})^{1 - 2 \xi}} \cdot \left[ 1 - \frac{\delta(\theta_{t+1})}{\delta(\theta_{t})} \right].
\end{align}
Next, we have,
\begin{align}
    \frac{1}{\delta(\theta_t)^{ 1 - 2 \xi}} &= \frac{1}{\delta(\theta_1)^{ 1 - 2 \xi}} + \frac{1}{\delta(\theta_t)^{ 1 - 2 \xi}} - \frac{1}{\delta(\theta_1)^{ 1 - 2 \xi}}  \\
    &= \frac{1}{\delta(\theta_1)^{ 1 - 2 \xi}} + \sum_{s=1}^{t-1}{ \left[ \frac{1}{\delta(\theta_{s+1})^{ 1 - 2 \xi}} - \frac{1}{\delta(\theta_{s})^{ 1 - 2 \xi}} \right] } \\
    &= \frac{1}{\delta(\theta_1)^{ 1 - 2 \xi}} + \sum_{s=1}^{t-1}{ \frac{1 - 2 \xi}{\delta(\theta_{s+1})^{ 1 - 2 \xi}} \cdot \frac{1}{ 1 - 2 \xi} \cdot \left[ 1 - \frac{\delta(\theta_{s+1})^{ 1 - 2 \xi}}{\delta(\theta_{s})^{ 1 - 2 \xi}} \right] } \\
    &\ge \frac{1}{\delta(\theta_1)^{ 1 - 2 \xi}} + \sum_{s=1}^{t-1}{ \frac{1 - 2 \xi}{ \bcancel{\delta(\theta_{s+1})^{ 1 - 2 \xi}}} \cdot \frac{\bcancel{\delta(\theta_{s+1})^{1 - 2 \xi}}}{\delta(\theta_{s})^{1 - 2 \xi}} \cdot \left[ 1 - \frac{\delta(\theta_{s+1})}{\delta(\theta_{s})} \right] } \qquad \left( \text{by \cref{eq:general_optimization_main_result_1_1a_intermediate_2}} \right) \\
    &= \frac{1}{\delta(\theta_1)^{ 1 - 2 \xi}} + \sum_{s=1}^{t-1}{ \frac{1 - 2 \xi}{\delta(\theta_{s})^{2 - 2 \xi}} \cdot \left[ \delta(\theta_{s}) - \delta(\theta_{s+1}) \right] } \\
    &\ge \frac{1}{\delta(\theta_1)^{ 1 - 2 \xi}} + \sum_{s=1}^{t-1}{ \frac{1 - 2 \xi}{ \bcancel{\delta(\theta_{s})^{ 2 - 2 \xi}}} \cdot \frac{C^2 }{ 2 \beta} \cdot \bcancel{\delta(\theta_s)^{2 - 2 \xi}} }  \qquad \left( \text{by \cref{eq:general_optimization_main_result_1_1a_intermediate_1}} \right) \\
    &= \frac{1}{\delta(\theta_1)^{ 1 - 2 \xi}} + \frac{(1 - 2 \xi) \cdot C^2}{2 \beta} \cdot (t-1),
\end{align} 
which implies for all $t \ge 1$,
\begin{align}
    f(\theta_t) - f(\theta^*) = \delta(\theta_t) \le \left[ \frac{1}{\left( f(\theta_1) - f(\theta^*) \right)^{ 1 - 2 \xi}} + \frac{(1 - 2 \xi) \cdot C^2}{2 \beta} \cdot (t-1) \right]^{- \frac{1}{1 - 2 \xi} } \in O\left( \frac{1}{ t^{\frac{1}{1 - 2 \xi}} } \right).
\end{align}

\paragraph{(1a) Second part:} \textit{$\Omega(1/t^{\frac{1}{1 - 2 \xi}})$ lower bound for GD update $\theta_{t+1} \gets \theta_t - \eta_t \cdot \nabla f(\theta_t) $ with $\eta_t \in (0, 1]$.}

According to the NS property of \cref{def:non_uniform_smoothness}, we have, for all $\theta$ and $\theta^\prime$,
\begin{align}
    f(\theta^\prime) \le f(\theta) + \nabla f(\theta)^\top \left( \theta^\prime - \theta \right) + \frac{\beta(\theta)}{2} \cdot \left\| \theta^\prime - \theta \right\|_2^2.
\end{align}
Fix $\theta$ and take minimum over $\theta^\prime$ on both sides of the above inequality. Then we have,
\begin{align}
    f(\theta^*) &\le f(\theta) + \min_{\theta^\prime}\left\{ \nabla f(\theta)^\top \left( \theta^\prime - \theta \right) + \frac{\beta(\theta)}{2} \cdot \left\| \theta^\prime - \theta \right\|_2^2 \right\} \\
    &= f(\theta) - \frac{1}{\beta(\theta)} \cdot \left\| \nabla f(\theta) \right\|_2^2 + \frac{1}{2 \cdot \beta(\theta)} \cdot \left\| \nabla f(\theta) \right\|_2^2 \qquad \left( \theta^\prime = \theta - \frac{1}{\beta(\theta)} \cdot \nabla f(\theta) \right) \\
    &= f(\theta) - \frac{1}{2 \cdot \beta(\theta)} \cdot \left\| \nabla f(\theta) \right\|_2^2,
\end{align}
which implies,
\begin{align}
\label{eq:general_optimization_main_result_1_1a_intermediate_3}
    \left\| \nabla f(\theta) \right\|_2^2 &\le 2 \cdot \beta(\theta) \cdot \delta(\theta) \\
    &\le 2 \cdot c \cdot \delta(\theta)^{2 - 2 \xi}. \qquad \left( \beta(\theta) \le c \cdot  \delta(\theta)^{1-2\xi} \right)
\end{align}
Therefore, we have,
\begin{align}
\label{eq:general_optimization_main_result_1_1a_intermediate_4}
    \delta(\theta_t) - \delta(\theta_{t+1}) &= f(\theta_t) - f(\theta_{t+1}) + \nabla f(\theta_t)^\top \left( \theta_{t+1} - \theta_t \right) - \nabla f(\theta_t)^\top \left( \theta_{t+1} - \theta_t \right) \\
    &\le \frac{\beta}{2} \cdot \left\| \theta_{t+1} - \theta_t \right\|_2^2 - \nabla f(\theta_t)^\top \left( \theta_{t+1} - \theta_t \right) \qquad \left(  \text{by NS and } \beta(\theta_t) \le \beta \right) \\
    &= \left( \frac{\beta}{2} \cdot \eta_t^2 + \eta_t \right) \cdot \left\| \nabla f(\theta_t) \right\|_2^2 \qquad \left( \theta_{t+1} \gets \theta_t - \eta_t \cdot \nabla f(\theta_t) \right) \\
    &\le \left( \frac{\beta}{2} \cdot \eta_t^2 + \eta_t \right) \cdot 2 \cdot c \cdot \delta(\theta_t)^{2 - 2 \xi} \qquad \left( \text{by \cref{eq:general_optimization_main_result_1_1a_intermediate_3}} \right) \\
    &\le \left( \beta + 2 \right) \cdot c \cdot \delta(\theta_t)^{2 - 2 \xi}. \qquad \left( \eta_t \in (0, 1] \right)
\end{align}
Next, we show that $\frac{\delta(\theta_{t+1})}{\delta(\theta_t)} \ge \frac{3 - 4 \xi }{ 4 - 4 \xi }$ holds for all large enough $t \ge 1$ by contradiction. According to the upper bound results in the first part, we have $\delta(\theta_t) \to 0$ as $t \to \infty$. Suppose $\frac{\delta(\theta_{t+1})}{\delta(\theta_t)} < \frac{3 - 4 \xi }{ 4 - 4 \xi }$, where $t \ge 1$ is large enough and $\delta(\theta_t)$ is small enough. We have,
\begin{align}
\label{eq:general_optimization_main_result_1_1a_intermediate_5}
    \delta(\theta_{t+1}) &\ge \delta(\theta_t) - \left( \beta + 2 \right) \cdot c \cdot \delta(\theta_t)^{2 - 2 \xi} \qquad \left( \text{by \cref{eq:general_optimization_main_result_1_1a_intermediate_4}} \right) \\
    &> \frac{4 - 4 \xi }{ 3 - 4 \xi } \cdot \delta(\theta_{t+1}) - \left( \beta + 2 \right) \cdot c \cdot \left( \frac{4 - 4 \xi }{ 3 - 4 \xi } \right)^{2 - 2 \xi} \cdot \delta(\theta_{t+1})^{2 - 2 \xi},
\end{align}
where the last inequality is because of the function $f: x \mapsto x - a \cdot x^{2 - 2 \xi}$ with $a > 0$ is monotonically increasing for all $0 < x \le \frac{1}{ \left[ (2 - 2 \xi) a \right]^{1/(1 - 2 \xi)} }$. \cref{eq:general_optimization_main_result_1_1a_intermediate_5} implies that,
\begin{align}
    \delta(\theta_{t+1})^{1 - 2 \xi} > \frac{1}{3 - 4 \xi} \cdot \frac{1}{ (\beta + 2) \cdot c} \cdot \left( \frac{3 - 4 \xi}{ 4 - 4 \xi} \right)^{2 - 2 \xi},
\end{align}
for large enough $t \ge 1$, which is a contradiction with $\delta(\theta_t) \to 0$ as $t \to \infty$. Thus we have $\frac{\delta(\theta_{t+1})}{\delta(\theta_t)} \ge \frac{3 - 4 \xi }{ 4 - 4 \xi }$ holds for all large enough $t \ge 1$. Denote
\begin{align}
    t_0 \coloneqq \min\Big\{t \ge 1: \frac{\delta(\theta_{s+1})}{\delta(\theta_{s})} \ge \frac{3 - 4 \xi }{ 4 - 4 \xi } , \text{ for all } s \ge t \Big\}.
\end{align}
According to \cref{lem:auxiliary_lemma_2}, given any $\alpha > 0$, we have, for all $x \in \left[ \frac{2 \alpha + 1}{2 \alpha + 2}, 1 \right]$,
\begin{align}
    \frac{1}{2 \alpha} \cdot (1 - x^\alpha) \le x^\alpha \cdot \left( 1 - x \right).
\end{align}
Let $\alpha = 1 - 2 \xi > 0$, since $\xi < 1/2$. We have $\frac{2 \alpha + 1}{2 \alpha + 2} = \frac{3 - 4 \xi }{ 4 - 4 \xi }$. Also let $x = \frac{\delta(\theta_{t+1})}{\delta(\theta_{t})} \in \left[ \frac{3 - 4 \xi }{ 4 - 4 \xi }, 1 \right]$. We have,
\begin{align}
\label{eq:general_optimization_main_result_1_1a_intermediate_6}
    \frac{1}{2 \cdot (1 - 2 \xi)} \cdot \left[ 1- \frac{\delta(\theta_{t+1})^{1 - 2 \xi}}{\delta(\theta_{t})^{1 - 2 \xi}} \right] \le \frac{\delta(\theta_{t+1})^{1 - 2 \xi}}{\delta(\theta_{t})^{1 - 2 \xi}} \cdot \left[ 1 - \frac{\delta(\theta_{t+1})}{\delta(\theta_{t})} \right],
\end{align}
for all $t \ge t_0$. On the other hand, since $t_0 \in O(1)$ and $1 - 2 \xi > 0$, we have, for all $t < t_0$,
\begin{align}
    \delta(\theta_{t+1})^{1 - 2 \xi} \ge c_0 > 0.
\end{align}
Next, we have, for all $t \ge t_0$,
\begin{align}
\MoveEqLeft
    \frac{1}{\delta(\theta_t)^{ 1 - 2 \xi}} = \frac{1}{\delta(\theta_1)^{ 1 - 2 \xi}} + \sum_{s=1}^{t-1}{ \left[ \frac{1}{\delta(\theta_{s+1})^{ 1 - 2 \xi}} - \frac{1}{\delta(\theta_{s})^{ 1 - 2 \xi}} \right] } \\
    &= \frac{1}{\delta(\theta_1)^{ 1 - 2 \xi}} + \sum_{s=1}^{t_0-1}{ \frac{1}{\delta(\theta_{s+1})^{ 1 - 2 \xi}} \cdot \left[ 1 - \frac{\delta(\theta_{s+1})^{ 1 - 2 \xi}}{\delta(\theta_{s})^{ 1 - 2 \xi}} \right] } + \sum_{s=t_0}^{t-1}{ \frac{2 \cdot (1 - 2 \xi)}{\delta(\theta_{s+1})^{ 1 - 2 \xi}} \cdot \frac{1}{2 \cdot (1 - 2 \xi)} \cdot \left[ 1 - \frac{\delta(\theta_{s+1})^{ 1 - 2 \xi}}{\delta(\theta_{s})^{ 1 - 2 \xi}} \right] } \\
    &\le \frac{1}{\delta(\theta_1)^{ 1 - 2 \xi}} + \sum_{s=1}^{t_0-1}{ \frac{1}{c_0} \cdot 1 } + \sum_{s=t_0}^{t-1}{ \frac{2 \cdot (1 - 2 \xi)}{\bcancel{\delta(\theta_{s+1})^{ 1 - 2 \xi}}} \cdot \frac{\bcancel{\delta(\theta_{s+1})^{1 - 2 \xi}}}{\delta(\theta_{s})^{1 - 2 \xi}} \cdot \left[ 1 - \frac{\delta(\theta_{s+1})}{\delta(\theta_{s})} \right] } \qquad \left( \text{by \cref{eq:general_optimization_main_result_1_1a_intermediate_6}} \right) \\
    &= \frac{1}{\delta(\theta_1)^{ 1 - 2 \xi}} + \frac{t_0-1}{c_0} +\sum_{s=t_0}^{t-1}{ \frac{2 \cdot (1 - 2 \xi)}{\delta(\theta_{s})^{ 2 - 2 \xi}} \cdot \left[ \delta(\theta_{s}) - \delta(\theta_{s+1}) \right] } \\
    &\le \frac{1}{\delta(\theta_1)^{ 1 - 2 \xi}} + \frac{t_0-1}{c_0} + \sum_{s=t_0}^{t-1}{ \frac{2 \cdot (1 - 2 \xi)}{\bcancel{\delta(\theta_{s})^{ 2 - 2 \xi}}} \cdot \left( \beta + 2 \right) \cdot c \cdot \bcancel{\delta(\theta_s)^{2 - 2 \xi}} } \qquad \left( \text{by \cref{eq:general_optimization_main_result_1_1a_intermediate_4}} \right) \\
    &= \frac{1}{\delta(\theta_1)^{ 1 - 2 \xi}} + \frac{t_0-1}{c_0} + 2 \cdot (1 - 2 \xi) \cdot \left( \beta + 2 \right) \cdot c \cdot (t - t_0),
\end{align}
which implies for all large enough $t \ge 1$,
\begin{align}
    f(\theta_t) - f(\theta^*) = \delta(\theta_t) \ge \left[ \frac{1}{\left( f(\theta_1) - f(\theta^*) \right)^{ 1 - 2 \xi}} + \frac{t_0-1}{c_0} + 2 \cdot (1 - 2 \xi) \cdot \left( \beta + 2 \right) \cdot c \cdot (t-t_0) \right]^{- \frac{1}{1 - 2 \xi} } \in \Omega\left( \frac{1}{ t^{\frac{1}{1 - 2 \xi}} } \right).
\end{align}

\paragraph{(1a) Third part:} \textit{$O(e^{-t})$ upper bound for GNGD update $\theta_{t+1} \gets \theta_t - \frac{ \nabla f(\theta_t) }{ \beta(\theta_t) }  $.} 

We have, for all $t \ge 1$ (or using \cref{lem:descent_lemma_NS_function}),
\begin{align}
    \delta(\theta_{t+1}) - \delta(\theta_t) &= f(\theta_{t+1}) - f(\theta_t) \\
    &\le \nabla f(\theta_t)^\top \left( \theta_{t+1} - \theta_t \right) + \frac{\beta(\theta_t)}{2} \cdot \left\| \theta_{t+1} - \theta_t \right\|_2^2 \qquad \left( \text{NS} \right) \\
    &= - \frac{1}{ \beta(\theta_t) } \cdot \left\| \nabla f(\theta_t) \right\|_2^2 + \frac{1}{2} \cdot \frac{1}{ \beta(\theta_t) } \cdot \left\| \nabla f(\theta_t) \right\|_2^2 \qquad \left( \theta_{t+1} \gets \theta_t - \frac{ \nabla f(\theta_t) }{ \beta(\theta_t) } \right) \\
\label{eq:non_uniform_smoothness_progress}
    &= - \frac{1}{ 2 \cdot \beta(\theta_t) } \cdot \left\| \nabla f(\theta_t) \right\|_2^2 \\
    &\le - \frac{1}{ 2 \cdot \beta(\theta_t) } \cdot C(\theta_t)^2 \cdot \delta(\theta_t)^{2 - 2 \xi} \qquad \left( \text{N\L{}} \right) \\ 
    &\le - \frac{1}{ 2 \cdot \beta(\theta_t) } \cdot C^2 \cdot \delta(\theta_t)^{2 - 2 \xi} \qquad \left( C \coloneqq \inf_{t \ge 1}{ C(\theta_t) } > 0 \right) \\
    &\le - \frac{C^2}{ 2 \cdot c } \cdot \delta(\theta_t), \qquad \left( \beta(\theta_t) \le c \cdot  \delta(\theta_t)^{1-2\xi} \right)
\end{align}
which implies for all $t \ge 1$,
\begin{align}
\label{eq:general_optimization_main_result_1_1a_intermediate_7}
    f(\theta_t) - f(\theta^*) = \delta(\theta_t) &\le \left( 1 - C^2 /(2 \cdot c) \right) \cdot \delta(\theta_{t-1}) \\
    &\le \exp\left\{- C^2/(2 \cdot c) \right\} \cdot \delta(\theta_{t-1}) \\
    &\le \exp\left\{- (t-1) \cdot C^2/(2 \cdot c) \right\} \cdot \delta(\theta_1) \\
    &= \exp\left\{- (t-1) \cdot C^2/(2 \cdot c) \right\} \cdot \left( f(\theta_1) - f(\theta^*) \right).
\end{align}

\paragraph{(1b) First part:}
\textit{$O(1/t^{\frac{1}{1 - 2 \xi}})$ upper bound for GD update $\theta_{t+1} \gets \theta_t - \eta \cdot \nabla f(\theta_t) $ with $\eta \in O(1)$.}

Denote $\beta_1 \coloneqq c \cdot \left\| \nabla f(\theta_1) \right\|_2^{\frac{1 - 2 \xi}{1 - \xi}}$. We have $\beta_1 \in (0, \infty)$, since $f$ is differentiable (\cref{def:non_uniform_smoothness}). Using $\eta \le \frac{1}{ \beta_1 }$ and according to \cref{lem:descent_lemma_smooth_function}, we have $\delta(\theta_2) \le \delta(\theta_1)$. Denote $\beta_2 \coloneqq c \cdot \left\| \nabla f(\theta_2) \right\|_2^{\frac{1 - 2 \xi}{1 - \xi}}$. We also have $\beta_2 \in (0, \infty)$. Repeating the update, we generate $\left\{ \theta_t \right\}_{ t \ge 1}$ such that $\delta(\theta_{t+1}) \le \delta(\theta_t)$. Denote 
\begin{align}
    \beta \coloneqq \sup_{t \ge 1}{ \left\{ \beta_t \right\} } = \sup_{t \ge 1}{ \left\{ c \cdot \left\| \nabla f(\theta_2) \right\|_2^{\frac{1 - 2 \xi}{1 - \xi}} \right\} }.
\end{align}
Now we have $0 \le \delta(\theta_{t+1}) \le \delta(\theta_t) \le \cdots \le \delta(\theta_1)$. According to the monotone convergence theorem, $\delta(\theta_t)$ converges to some finite value. And the gradient $\left\| \nabla f(\theta_t) \right\|_2 \to 0$, otherwise a small gradient update can decrease the sub-optimality, which is a contradiction with convergence. Thus we have $\beta \in (\beta_1, \infty)$, since $\beta_t \to 0$ as $t \to \infty$.
Using $\eta = \frac{1}{ \beta}$, we have $\eta \le \frac{1}{ \beta_t }$ holds for all $t \ge 1$, and, 
\begin{align}
    \beta(\theta_t) \le c \cdot \left\| \nabla f(\theta_t) \right\|_2^{\frac{1 - 2 \xi}{1 - \xi}} = \beta_t \le \beta.
\end{align}
Using similar calculations in the first part of (1a), we have the $O(1/t^{\frac{1}{1 - 2 \xi}})$ upper bound.

\paragraph{(1b) Second part:}
\textit{$\Omega(1/t^{\frac{1}{1 - 2 \xi}})$ lower bound for GD update $\theta_{t+1} \gets \theta_t - \eta_t \cdot \nabla f(\theta_t) $ with $\eta_t \in (0, 1]$.}

According to \cref{eq:general_optimization_main_result_1_1a_intermediate_3}, we have,
\begin{align}
\label{eq:general_optimization_main_result_1_1b_intermediate_1}
    \left\| \nabla f(\theta) \right\|_2^2 &\le 2 \cdot \beta(\theta) \cdot \delta(\theta) \\
    &\le 2 \cdot c \cdot \left\| \nabla f(\theta) \right\|_2^{\frac{1 - 2 \xi}{1 - \xi}} \cdot \delta(\theta), \qquad \left( \beta(\theta) \le c \cdot \left\| \nabla f(\theta) \right\|_2^{\frac{1 - 2 \xi}{1 - \xi}} \right)
\end{align}
which is equivalent to,
\begin{align}
\label{eq:general_optimization_main_result_1_1b_intermediate_2}
    \left\| \nabla f(\theta) \right\|_2^2 \le 2 \cdot c_1 \cdot \delta(\theta)^{2 - 2 \xi},
\end{align}
where $c_1 \coloneqq \frac{1}{2} \cdot \left( 2 \cdot c \right)^{2 - 2 \xi}$. According to \cref{eq:general_optimization_main_result_1_1a_intermediate_4}, we have,
\begin{align}
\label{eq:general_optimization_main_result_1_1b_intermediate_3}
    \delta(\theta_t) - \delta(\theta_{t+1}) &\le \left( \frac{\beta}{2} \cdot \eta_t^2 + \eta_t \right) \cdot \left\| \nabla f(\theta_t) \right\|_2^2 \\
    &\le \left( \beta + 2 \right) \cdot c_1 \cdot \delta(\theta_t)^{2 - 2 \xi}. \qquad \left( \text{by \cref{eq:general_optimization_main_result_1_1b_intermediate_2} and } \eta_t \in (0, 1] \right)
\end{align}
Using similar calculations in the second part of (1a), we have the $\Omega(1/t^{\frac{1}{1 - 2 \xi}})$ lower bound.

\paragraph{(1b) Third part:}
\textit{$O(e^{-t})$ upper bound for GNGD update $\theta_{t+1} \gets \theta_t - \frac{ \nabla f(\theta_t) }{ \beta(\theta_t) }  $.} 

According to \cref{lem:descent_lemma_NS_function}, we have, for all $t \ge 1$,
\begin{align}
    \delta(\theta_{t+1}) - \delta(\theta_t) &\le - \frac{1}{ 2 \cdot \beta(\theta_t) } \cdot \left\| \nabla f(\theta_t) \right\|_2^2 \\
    &\le - \frac{1}{ 2 \cdot c } \cdot \left\| \nabla f(\theta_t) \right\|_2^{\frac{1}{1 - \xi}} \qquad \left( \beta(\theta_t) \le c \cdot \left\| \nabla f(\theta_t) \right\|_2^{\frac{1 - 2 \xi}{1 - \xi}} \right) \\
    &\le - \frac{1}{ 2 \cdot c } \cdot C(\theta_t)^{\frac{1}{1 - \xi}} \cdot  \delta(\theta_t) \qquad \left( \text{N\L{}} \right) \\ 
    &\le - \frac{1}{ 2 \cdot c } \cdot C^{\frac{1}{1 - \xi}} \cdot  \delta(\theta_t), \qquad \left( C \coloneqq \inf_{t \ge 1}{ C(\theta_t) } > 0 \right)
\end{align}
which implies (similar to \cref{eq:general_optimization_main_result_1_1a_intermediate_7}),
\begin{align}
    f(\theta_t) - f(\theta^*) = \delta(\theta_t) \le \exp\left\{- (t -1) \cdot C^{\frac{1}{1 - \xi}} / (2 \cdot c) \right\} \cdot \left( f(\theta_1) - f(\theta^*) \right).
\end{align}

\paragraph{(2a) First part:} $O(1/t^{\frac{1}{1 - 2 \xi}})$ upper bound for GD when $\xi < 1/2$.

Similar to the first part of (1b), we denote $\beta_t \coloneqq L_0 + L_1 \cdot \left\| \nabla f(\theta_t) \right\|_2$ and $\beta \coloneqq \sup_{t \ge 1}{ \left\{ \beta_t \right\} } \in ( L_0, \infty)$ since $\left\| \nabla f(\theta_t) \right\|_2 \to 0$ as $t \to \infty$.
Using $\eta = \frac{1}{ \beta}$, we have $\eta \le \frac{1}{ \beta_t }$ holds for all $t \ge 1$ and  $\beta(\theta_t) \le L_0 + L_1 \cdot \left\| \nabla f(\theta_t) \right\|_2 \le \beta$. According to \cref{eq:general_optimization_main_result_1_1a_intermediate_1} and the first part of (1a), we have the $O(1/t^{\frac{1}{1 - 2 \xi}})$ upper bound.

\paragraph{(2a) Second part:} $O(e^{-t})$ upper bound for GD when $\xi = 1/2$.

According to \cref{lem:descent_lemma_smooth_function}, we have, for all $t \ge 1$,
\begin{align}
    \delta(\theta_{t+1}) - \delta(\theta_t) &\le - \frac{1}{ 2 \beta} \cdot \left\| \nabla f(\theta_t) \right\|_2^2 \\
    &\le - \frac{1}{ 2 \beta} \cdot C(\theta_t)^2 \cdot \delta(\theta_t) \qquad \left( \text{N\L{} with } \xi = 1 / 2 \right) \\
    &\le - \frac{1}{ 2 \beta} \cdot C^2 \cdot \delta(\theta_t), \qquad \left( C \coloneqq \inf_{t \ge 1}{ C(\theta_t) } > 0 \right)
\end{align}
which implies (similar to \cref{eq:general_optimization_main_result_1_1a_intermediate_7}),
\begin{align}
    f(\theta_t) - f(\theta^*) = \delta(\theta_t) \le \exp\left\{- (t -1) \cdot C^{2} / (2 \beta) \right\} \cdot \left( f(\theta_1) - f(\theta^*) \right).
\end{align}

\paragraph{(2a) Third part:} $O(1/t^{\frac{1}{1 - 2 \xi}})$ upper bound for GNGD when $\xi < 1/2$.

According to \cref{lem:descent_lemma_NS_function}, we have, for all $t \ge 1$,
\begin{align}
    \delta(\theta_{t+1}) - \delta(\theta_t) &\le - \frac{1}{ 2 \cdot \beta(\theta_t) } \cdot \left\| \nabla f(\theta_t) \right\|_2^2 \\
    &\le - \frac{1}{2} \cdot \frac{ \left\| \nabla f(\theta_t) \right\|_2^2 }{ L_0 + L_1 \cdot \left\| \nabla f(\theta_t) \right\|_2 } \qquad \left( \beta(\theta_t) \le L_0 + L_1 \cdot \left\| \nabla f(\theta_t) \right\|_2 \right) \\
\label{eq:general_optimization_main_result_1_2a_intermediate_1}
    &\le - \frac{1}{2} \cdot \frac{ \left\| \nabla f(\theta_t) \right\|_2^2 }{ L_0 + L_1 \cdot \beta } \qquad \left( \beta \coloneqq \sup_{t \ge 1}{ \left\{ \left\| \nabla f(\theta_t) \right\|_2 \right\} } \in ( \left\| \nabla f(\theta_1) \right\|_2, \infty ) \right) \\
    &\le - \frac{1}{2} \cdot \frac{ C^2 }{ L_0 + L_1 \cdot \beta } \cdot \delta(\theta_t)^{2 - 2 \xi}, \qquad \left( \text{N\L{} and } C \coloneqq \inf_{t \ge 1}{ C(\theta_t) } > 0 \right)
\end{align}
which is similar to \cref{eq:general_optimization_main_result_1_1a_intermediate_1}. Using similar calculations in the first part of (1a), we have the $O(1/t^{\frac{1}{1 - 2 \xi}})$ upper bound.

\paragraph{(2a) Fourth part:} $O(e^{-t})$ upper bound for GNGD when $\xi = 1/2$.

We have, for all $t \ge 1$,
\begin{align}
    \delta(\theta_{t+1}) - \delta(\theta_t) &\le - \frac{1}{2} \cdot \frac{ \left\| \nabla f(\theta_t) \right\|_2^2 }{ L_0 + L_1 \cdot \beta } \qquad \left( \text{by \cref{eq:general_optimization_main_result_1_2a_intermediate_1}} \right) \\
    &\le - \frac{1}{2} \cdot \frac{ C^2 }{ L_0 + L_1 \cdot \beta } \cdot \delta(\theta_t), \qquad \left( \text{N\L{} with } \xi = 1 / 2 \text{ and } C \coloneqq \inf_{t \ge 1}{ C(\theta_t) } > 0 \right)
\end{align}
which implies (similar to \cref{eq:general_optimization_main_result_1_1a_intermediate_7}),
\begin{align}
    f(\theta_t) - f(\theta^*) = \delta(\theta_t) \le \exp\left\{- (t -1) \cdot C^{2} / (2 \cdot (L_0 + L_1 \cdot \beta) ) \right\} \cdot \left( f(\theta_1) - f(\theta^*) \right).
\end{align}

\paragraph{(2b) First part:}
$O(1/t^{\frac{1}{1 - 2 \xi}})$ upper bound for GD when $\xi < 1/2$.

Denote $\beta_t \coloneqq L_0 \cdot \frac{\left\| \nabla f(\theta_t) \right\|^2}{ \delta(\theta_t)^{2 - 2 \xi} } + L_1 \cdot \left\| \nabla f(\theta_t) \right\|_2$ and $\beta \coloneqq \sup_{t \ge 1}{ \left\{ \beta_t \right\} } \in ( \beta_1, \infty)$. According to \cref{eq:general_optimization_main_result_1_1a_intermediate_1} and the first part of (1a), we have the $O(1/t^{\frac{1}{1 - 2 \xi}})$ upper bound.

\paragraph{(2b) Second part:}
$O(e^{-t})$ upper bound for GD when $\xi = 1/2$.

According to \cref{lem:descent_lemma_smooth_function}, we have, for all $t \ge 1$ (same as the second part of (2a)),
\begin{align}
    \delta(\theta_{t+1}) - \delta(\theta_t) &\le - \frac{1}{ 2 \beta} \cdot \left\| \nabla f(\theta_t) \right\|_2^2 \\
    &\le - \frac{1}{ 2 \beta} \cdot C(\theta_t)^2 \cdot \delta(\theta_t) \qquad \left( \text{N\L{} with } \xi = 1 / 2 \right) \\
    &\le - \frac{1}{ 2 \beta} \cdot C^2 \cdot \delta(\theta_t), \qquad \left( C \coloneqq \inf_{t \ge 1}{ C(\theta_t) } > 0 \right)
\end{align}
which implies (similar to \cref{eq:general_optimization_main_result_1_1a_intermediate_7}),
\begin{align}
    f(\theta_t) - f(\theta^*) = \delta(\theta_t) \le \exp\left\{- (t -1) \cdot C^{2} / (2 \beta) \right\} \cdot \left( f(\theta_1) - f(\theta^*) \right).
\end{align}

\paragraph{(2b) Third part:}
$O(1/t^{\frac{1}{1 - 2 \xi}})$ upper bound for GNGD when $\xi < 1/2$.

According to \cref{lem:descent_lemma_NS_function}, we have, for all $t \ge 1$,
\begin{align}
\label{eq:general_optimization_main_result_1_2b_intermediate_1}
\MoveEqLeft
    \delta(\theta_{t+1}) - \delta(\theta_t) \le - \frac{1}{ 2 \cdot \beta(\theta_t) } \cdot \left\| \nabla f(\theta_t) \right\|_2^2 \\
    &\le - \frac{1}{2} \cdot \frac{\left\| \nabla f(\theta_t) \right\|_2^2}{ L_0 \cdot \frac{\left\| \nabla f(\theta_t) \right\|^2}{ \delta(\theta_t)^{2 - 2 \xi} } + L_1 \cdot \left\| \nabla f(\theta_t) \right\|_2  } \qquad \left( \beta(\theta_t) \le L_0 \cdot \frac{\left\| \nabla f(\theta_t) \right\|^2}{ \delta(\theta_t)^{2 - 2 \xi} } + L_1 \cdot \left\| \nabla f(\theta_t) \right\|_2 \right) \\
    &= - \frac{1}{2} \cdot \frac{ \delta(\theta_t)^{2 - 2 \xi} }{ L_0 + L_1 \cdot \frac{ \delta(\theta_t)^{2 - 2 \xi} }{\left\| \nabla f(\theta_t) \right\|_2}  } \\
    &\le - \frac{1}{2} \cdot \frac{ \delta(\theta_t)^{2 - 2 \xi} }{ L_0 + L_1 \cdot \frac{ \delta(\theta_t)^{1 - \xi} }{ C(\theta_t) }  } \qquad \left( \text{N\L{}:} \  \left\| \nabla f(\theta_t) \right\|_2 \ge C(\theta_t) \cdot \delta(\theta_t)^{1 - \xi} \right) \\
    &\le - \frac{1}{2} \cdot \frac{ \delta(\theta_t)^{2 - 2 \xi} }{ L_0 + L_1 \cdot \frac{ \delta(\theta_t)^{1 - \xi} }{ C }  } \qquad \left( C \coloneqq \inf_{t \ge 1}{ C(\theta_t) } > 0 \right) \\
    &\le - \frac{1}{2} \cdot \frac{ \delta(\theta_t)^{2 - 2 \xi} }{ L_0 + L_1 \cdot \frac{ \delta(\theta_1)^{1 - \xi} }{ C }  }, \qquad \left( \delta_{t+1} \le \delta_t, \text{  by \cref{eq:non_uniform_smoothness_progress}} \right)
\end{align}
which is similar to \cref{eq:general_optimization_main_result_1_1a_intermediate_1}. Using similar calculations in the first part of (1a), we have the $O(1/t^{\frac{1}{1 - 2 \xi}})$ upper bound.

\paragraph{(2b) Fourth part:}
$O(e^{-t})$ upper bound for GNGD when $\xi = 1/2$.

We have, for all $t \ge 1$,
\begin{align}
    \delta(\theta_{t+1}) - \delta(\theta_t) &\le - \frac{1}{2} \cdot \frac{ \delta(\theta_t)^{2 - 2 \xi} }{ L_0 + L_1 \cdot \frac{ \delta(\theta_1)^{1 - \xi} }{ C }  } \qquad \left( \delta_{t+1} \le \delta_t \text{ by \cref{eq:general_optimization_main_result_1_2b_intermediate_1}} \right) \\
    &= - \frac{1}{2} \cdot \frac{ \delta(\theta_t) }{ L_0 + L_1 \cdot \frac{ \delta(\theta_1)^{1 / 2} }{ C }  }, \qquad \left( \xi = 1 / 2 \right)
\end{align}
which implies (similar to \cref{eq:general_optimization_main_result_1_1a_intermediate_7}),
\begin{align}
    f(\theta_t) - f(\theta^*) = \delta(\theta_t) &\le \exp\left\{- \frac{C \cdot (t -1) }{ 2 \cdot (L_0 \cdot C + L_1 \cdot \delta(\theta_1)^{1 / 2} ) } \right\} \cdot \left( f(\theta_1) - f(\theta^*) \right) \\
    &\le \exp\left\{- \frac{C \cdot (t -1) }{ 2 \cdot (L_0 + L_1 \cdot \delta(\theta_1)^{1 / 2} ) } \right\} \cdot \left( f(\theta_1) - f(\theta^*) \right). \qquad \left( \text{if } C \le 1 \right) 
\end{align}

\paragraph{(3a)} \textit{$O(e^{-t})$ upper bound for GNGD update when $\xi \in (1/2, 1)$.} 

According to \cref{lem:descent_lemma_NS_function}, we have, for all $t \ge 1$ (same as the third part of (1b)),
\begin{align}
    \delta(\theta_{t+1}) - \delta(\theta_t) &\le - \frac{1}{ 2 \cdot \beta(\theta_t) } \cdot \left\| \nabla f(\theta_t) \right\|_2^2 \\
    &\le - \frac{1}{ 2 \cdot c } \cdot \left\| \nabla f(\theta_t) \right\|_2^{\frac{1}{1 - \xi}} \qquad \left( \beta(\theta_t) \le c \cdot \left\| \nabla f(\theta_t) \right\|_2^{\frac{1 - 2 \xi}{1 - \xi}} \right) \\
    &\le - \frac{1}{ 2 \cdot c } \cdot C(\theta_t)^{\frac{1}{1 - \xi}} \cdot  \delta(\theta_t) \qquad \left( \text{N\L{}} \right) \\ 
    &\le - \frac{1}{ 2 \cdot c } \cdot C^{\frac{1}{1 - \xi}} \cdot  \delta(\theta_t), \qquad \left( C \coloneqq \inf_{t \ge 1}{ C(\theta_t) } > 0 \right)
\end{align}
which implies (similar to \cref{eq:general_optimization_main_result_1_1a_intermediate_7}),
\begin{align}
    f(\theta_t) - f(\theta^*) = \delta(\theta_t) \le \exp\left\{- (t -1) \cdot C^{\frac{1}{1 - \xi}} / (2 \cdot c) \right\} \cdot \left( f(\theta_1) - f(\theta^*) \right).
\end{align}

\paragraph{(3b)} \textit{$O(e^{-t})$ upper bound for GNGD update when $\xi \in (1/2, 1)$.} 
\end{proof}

According to \cref{lem:descent_lemma_NS_function}, we have, for all $t \ge 1$ (same as the third part of (1a)),
\begin{align}
    \delta(\theta_{t+1}) - \delta(\theta_t) &\le - \frac{1}{ 2 \cdot \beta(\theta_t) } \cdot \left\| \nabla f(\theta_t) \right\|_2^2 \\
    &\le - \frac{1}{ 2 \cdot \beta(\theta_t) } \cdot C(\theta_t)^2 \cdot \delta(\theta_t)^{2 - 2 \xi} \qquad \left( \text{N\L{}} \right) \\ 
    &\le - \frac{1}{ 2 \cdot \beta(\theta_t) } \cdot C^2 \cdot \delta(\theta_t)^{2 - 2 \xi} \qquad \left( C \coloneqq \inf_{t \ge 1}{ C(\theta_t) } > 0 \right) \\
    &\le - \frac{C^2}{ 2 \cdot c } \cdot \delta(\theta_t), \qquad \left( \beta(\theta_t) \le c \cdot  \delta(\theta_t)^{1-2\xi} \right)
\end{align}
which implies (similar to \cref{eq:general_optimization_main_result_1_1a_intermediate_7}),
\begin{align}
    f(\theta_t) - f(\theta^*) = \delta(\theta_t) \le \exp\left\{- (t -1) \cdot C^2 / (2 \cdot c) \right\} \cdot \left( f(\theta_1) - f(\theta^*) \right).
\end{align}

\subsection{Function Classes in \cref{fig:non_uniform_function_classes_diagram_table}}
\label{sec:general_optimization_function_class_relation}

\textbf{\cref{prop:function_class_relation}.}
The following results hold: 
\begin{description}
    \item[(1)] $\text{D} \subseteq \text{C}$. If a function satisfies N\L{} with degree $\xi$, then it satisfies N\L{} with degree $\xi^\prime < \xi$.
    \item[(2)] $\text{F} \subseteq \text{D}$. A strongly convex function satisfies N\L{} with $\xi \ge 1/2$.
    \item[(3)] $\text{F} \cap \text{A} = \emptyset$. A strongly convex function cannot satisfy NS with $\beta(\theta) \to 0$ as $\theta, \theta^\prime \to \theta^*$.
    \item[(4)] $\text{E} \subseteq \text{C}$. A (not strongly) convex function satisfies N\L{} with $\xi < 1/2$.
\end{description}
\begin{proof}
(1) $\text{D} \subseteq \text{C}$. Suppose a function $f: \Theta \to \sR$ satisfies N\L{} with $\xi$, i.e., 
\begin{align}
    \left\| \frac{ d f(\theta) }{d \theta} \right\|_2 \ge C(\theta) \cdot \left| f(\theta) - f(\theta^*) \right|^{1-\xi},
\end{align}
where $\xi \in (-\infty, 1]$, and $C(\theta) > 0$ holds for all $\theta \in \Theta$. Let $\xi^\prime < \xi$. If $\left| f(\theta) - f(\theta^*) \right| > 0$, then we have,
\begin{align}
    \left| f(\theta) - f(\theta^*) \right|^{1-\xi} &= \frac{ \left| f(\theta) - f(\theta^*) \right|^{1-\xi^\prime} }{ \left| f(\theta) - f(\theta^*) \right|^{\xi - \xi^\prime} } \\
    &\ge c(\theta) \cdot \left| f(\theta) - f(\theta^*) \right|^{1-\xi^\prime},
\end{align}
where $c(\theta) \coloneqq \frac{1}{  \left| f(\theta) - f(\theta^*) \right|^{\xi - \xi^\prime}  } > 0$, and $c(\theta) \not\to 0$ as $\theta \to \theta^*$ (or $c(\theta) > c > 0$ for all $\theta$ within a finite distance of $\theta^*$). If $\left| f(\theta) - f(\theta^*) \right| = 0$, then it trivially holds that
\begin{align}
    \left| f(\theta) - f(\theta^*) \right|^{1-\xi} \ge \left| f(\theta) - f(\theta^*) \right|^{1-\xi^\prime}.
\end{align}

(2) $\text{F} \subseteq \text{D}$. Suppose a function $f: \Theta \to \sR$ is strongly convex. We have, there exists $\mu > 0$, for all $\theta$, $\theta^\prime \in \Theta$, 
\begin{align}
    f(\theta^\prime) \ge f(\theta) + \nabla{f(\theta)}^\top (\theta^\prime - \theta) + \frac{\mu}{2} \cdot \left\| \theta^\prime - \theta \right\|_2^2.
\end{align}
Fix $\theta$ and take minimum over $\theta^\prime$ on both sides of the above inequality. Then we have,
\begin{align}
    f(\theta^*) &\ge f(\theta) + \min_{\theta^\prime}\left\{ \nabla f(\theta)^\top \left( \theta^\prime - \theta \right) + \frac{\mu}{2} \cdot \left\| \theta^\prime - \theta \right\|_2^2 \right\} \\
    &= f(\theta) - \frac{1}{\mu} \cdot \left\| \nabla f(\theta) \right\|_2^2 + \frac{1}{2 \mu} \cdot \left\| \nabla f(\theta) \right\|_2^2 \qquad \left( \theta^\prime = \theta - \frac{1}{\mu} \cdot \nabla f(\theta) \right) \\
    &= f(\theta) - \frac{1}{2 \mu} \cdot \left\| \nabla f(\theta) \right\|_2^2,
\end{align}
which is equivalent to,
\begin{align}
    \left\| \nabla{f(\theta)} \right\|_2 \ge \sqrt{2 \mu} \cdot \left( f(\theta) - f(\theta^*) \right)^{\frac{1}{2}},
\end{align}
which means $f$ satisfies N\L{} inequality with $\xi = 1/2$.

(3) $\text{F} \cap \text{A} = \emptyset$. Suppose a function $f: \Theta \to \sR$ is strongly convex. There exists $\mu > 0$, for all $\theta \in \Theta$,
\begin{align}
\label{eq:function_class_relation_intermediate_1}
    \left| z^\top \frac{\partial^2 f(\theta)}{\partial \theta^2} z \right| \ge \mu \cdot \left\| z \right\|_2^2,
\end{align}
holds for all vector $z$ that has the same dimension as $\theta$. Next we show $f \not\in \text{A}$. Suppose $f \in \text{A}$. We have,
\begin{align}
    \beta(\theta^*) = \sup_{z}{ \left| z^\top \frac{\partial^2 f(\theta^*)}{\partial (\theta^*)^2} z \right| } = 0,
\end{align}
which is a contradiction with \cref{eq:function_class_relation_intermediate_1}. Therefore $f \not\in \text{A}$, and $\text{F} \cap \text{A} = \emptyset$.

(4) $\text{E} \subseteq \text{C}$. Suppose a function $f: \Theta \to \sR$ is convex. We have, for all $\theta$, $\theta^\prime \in \Theta$,  
\begin{align}
    f(\theta^\prime) \ge f(\theta) + \nabla{f(\theta)}^\top (\theta^\prime - \theta).
\end{align}
Take $\theta^\prime = \theta^*$. We have,
\begin{align}
\label{eq:function_class_relation_intermediate_2}
    f(\theta^*) \ge f(\theta) + \nabla{f(\theta)}^\top (\theta^* - \theta),
\end{align}
which implies,
\begin{align}
    \left\| \nabla{f(\theta)} \right\|_2 &= \frac{1}{\left\| \theta - \theta^* \right\|_2} \cdot \left\| \nabla{f(\theta)} \right\|_2 \cdot \left\| \theta - \theta^* \right\|_2 \\
    &\ge \frac{1}{\left\| \theta - \theta^* \right\|_2} \cdot  \nabla{f(\theta)}^\top (\theta^* - \theta) \qquad \left( \text{by Cauchy-Schwarz} \right) \\
    &\ge \frac{1}{\left\| \theta - \theta^* \right\|_2} \cdot \left( f(\theta) - f(\theta^*) \right), \qquad \left( \text{by \cref{eq:function_class_relation_intermediate_2}} \right)
\end{align}
and $C(\theta) = \frac{1}{\left\| \theta - \theta^* \right\|_2} \not\to 0$ as $\theta \to \theta^*$ (or $C(\theta) > c > 0$ for all $\left\| \theta - \theta^* \right\|_2$ smaller than a finite value, e.g., within a bounded constraint). Therefore $f$ satisfies N\L{} inequality with $\xi = 0$.
\end{proof}

\textbf{\cref{prop:non_empty_convex_function_classes}.}
\begin{description}
    \item[(1)] $\text{ACE} \not= \emptyset$. There exists at least one (not strongly) convex function which satisfies N\L{} with $\xi < 1/2$ and NS with $\beta(\theta) \to 0$ as $\theta, \theta^\prime \to \theta^*$.
    \item[(2)] $\text{ADE} \not= \emptyset$. There exists at least one (not strongly) convex function which satisfies N\L{} with $\xi \ge 1/2$ and NS with $\beta(\theta) \to 0$ as $\theta, \theta^\prime \to \theta^*$.
    \item[(3)] $\text{BCE} \not= \emptyset$. There exists at least one (not strongly) convex function which satisfies N\L{} with $\xi < 1/2$ and NS with $\beta(\theta) \to \beta > 0$ as $\theta, \theta^\prime \to \theta^*$.
    \item[(4)] $\text{BDE} \not= \emptyset$. There exists at least one (not strongly) convex function which satisfies N\L{} with $\xi \ge 1/2$ and NS with $\beta(\theta) \to \beta > 0$ as $\theta, \theta^\prime \to \theta^*$.
    \item[(5)] $\text{BF} \not= \emptyset$. There exists at least one strongly convex function which satisfies NS with $\beta(\theta) \to \beta > 0$ as $\theta, \theta^\prime \to \theta^*$.
\end{description}
\begin{proof}
(1) $\text{ACE} \not= \emptyset$. Consider minimizing the following function $f: \sR \to \sR$,
\begin{align}
    f(x) = x^4.
\end{align}
The second order derivative (Hessian) is $f^{\prime\prime}(x) = 12 \cdot x^2 \ge 0$, which means $f$ is (not strongly) convex. According to Taylor's theorem, we have, for all $x, x^\prime \in \sR$,
\begin{align}
    \left| f(x^\prime) - f(x) - \Big\langle \frac{d f(x)}{d x}, x^\prime - x \Big\rangle \right| &\le \frac{ \left| f^{\prime\prime}(x_\zeta) \right| }{2} \cdot \| \theta^\prime - \theta \|_2^2 \\
    &= \frac{ 12 \cdot x_\zeta^2 }{2} \cdot \| \theta^\prime - \theta \|_2^2,
\end{align}
where $x_\zeta \coloneqq x + \zeta \cdot (x^\prime - x)$ with some $\zeta \in [0,1]$. Thus we have $\beta(x) = 12 \cdot x_\zeta^2 \to 0$ as $x, x^\prime \to 0$. Next, we have,
\begin{align}
    \left| f^\prime(x) \right| = \left| 4 \cdot x^3 \right| = 4 \cdot \left( \left| x \right|^4 \right)^{\frac{3}{4}} = 4 \cdot \left( f(x) - f(0) \right)^{ 1 - \frac{1}{4}},
\end{align}
which means $f$ satisfies N\L{} inequality with $\xi = 1/4 < 1/2$.

(2) $\text{ADE} \not= \emptyset$. Consider minimizing the following function $f: \sR^K \to \sR$,
\begin{align}
    f(\theta) = \KL(y \| \pi_\theta) = \KL(y \| \softmax(\theta) ),
\end{align}
where $y \in \{ 0, 1 \}^K$ is a one-hot vector. We show that $f$ is a (not strongly) convex function. The gradient of $f$ is,
\begin{align}
\label{eq:function_class_relation_intermediate_3}
    \frac{\partial f(\theta)}{\partial \theta} &= \left( \frac{d \pi_\theta }{d \theta} \right)^\top \left( \frac{d \{ \KL(y \| \pi_\theta) \} }{d \pi_\theta} \right) \\
    &= \left( \diagonalmatrix{(\pi_\theta)} - \pi_\theta \pi_\theta^\top \right) \diagonalmatrix{\left( \frac{1}{\pi_\theta} \right)} ( - y ) \\
    &= \pi_\theta - y.
\end{align}
Therefore the Hessian is,
\begin{align}
\label{eq:function_class_relation_intermediate_4}
    \frac{\partial^2 f(\theta)}{\partial \theta^2} = \frac{d \pi_\theta }{d \theta} = \diagonalmatrix{(\pi_\theta)} - \pi_\theta \pi_\theta^\top.
\end{align}
According to \citet[Lemma 22]{mei2020global}, we have,
\begin{align}
    \diagonalmatrix{(\pi_\theta)} - \pi_\theta \pi_\theta^\top \succeq \rvzero,
\end{align}
and the minimum eigenvalue of $\diagonalmatrix{(\pi_\theta)} - \pi_\theta \pi_\theta^\top$ is $0$, which means $f$ is convex but not strongly convex. Next, according to \citet[Lemma 17]{mei2020escaping}, we have,
\begin{align}
\label{eq:function_class_relation_intermediate_5}
    \KL(y \| \pi_\theta) &= \sum_{a}{ y(a) \cdot \log{ \left( \frac{y(a)}{\pi_\theta(a)} \right) } } \\
    &\le \sum_{a}{ y(a) \cdot \left( \frac{y(a)}{\pi_\theta(a)} -1 \right) } \qquad \left( \log{x} \le x - 1 \right) \\
    &= \sum_{a}{ \left(  y(a) - \pi_\theta(a) + \pi_\theta(a) \right) \cdot \frac{y(a) - \pi_\theta(a) }{\pi_\theta(a)} } \\
    &= \sum_{a}{ \frac{ \left( y(a) - \pi_\theta(a) \right)^2 }{ \pi_\theta(a)} } \\
    &\le \frac{1}{ \min_{a}{\pi_\theta(a) } } \cdot \sum_{a}{ \left(  y(a) - \pi_\theta(a) \right)^2 },
\end{align}
which implies,
\begin{align}
    \left\| \frac{\partial f(\theta)}{\partial \theta} \right\|_2 &= \left\| \pi_\theta - y \right\|_2 \qquad \left( \text{by \cref{eq:function_class_relation_intermediate_3}} \right) \\
    &\ge \min_{a}{ \sqrt{ \pi_\theta(a) } } \cdot \left[ \KL(y \| \pi_\theta) -  \KL(y \| y) \right]^{\frac{1}{2}}, \qquad \left( \text{by \cref{eq:function_class_relation_intermediate_5}} \right)
\end{align}
which means $f$ satisfies N\L{} with $\xi = 1/2$. Denote $\theta_\zeta \coloneqq \theta + \zeta \cdot (\theta^\prime - \theta)$ with some $\zeta \in [0,1]$. We have, as $\pi_\theta, \pi_{\theta^\prime} \to y$,
\begin{align}
    \beta(\theta) &= \sup_{z}{ \left| z^\top \frac{\partial^2 f(\theta_\zeta)}{\partial \theta_\zeta^2} z \right| } \\
    &= \sup_{z}{ \left| z^\top \left( \diagonalmatrix{(\pi_{\theta_\zeta})} - \pi_{\theta_\zeta} \pi_{\theta_\zeta}^\top \right)  z \right| } \qquad \left( \text{by \cref{eq:function_class_relation_intermediate_4}} \right) \\
    &\to \sup_{z}{ \left| z^\top \left( \diagonalmatrix{(y)} - y y^\top \right)  z \right| } \\
    &= \sup_{z}{ \left| z^\top \rvzero  z \right| } \qquad \left( y \text{ is one-hot} \right) \\
    &= 0.
\end{align}

(3) $\text{BCE} \not= \emptyset$. Consider the (modified) Huber loss function,
\begin{align}
\label{eq:modified_huber_loss_function}
    f(x) = \begin{cases}
	    x^2, & \text{if } \left| x \right| \le 1, \\
		2 \cdot \left| x \right| - 1, & \text{otherwise}
	\end{cases}
\end{align}
which is a (not strongly) convex function. According to (4) in \cref{prop:function_class_relation}, $f$ satisfies N\L{} inequality with $\xi = 0$. Denote $x_\zeta \coloneqq x + \zeta \cdot (x^\prime - x)$ with some $\zeta \in [0,1]$. We have $\beta(x) = \left| f^{\prime\prime}(x_\zeta) \right| \to 2 > 0$, as $x, x^\prime \to 0$.

(4) $\text{BDE} \not= \emptyset$. Consider minimizing the same function as in (2),
\begin{align}
    f(\theta) = \KL(y \| \pi_\theta) = \KL(y \| \softmax(\theta) ),
\end{align}
where $y \in (0, 1)^K$ is a probability vector with $\min_{a}{ y(a) } > 0$, i.e., $y$ is bounded away from the boundary of probability simplex. As shown in (2), $f$ is (not strongly) convex and $f$ satisfies N\L{} with $\xi = 1/2$. Next, we have,
\begin{align}
    \beta(\theta) &= \sup_{z}{ \left| z^\top \left( \diagonalmatrix{(\pi_{\theta_\zeta})} - \pi_{\theta_\zeta} \pi_{\theta_\zeta}^\top \right)  z \right| } \qquad \left( \text{by \cref{eq:function_class_relation_intermediate_4}} \right) \\
    &\to \sup_{z}{ \left| z^\top \left( \diagonalmatrix{(y)} - y y^\top \right)  z \right| } \\
    &= \sup_{z}{ \left| \expectation_{a \sim y}[z(a)^2] - \left( \expectation_{a \sim y}[z(a)] \right)^2 \right| }\\
    &= \sup_{z}{ \left| \Var_{a \sim y}[z(a)] \right| } > 0.
\end{align}

(5) $\text{BF} \not= \emptyset$. Consider minimizing the following function,
\begin{align}
    f(x) = x^2,
\end{align}
where $x \in \sR$. $f$ is strongly convex, and $\beta(x) = \beta = 2$. Thus $\beta(x) \to 2 > 0$ as $x, x^\prime \to 0$ in \cref{def:non_uniform_smoothness}.
\end{proof}

\textbf{\cref{prop:non_empty_non_convex_function_classes}.}
The following results hold:
\begin{description}
    \item[(1)] $\gW \coloneqq \text{AC} \setminus ( \text{AD} \cup \text{ACE} ) \not= \emptyset$. There exists at least one non-convex function which satisfies N\L{} with $\xi < 1/2$ and NS with $\beta(\theta) \to 0$ as $\theta, \theta^\prime \to \theta^*$.
    \item[(2)] $\gX \coloneqq \text{AD} \setminus \text{ADE}\not= \emptyset$. There exists at least one non-convex function which satisfies N\L{} with $\xi \ge 1/2$ and NS with $\beta(\theta) \to 0$ as $\theta, \theta^\prime \to \theta^*$.
    \item[(3)] $\gY \coloneqq \text{BC} \setminus ( \text{BD} \cup \text{BCE} ) \not= \emptyset$. There exists at least one non-convex function which satisfies N\L{} with $\xi < 1/2$ and NS with $\beta(\theta) \to \beta > 0$ as $\theta, \theta^\prime \to \theta^*$.
    \item[(4)] $\gZ \coloneqq \text{BD} \setminus ( \text{BDE} \cup \text{BF} ) \not= \emptyset$. There exists at least one non-convex function which satisfies N\L{} with $\xi \ge 1/2$ and NS with $\beta(\theta) \to \beta > 0$ as $\theta, \theta^\prime \to \theta^*$.
\end{description}
\begin{proof}
(1) $\gW \coloneqq \text{AC} \setminus ( \text{AD} \cup \text{ACE} ) \not= \emptyset$. Consider maximizing the expected reward,
\begin{align}
    f(\theta) = \pi_\theta^\top r,
\end{align}
where $\pi_\theta = \softmax(\theta)$ and $\theta \in \sR^K$. According to \citet[Proposition 1]{mei2020global}, $f$ is non-concave. According to
\cref{lem:non_uniform_lojasiewicz_softmax_special}, we have,
\begin{align}
    \left\| \frac{d \pi_\theta^\top r}{d \theta} \right\|_2 \ge \pi_\theta(a^*) \cdot ( \pi^* - \pi_\theta )^\top r,
\end{align}
which means $f$ satisfies N\L{} inequality with $\xi = 0$. As shown in \cref{lem:non_uniform_smoothness_softmax_special}, we have $\beta(\theta_\zeta) = 3 \cdot \Big\| \frac{d \pi_{\theta_\zeta}^\top r}{d {\theta_\zeta}} \Big\|_2 $. Therefore, $\beta(\theta_\zeta) \to 0$ as $\pi_\theta, \pi_{\theta^\prime} \to \pi^*$.

(2) $\gX \coloneqq \text{AD} \setminus \text{ADE}\not= \emptyset$. Consider minimizing the function $f: \sR^K \to \sR$,
\begin{align}
    f(\theta) = \left\| \pi_\theta - y \right\|_2^2,
\end{align}
where $\pi_\theta = \softmax(\theta)$, $\theta \in \sR^K$, and $y \in \{ 0, 1 \}$ is a one-hot vector. We show that $f$ is non-convex using one example. Let $y = (1, 0, 0)^\top$. Let $\theta_1 = (0, 0, 0)^\top$, $\pi_{\theta_1} = \softmax(\theta_1) = (1/3, 1/3, 1/3)^\top$, $\theta_2 = (\log{4}, \log{36}, \log{100})^\top$, and $\pi_{\theta_2} = \softmax(\theta_2) = (4/140, 36/140, 100/140)^\top$. We have,
\begin{align}
    f(\theta_1) = \left\| \pi_{\theta_1} - y \right\|_2^2 = \frac{2}{3}, \text{ and } f(\theta_2) = \left\| \pi_{\theta_2} - y \right\|_2^2 = \frac{38}{25}.
\end{align}
Denote $\bar{\theta} = \frac{1}{2} \cdot \left( \theta_1 + \theta_2 \right) = (\log{2}, \log{6}, \log{10})^\top$ we have $\pi_{\bar{\theta}} = \softmax(\bar{\theta}) = \left( 2/18, 6/18, 10/18 \right)^\top$ and
\begin{align}
    f(\bar{\theta}) &= \left\| \pi_{\bar{\theta}} - y \right\|_2^2 = \frac{98}{81}.
\end{align}
Therefore we have,
\begin{align}
    \frac{1}{2} \cdot \left( f(\theta_1) + f(\theta_2) \right) = \frac{82}{75} = \frac{2214}{2025} < \frac{2450}{2025} = \frac{98}{81} = f(\bar{\theta}),
\end{align}
which means $f$ is non-convex. Denote $H(\pi_\theta) \coloneqq \diagonalmatrix(\pi_\theta) - \pi_\theta \pi_\theta^\top$ as the Jacobian of $\theta \mapsto \softmax(\theta)$. We have,
\begin{align}
    \left\| \frac{\partial f(\theta)}{\partial \theta} \right\|_2 &= \left\| \left( \frac{d \pi_\theta }{d \theta} \right)^\top \left( \frac{d f(\theta)}{d \pi_\theta} \right) \right\|_2 \\
    &= 2 \cdot \left\| H(\pi_\theta)  \left( \pi_\theta - y \right) \right\|_2 \\
    &\ge 2 \cdot \min_{a}{ \pi_\theta(a) } \cdot \left\| \pi_\theta - y  \right\|_2 \qquad \left( \text{by \citet[Lemma 23]{mei2020global}} \right) \\
    &= 2 \cdot \min_{a}{ \pi_\theta(a) } \cdot \left[ f(\theta) - f(y) \right]^{\frac{1}{2}},
\end{align}
which means $f$ satisfies N\L{} inequality with $\xi = 1/2$. Denote $ S \coloneqq S(y,\theta)\in \R^{K\times K}$ as 
the second derivative (Hessian) of $f$. We have,
\begin{align}
    S &= \frac{d }{d \theta } \left\{ \frac{d f(\theta)}{d \theta} \right\} \\
    &= \frac{d }{d \theta } \left\{ H(\pi_\theta)  \left( \pi_\theta - y \right) \right\}.
\end{align}
Continuing with our calculation fix $i, j \in [K]$. Then, 
\begin{align}
\label{eq:non_empty_non_convex_function_classes_hessian_mse}
\MoveEqLeft
    S_{(i, j)} = \frac{d \{ \pi_\theta(i) \cdot  \left[ \pi_\theta(i) - y(i) -  \pi_\theta^\top \left( \pi_\theta - y \right) \right] \} }{d \theta(j)} \\
    &= \frac{d \pi_\theta(i) }{d \theta(j)} \cdot \left[ \pi_\theta(i) - y(i) -  \pi_\theta^\top \left( \pi_\theta - y \right) \right] + \pi_\theta(i) \cdot \frac{d \{ \pi_\theta(i) - y(i) -  \pi_\theta^\top \left( \pi_\theta - y \right) \} }{d \theta(j)} \\
    &= (\delta_{ij} \pi_\theta(j) -  \pi_\theta(i) \pi_\theta(j) ) \cdot \left[ \pi_\theta(i) - y(i) -  \pi_\theta^\top \left( \pi_\theta - y \right) \right] \\
    &\quad + \pi_\theta(i) \cdot \left[ \delta_{ij} \pi_\theta(j) -  \pi_\theta(i) \pi_\theta(j) - \pi_\theta(j) \cdot \left( \pi_\theta(j) - y(j) -  \pi_\theta^\top \left( \pi_\theta - y \right) \right) - \pi_\theta(j) \cdot \left( \pi_\theta(j) - \pi_\theta^\top \pi_\theta \right) \right]  \\
    &= \delta_{ij} \pi_\theta(j) \cdot \left[ \pi_\theta(i) - y(i) -  \pi_\theta^\top \left( \pi_\theta - y \right) \right]  - \pi_\theta(i) \pi_\theta(j) \cdot \left[ \pi_\theta(i) - y(i) -  \pi_\theta^\top \left( \pi_\theta - y \right) \right] \\
    &\quad - \pi_\theta(i) \pi_\theta(j) \cdot \left[ \pi_\theta(j) - y(j) -  \pi_\theta^\top \left( \pi_\theta - y \right) \right] + \pi_\theta(i) \pi_\theta(j) \cdot \left[ \delta_{ij} - \pi_\theta(i) - \pi_\theta(j) + \pi_\theta^\top \pi_\theta \right],
\end{align}
where
\begin{align}
\label{eq:delta_ij_notation}
    \delta_{ij} = \begin{cases}
		1, & \text{if } i = j, \\
		0, & \text{otherwise}
	\end{cases}
\end{align}
is Kronecker's $\delta$-function. To show the bound on 
the spectral radius of $S$, pick $z \in \sR^K$. Then,
\begin{align}
\label{eq:non_empty_non_convex_function_classes_spectral_radius_mse}
\MoveEqLeft
    \left| z^\top S z \right| = \left| \sum\limits_{i=1}^{K}{ \sum\limits_{j=1}^{K}{ S_{(i,j)} \cdot z(i) \cdot z(j)} } \right| \\
    &= \Big| \left( H(\pi_\theta)  \left( \pi_\theta - y \right) \right)^\top \left( z \odot z \right) - 2 \cdot \left( H(\pi_\theta)  \left( \pi_\theta - y \right) \right)^\top z \cdot \left( \pi_\theta^\top z \right) \\
    &\quad + \left( \pi_\theta \odot \pi_\theta \right)^\top \left( z \odot z \right) - 2 \cdot \left( \pi_\theta \odot \pi_\theta \right)^\top z \cdot \left( \pi_\theta^\top z \right) + \left( \pi_\theta^\top z \right)^2 \cdot \left( \pi_\theta^\top \pi_\theta \right) \Big|,
\end{align}
where $\odot$ is Hadamard (component-wise) product. We have, as $\pi_\theta \to y$, 
\begin{align}
    \left( H(\pi_\theta)  \left( \pi_\theta - y \right) \right)^\top \left( z \odot z \right) - 2 \cdot \left( H(\pi_\theta)  \left( \pi_\theta - y \right) \right)^\top z \cdot \left( \pi_\theta^\top z \right) &\to \left( H(y)  \rvzero \right)^\top \left( z \odot z \right) -  2 \cdot \left( H(y)  \rvzero \right)^\top z \cdot \left( y^\top z \right) \\
    &= 0.
\end{align}
Since $y$ is one-hot vector, we have, as $\pi_\theta \to y$,
\begin{align}
    \left( \pi_\theta \odot \pi_\theta \right)^\top \left( z \odot z \right) - 2 \cdot \left( \pi_\theta \odot \pi_\theta \right)^\top z \cdot \left( \pi_\theta^\top z \right) + \left( \pi_\theta^\top z \right)^2 \cdot \pi_\theta^\top \pi_\theta &\to y^\top \left( z \odot z \right) - 2 \cdot \left( y^\top z \right)^2 + \left( y^\top z \right)^2 \cdot y^\top y \\
    &= \left( y^\top z \right)^2 - 2 \cdot \left( y^\top z \right)^2 + \left( y^\top z \right)^2 = 0,
\end{align}
which means $\beta(\theta) \to 0$ as $\theta, \theta^\prime \to \theta^*$ in \cref{def:non_uniform_smoothness}.

(3) $\gY \coloneqq \text{BC} \setminus ( \text{BD} \cup \text{BCE} ) \not= \emptyset$. Consider minimizing the function $f: \sR \to \sR$,
\begin{align}
\label{eq:modified_mse_huber_loss_function}
    f(\theta) = \begin{cases}
	    2 \cdot \left( \pi_\theta - \pi_{\theta^*} \right)^2, & \text{if } \left| \pi_\theta - \pi_{\theta^*} \right| \le 0.2, \\
		25 \cdot \left( \pi_\theta - \pi_{\theta^*} \right)^4 + 0.04, & \text{otherwise}
	\end{cases}
\end{align}
where $\theta \in \sR$, $\theta^* = 0$, and $\pi_\theta$ is defined as,
\begin{align}
    \pi_\theta = \sigmoid(\theta) = \frac{1}{1 + e^{- \theta}},
\end{align}
where $\sigmoid: \sR \to (0,1)$ is the sigmoid activation. \cref{fig:example_glm_mse_huber} shows the image of $f$, indicating that $f$ is a non-convex function.
\begin{figure*}[ht]
\centering
\includegraphics[width=0.3\linewidth]{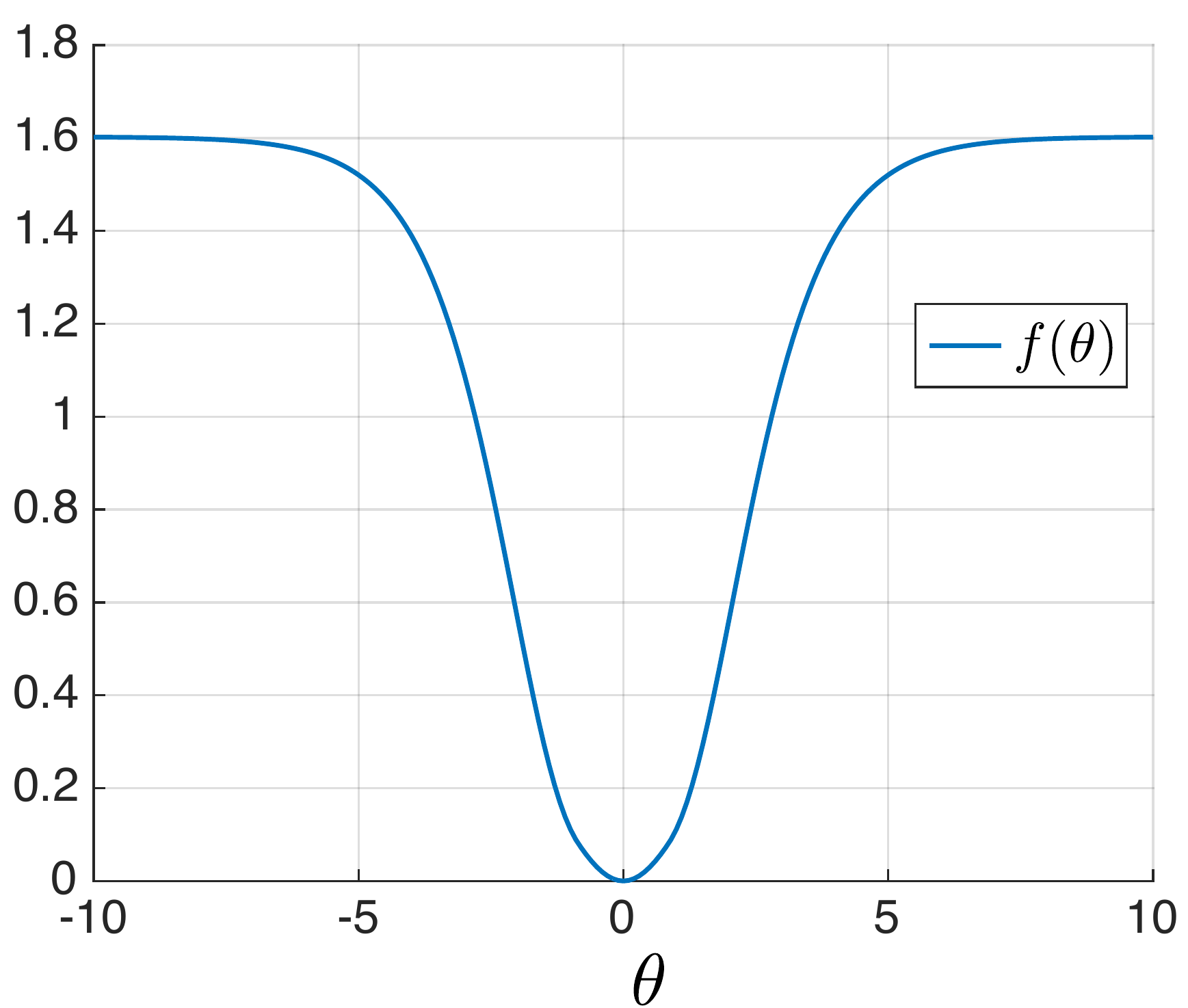}
\vskip -0.1in
\caption{The image of $f$.
} 
\label{fig:example_glm_mse_huber}
\end{figure*}

Since $\theta^* = 0$, we have $\pi_{\theta^*} = 1/2$, and for all $\left| \pi_\theta - \pi_{\theta^*} \right| > 0.2$,
\begin{align}
    \left| \frac{d f(\theta)}{d \theta} \right| &= \left| \frac{d \pi_\theta}{d \theta} \cdot \frac{d f(\theta)}{d \pi_\theta} \right| \\
    &= \left| \pi_\theta \cdot \left( 1 - \pi_\theta \right) \cdot 100 \cdot \left( \pi_\theta - \pi_{\theta^*} \right)^3 \right| \\
    &= 100 \cdot \pi_\theta \cdot \left( 1 - \pi_\theta \right) \cdot \left[ \left( \pi_\theta - \pi_{\theta^*} \right)^4 \right]^{\frac{3}{4}} \\
    &= 100 \cdot \pi_\theta \cdot \left( 1 - \pi_\theta \right) \cdot \left[ f(\theta) - f(\theta^*) \right]^{1 - \frac{1}{4}},
\end{align}
which means $f$ satisfies N\L{} inequality with $\xi = 1/4 < 1/2$. For all $\left| \pi_\theta - \pi_{\theta^*} \right| \le 1$, the Hessian of $f$ is,
\begin{align}
    \left| \frac{d^2 f(\theta)}{d \theta^2} \right| &= \left| \frac{d }{d \theta } \left\{ 100 \cdot \pi_\theta \cdot \left( 1 - \pi_\theta \right) \cdot  \left( \pi_\theta - \pi_{\theta^*} \right) \right\} \right| \\
    &= \left| 100 \cdot \pi_\theta \cdot \left( 1 - \pi_\theta \right) \cdot  \left( \pi_\theta - \pi_{\theta^*} \right) \cdot \left( 1 - 2 \pi_\theta \right) + 100 \cdot \pi_\theta^2 \cdot \left( 1 - \pi_\theta \right)^2 \right|.
\end{align}
As $\pi_\theta \to \pi_{\theta^*} = 1/2$, we have
\begin{align}
    100 \cdot \pi_\theta \cdot \left( 1 - \pi_\theta \right) \cdot  \left( \pi_\theta - \pi_{\theta^*} \right) \cdot \left( 1 - 2 \pi_\theta \right) \to 0,
\end{align}
and,
\begin{align}
    100 \cdot \pi_\theta^2 \cdot \left( 1 - \pi_\theta \right)^2 \to 100 \cdot \frac{1}{4} \cdot \frac{1}{4} = \frac{25}{4} > 0,
\end{align}
which means $\beta(\theta) \to \beta > 0$ as $\theta, \theta^\prime \to \theta^*$ in \cref{def:non_uniform_smoothness}.

(4) $\gZ \coloneqq \text{BD} \setminus ( \text{BDE} \cup \text{BF} ) \not= \emptyset$. Consider minimizing the same function as in (2), 
\begin{align}
    f(\theta) = \left\| \pi_\theta - y \right\|_2^2,
\end{align}
where $\pi_\theta = \softmax(\theta)$, $\theta \in \sR^K$, and $y \in (0, 1)$ is a probability vector with $\min_{a}{ y(a) } > 0$, i.e., $y$ is bounded away from the boundary of probability simplex. We show that $f$ is non-convex using one example. Let $y = (1/2, 1/4, 1/4)^\top$. Let $\theta_1 = (0, 0, 0)^\top$, $\pi_{\theta_1} = \softmax(\theta_1) = (1/3, 1/3, 1/3)^\top$, $\theta_2 = (\log{4}, \log{36}, \log{100})^\top$, and $\pi_{\theta_2} = \softmax(\theta_2) = (4/140, 36/140, 100/140)^\top$. We have,\begin{align}
    f(\theta_1) = \left\| \pi_{\theta_1} - y \right\|_2^2 = \frac{1}{24}, \text{ and } f(\theta_2) = \left\| \pi_{\theta_2} - y \right\|_2^2 = \frac{613}{1400}.
\end{align}
Denote $\bar{\theta} = \frac{1}{2} \cdot \left( \theta_1 + \theta_2 \right) = (\log{2}, \log{6}, \log{10})^\top$ we have $\pi_{\bar{\theta}} = \softmax(\bar{\theta}) = \left( 2/18, 6/18, 10/18 \right)^\top$ and
\begin{align}
    f(\bar{\theta}) &= \left\| \pi_{\bar{\theta}} - y \right\|_2^2 = \frac{163}{648}.
\end{align}
Therefore we have,
\begin{align}
    \frac{1}{2} \cdot \left( f(\theta_1) + f(\theta_2) \right) = \frac{1007}{4200} = \frac{27189}{113400} < \frac{28525}{113400} = \frac{163}{648} = f(\bar{\theta}),
\end{align}
which means $f$ is non-convex. Similar as (2), we have the Hessian of $f$, \begin{align}
    S_{(i,j)}&= \underbrace{ \delta_{ij} \pi_\theta(j) \cdot \left[ \pi_\theta(i) - y(i) -  \pi_\theta^\top \left( \pi_\theta - y \right) \right]}_{\text{(a)}}  -  \underbrace{ \pi_\theta(i) \pi_\theta(j) \cdot \left[ \pi_\theta(i) - y(i) -  \pi_\theta^\top \left( \pi_\theta - y \right) \right] }_{\text{(b)}} \\
    &\quad - \underbrace{ \pi_\theta(i) \pi_\theta(j) \cdot \left[ \pi_\theta(j) - y(j) -  \pi_\theta^\top \left( \pi_\theta - y \right) \right] }_{\text{(c)}} + \underbrace{ \pi_\theta(i) \pi_\theta(j) \cdot \left[ \delta_{ij} - \pi_\theta(i) - \pi_\theta(j) + \pi_\theta^\top \pi_\theta \right] }_{\text{(d)}},
\end{align}
where $(a)=(b)=(c)=0$ when $\pi_{\theta} = y$. Hence, at the optimal point $\theta^*$, we have,
\begin{align}
    S = \frac{1}{128} \cdot \begin{bmatrix} 12 & - 6 & -6\vspace{1ex}\\
    -6 & 7 & -1\vspace{1ex}\\
    -6 & -1 & 7 \end{bmatrix},
\end{align}
and the eigenvalues of $S$ are $0$, $\frac{1}{16}$, and $\frac{9}{64}$. Thus as $\theta, \theta^{\prime}\rightarrow \theta^*$, the Hessian spectral radius of $f$ satisfies $\beta(\theta) \rightarrow\beta=\frac{9}{64}$.
\end{proof}

\textbf{\cref{prop:absulte_power_p}.}
The convex function $f: x \mapsto |x|^p$ with $p > 1$ satisfies the N\L{} inequality with $\xi = 1/p$ and the NS property with $\beta(x) \le c_1 \cdot \delta(x)^{1 - 2 \xi}$.
\begin{proof}
For $p > 1$, $f$ is differentiable, and we have,
\begin{align}
    \left| f^\prime(x) \right| = \left| p \cdot |x|^{p-1} \cdot \sign\{ x \} \right| = p \cdot \left( \left| x \right|^p \right)^{\frac{p - 1}{p}} = p \cdot \left( f(x) - f(0) \right)^{ 1 - \frac{1}{p}},
\end{align}
which means $f$ satisfies N\L{} inequality with $\xi = 1/p$. On the other hand, the Hessian of $f$ is,
\begin{equation*}
    \left|f^{\prime\prime}(x) \right| = \left| p \cdot (p-1)\cdot \left|x\right|^{p-2} \right| = p\cdot(p-1)\cdot(\left|x\right|^p)^{\frac{p-2}{p}} = p\cdot(p-1)\cdot (f(x)-f(0))^{1-\frac{2}{p}}. \qedhere
\end{equation*}
\end{proof}

\section{Proofs for \cref{sec:policy_gradient}}
\label{sec:proofs_policy_gradient}

\subsection{One-state MDPs}
\label{sec:proofs_policy_gradient_one_state_mdps}

\textbf{\cref{lem:non_uniform_lojasiewicz_softmax_special}} (N\L{}) \textbf{.}
Let $a^*$ be the uniqe optimal action.
Denote $\pi^* = \argmax_{\pi \in \Delta}{ \pi^\top r}$. Then, 
\begin{align}
    \left\| \frac{d \pi_\theta^\top r}{d \theta} \right\|_2 \ge \pi_\theta(a^*) \cdot ( \pi^* - \pi_\theta )^\top r.
\end{align}

\begin{proof}
See the proof in \citep[Lemma 3]{mei2020global}. We include a proof for completeness.

Using the expression of the policy gradient, 
\begin{align}
    \left\| \frac{d \pi_\theta^\top r}{d \theta} \right\|_2 &= \left( \sum_{a \in \gA }{\left[ \pi_\theta(a) \cdot (r(a) - \pi_\theta^\top r) \right]^2} \right)^\frac{1}{2} \\
    &\ge  \pi_\theta(a^*) \cdot (r(a^*) - \pi_\theta^\top r). \qedhere
\end{align}
\end{proof}

\textbf{\cref{lem:non_uniform_smoothness_softmax_special}} (NS) \textbf{.}
Let $\pi_\theta = \softmax(\theta)$ and $\pi_{\theta^\prime} = \softmax(\theta^\prime)$. Denote $\theta_\zeta \coloneqq \theta + \zeta \cdot (\theta^\prime - \theta)$ with some $\zeta \in [0,1]$. For any $r \in \left[ 0, 1\right]^K$, $\theta \mapsto \pi_\theta^\top r$ is $\beta(\theta_\zeta)$ non-uniform smooth with $\beta(\theta_\zeta) = 3 \cdot \Big\| \frac{d \pi_{\theta_\zeta}^\top r}{d {\theta_\zeta}} \Big\|_2 $.
\begin{proof}
Let  $S:=S(r,\theta)\in \R^{K\times K}$ be 
the second derivative of the value map $\theta \mapsto \pi_\theta^\top r$.
By Taylor's theorem, it suffices to show that the spectral radius of $S$ is upper bounded. Denote $H(\pi_\theta) \coloneqq \diagonalmatrix(\pi_\theta) - \pi_\theta \pi_\theta^\top$ as the Jacobian of $\theta \mapsto \softmax(\theta)$.
Now, by its definition we have
\begin{align}
    S &= \frac{d }{d \theta } \left\{ \frac{d \pi_\theta^\top r}{d \theta} \right\} \\
    &= \frac{d }{d \theta } \left\{ H(\pi_\theta) r \right\}  \\
    &= \frac{d }{d \theta } \left\{ ( \diagonalmatrix(\pi_\theta) - \pi_\theta \pi_\theta^\top) r \right\}.
\end{align}
Continuing with our calculation fix $i, j \in [K]$. Then, 
\begin{align}
    S_{(i, j)} &= \frac{d \{ \pi_\theta(i) \cdot  ( r(i) - \pi_\theta^\top r ) \} }{d \theta(j)} \\
    &= \frac{d \pi_\theta(i) }{d \theta(j)} \cdot ( r(i) - \pi_\theta^\top r ) + \pi_\theta(i) \cdot \frac{d \{ r(i) - \pi_\theta^\top r \} }{d \theta(j)} \\
    &= (\delta_{ij} \pi_\theta(j) -  \pi_\theta(i) \pi_\theta(j) ) \cdot ( r(i) - \pi_\theta^\top r ) - \pi_\theta(i) \cdot ( \pi_\theta(j) r(j) - \pi_\theta(j) \pi_\theta^\top r ) \\
    &= \delta_{ij} \pi_\theta(j) \cdot ( r(i) - \pi_\theta^\top r ) -  \pi_\theta(i) \pi_\theta(j) \cdot ( r(i) - \pi_\theta^\top r ) - \pi_\theta(i) \pi_\theta(j) \cdot ( r(j) -  \pi_\theta^\top r ),
\end{align}
where
\begin{align}
    \delta_{ij} = \begin{cases}
		1, & \text{if } i = j, \\
		0, & \text{otherwise}
	\end{cases}
\end{align}
is Kronecker's $\delta$-function as defined in \cref{eq:delta_ij_notation}. To show the bound on 
the spectral radius of $S$, pick $y \in \sR^K$. Then,
\begin{align}
\label{eq:non_uniform_smoothness_softmax_special_Hessian_spectral_radius}
    \left| y^\top S y \right| &= \left| \sum\limits_{i=1}^{K}{ \sum\limits_{j=1}^{K}{ S_{(i,j)} \cdot y(i) \cdot y(j)} } \right| \\
    &= \left| \sum_{i}{ \pi_\theta(i) ( r(i) - \pi_\theta^\top r ) y(i)^2 } - 2 \sum_{i} \pi_\theta(i) ( r(i) - \pi_\theta^\top r ) y(i) \sum_{j} \pi_\theta(j) y(j) \right| \\
    &= \left| \left( H(\pi_\theta) r \right)^\top \left( y \odot y \right) - 2 \cdot \left( H(\pi_\theta) r \right)^\top y \cdot \left( \pi_\theta^\top y \right) \right| \\
    &\le \left\| H(\pi_\theta) r \right\|_\infty \cdot \left\| y \odot y \right\|_1 + 2 \cdot \left\| H(\pi_\theta) r \right\|_2 \cdot \left\| y \right\|_2 \cdot \left\| \pi_\theta \right\|_1 \cdot \left\| y \right\|_\infty \\
    &\le 3 \cdot \left\| H(\pi_\theta) r \right\|_2 \cdot \left\| y \right\|_2^2.
\end{align}
According to Taylor's theorem, $\forall \theta, \ \theta^\prime$,
\begin{align}
    \left| ( \pi_{\theta^\prime} - \pi_\theta)^\top r - \Big\langle \frac{d \pi_\theta^\top r}{d \theta}, \theta^\prime - \theta \Big\rangle \right| &= \frac{1}{2} \cdot \left| \left( \theta^\prime - \theta \right)^\top S(r,\theta_\zeta) \left( \theta^\prime - \theta \right) \right| \\
    &\le \frac{3 }{2} \cdot \left\| H(\pi_{\theta_\zeta}) r \right\|_2 \cdot \| \theta^\prime - \theta \|_2^2 \qquad \left( \text{by \cref{eq:non_uniform_smoothness_softmax_special_Hessian_spectral_radius}} \right) \\
    &= \frac{3 }{2} \cdot \bigg\| \frac{d \pi_{\theta_\zeta}^\top r}{d {\theta_\zeta}} \bigg\|_2 \cdot \| \theta^\prime - \theta \|_2^2. \qquad \left( \text{by \cref{lem:policy_gradient_norm_softmax}} \right) \qedhere
\end{align}
\end{proof}

\textbf{\cref{lem:non_uniform_smoothness_intermediate_policy_gradient_norm_special}.}
Let 
\begin{align}
    \theta^\prime = \theta + \eta \cdot \frac{d \pi_{\theta}^\top r}{d {\theta}} \Big/ \left\| \frac{d \pi_{\theta}^\top r}{d {\theta}} \right\|_2.
\end{align}
Denote $\theta_\zeta \coloneqq \theta + \zeta \cdot (\theta^\prime - \theta)$ with some $\zeta \in [0,1]$. We have, for all $\eta \in ( 0 , 1/3)$,
\begin{align}
    \bigg\| \frac{d \pi_{\theta_{\zeta}}^\top r}{d {\theta_{\zeta}}} \bigg\|_2 \le \frac{1}{1 - 3 \eta} \cdot \left\| \frac{d \pi_{\theta}^\top r}{d {\theta}} \right\|_2.
\end{align}
\begin{proof}
Denote $\zeta_1 \coloneqq \zeta$. Also denote $\theta_{\zeta_2} \coloneqq \theta + \zeta_2 \cdot (\theta_{\zeta_1} - \theta)$ with some $\zeta_2 \in [0,1]$. We have,
\begin{align}
\label{eq:non_uniform_smoothness_intermediate_policy_gradient_norm_special_1}
\MoveEqLeft
    \left\| \frac{d \pi_{\theta_{\zeta_1}}^\top r}{d {\theta_{\zeta_1}}} - \frac{d \pi_{\theta}^\top r}{d {\theta}} \right\|_2 = \left\| \int_{0}^{1} \bigg\langle \frac{d^2 \{ \pi_{\theta_{\zeta_2}}^\top r \} }{d {\theta_{\zeta_2}^2 }}, \theta_{\zeta_1} - \theta \bigg\rangle d \zeta_2 \right\|_2 \\
    &\le \int_{0}^{1} \left\| \frac{d^2 \{ \pi_{\theta_{\zeta_2}}^\top r \} }{d {\theta_{\zeta_2}^2 }} \right\|_2 \cdot \left\| \theta_{\zeta_1} - \theta \right\|_2 d \zeta_2 \\
    &\le \int_{0}^{1} 3 \cdot \left\| \frac{d  \pi_{\theta_{\zeta_2}}^\top r }{d {\theta_{\zeta_2} }} \right\|_2 \cdot \zeta_1 \cdot \left\| \theta^\prime - \theta \right\|_2 d \zeta_2 \qquad \left( \text{by \cref{eq:non_uniform_smoothness_softmax_special_Hessian_spectral_radius}} \right)  \\
    &\le \int_{0}^{1} 3 \cdot \left\| \frac{d  \pi_{\theta_{\zeta_2}}^\top r }{d {\theta_{\zeta_2} }} \right\|_2 \cdot \eta \ d \zeta_2, \qquad \left( \zeta_1 \in [0, 1], \text{ using } \theta^\prime = \theta + \eta \cdot \frac{d \pi_{\theta}^\top r}{d {\theta}} \bigg/ \left\| \frac{d \pi_{\theta}^\top r}{d {\theta}} \right\|_2 \right)
\end{align}
where the second last inequality is because of the Hessian is symmetric, and its operator norm is equal to its spectral radius. Therefore we have,
\begin{align}
\label{eq:non_uniform_smoothness_intermediate_policy_gradient_norm_special_2}
    \left\| \frac{d \pi_{\theta_{\zeta_1}}^\top r}{d {\theta_{\zeta_1}}} \right\|_2 &\le \left\| \frac{d \pi_{\theta}^\top r}{d {\theta}} \right\|_2 + \left\| \frac{d \pi_{\theta_{\zeta_1}}^\top r}{d {\theta_{\zeta_1}}} - \frac{d \pi_{\theta}^\top r}{d {\theta}} \right\|_2 \qquad \left( \text{by triangle inequality} \right) \\
    &\le \left\| \frac{d \pi_{\theta}^\top r}{d {\theta}} \right\|_2 + 3 \eta \cdot \int_{0}^{1} \left\| \frac{d  \pi_{\theta_{\zeta_2}}^\top r }{d {\theta_{\zeta_2} }} \right\|_2 d \zeta_2. \qquad \left( \text{by \cref{eq:non_uniform_smoothness_intermediate_policy_gradient_norm_special_1}} \right)
\end{align}
Denote $\theta_{\zeta_3} \coloneqq \theta + \zeta_3 \cdot (\theta_{\zeta_2} - \theta)$ with some $\zeta_3 \in [0,1]$. Using similar calculation as in \cref{eq:non_uniform_smoothness_intermediate_policy_gradient_norm_special_1}, we have,
\begin{align}
\label{eq:non_uniform_smoothness_intermediate_policy_gradient_norm_special_3}
    \left\| \frac{d  \pi_{\theta_{\zeta_2}}^\top r }{d {\theta_{\zeta_2} }} \right\|_2 &\le \left\| \frac{d \pi_{\theta}^\top r}{d {\theta}} \right\|_2 + \left\| \frac{d \pi_{\theta_{\zeta_2}}^\top r}{d {\theta_{\zeta_2}}} - \frac{d \pi_{\theta}^\top r}{d {\theta}} \right\|_2 \\
    &\le \left\| \frac{d \pi_{\theta}^\top r}{d {\theta}} \right\|_2 + 3 \eta \cdot \int_{0}^{1}  \left\| \frac{d  \pi_{\theta_{\zeta_3}}^\top r }{d {\theta_{\zeta_3} }} \right\|_2 d \zeta_3.
\end{align}
Combining \cref{eq:non_uniform_smoothness_intermediate_policy_gradient_norm_special_2,eq:non_uniform_smoothness_intermediate_policy_gradient_norm_special_3}, we have,
\begin{align}
    \left\| \frac{d \pi_{\theta_{\zeta_1}}^\top r}{d {\theta_{\zeta_1}}} \right\|_2 \le \left( 1 + 3 \eta \right) \cdot \left\| \frac{d \pi_{\theta}^\top r}{d {\theta}} \right\|_2 + \left( 3 \eta \right)^2 \cdot \int_{0}^{1} \int_{0}^{1} \left\| \frac{d  \pi_{\theta_{\zeta_3}}^\top r }{d {\theta_{\zeta_3} }} \right\|_2 d \zeta_3 d \zeta_2,
\end{align}
which implies,
\begin{align}
    \left\| \frac{d \pi_{\theta_{\zeta_1}}^\top r}{d {\theta_{\zeta_1}}} \right\|_2 &\le \left[ \sum_{i = 0}^{\infty}{ ( 3 \eta )^i } \right] \cdot \left\| \frac{d \pi_{\theta}^\top r}{d {\theta}} \right\|_2 \\
    &= \frac{1}{1 - 3 \eta} \cdot \left\| \frac{d \pi_{\theta}^\top r}{d {\theta}} \right\|_2. \qquad \left( \eta \in ( 0 , 1/3) \right) \qedhere
\end{align}
\end{proof}

\textbf{\cref{lem:lower_bound_cT_softmax_special}} (Non-vanishing N\L{} coefficient) \textbf{.}
Using normalized policy gradient method, we have $\inf_{t\ge 1} \pi_{\theta_t}(a^*) > 0$.
\begin{proof}
The proof is similar to \citet[Lemma 5]{mei2020global}. Let 
\begin{align}
    c = \frac{K}{2 \Delta} \cdot \left(1 - \frac{\Delta}{K} \right)
\end{align}
and 
\begin{align}
    \Delta = r(a^*) - \max_{a \not= a^*}{ r(a) } > 0
\end{align}
denote the reward gap of $r$. We will prove that $\inf_{t\ge 1} \pi_{\theta_t}(a^*) = \min_{1 \le t \le t_0}{ \pi_{\theta_t}(a^*) }$, where $t_0 =\min\{ t: \pi_{\theta_t}(a^*) \ge \frac{c}{c+1} \}$. 
Note that $t_0$ depends only on $\theta_1$ and $c$, and $c$ depends only on the problem.
Define the following regions,
\begin{align}
    \gR_1 &= \left\{ \theta : \frac{d \pi_\theta^\top r}{d \theta(a^*)} \ge \frac{d \pi_\theta^\top r}{d \theta(a)}, \ \forall a \not= a^* \right\}, \\
    \gR_2 & = \left\{ \theta : \pi_\theta(a^*) \ge \pi_\theta(a), \ \forall a \not= a^* \right\}\,,\\
    \gN_c & = \left\{ \theta : \pi_\theta(a^*) \ge \frac{c}{c+1} \right\}.
\end{align}
We make the following three-part claim.
\begin{claim}\label{cl:regions}
The following hold :
\begin{description}
 \item[a)] \label{cl:regions:a} Following a NPG update $\theta_{t+1} = \theta_t + \eta \cdot \frac{d \pi_{\theta_t}^\top r}{d {\theta_t}} \Big/ \Big\| \frac{d \pi_{\theta_t}^\top r}{d {\theta_t}} \Big\|_2$, if $\theta_{t} \in \gR_1$, then {\em (i)}  $\theta_{t+1} \in \gR_1$ and {\em (ii)} $\pi_{\theta_{t+1}}(a^*) \ge \pi_{\theta_{t}}(a^*)$.
\item[b)] We have $\gR_2\subset \gR_1$ and $\gN_c \subset \gR_1$.
\item[c)] For $\eta=1/6$, there exists a finite time $t_0 \ge 1$, such that $\theta_{t_0} \in \gN_c$, and thus $\theta_{t_0} \in \gR_1$, which implies that $\inf_{t\ge 1} \pi_{\theta_t}(a^*) = \min_{1 \le t \le t_0}{ \pi_{\theta_t}(a^*) }$..
\end{description}
\end{claim}
\paragraph{Claim a)} Part~(i): We want to show that 
if $\theta_{t} \in \gR_1$, then $\theta_{t+1} \in \gR_1$. 
Let 
\begin{align}
\gR_1(a) &= \left\{ \theta : \frac{d \pi_\theta^\top r}{d \theta(a^*)} \ge \frac{d \pi_\theta^\top r}{d \theta(a)}\right\}.
\end{align}

Note that $\gR_1 = \cap_{a\ne a^*} \gR_1(a)$.
Pick $a\ne a^*$. Clearly, it suffices to show that if $\theta_t\in \gR_1(a)$ then $\theta_{t+1}\in \gR_1(a)$.
Hence, suppose that $\theta_t\in \gR_1(a)$.
We consider two cases.

\noindent Case (a): $\pi_{\theta_t}(a^*) \ge \pi_{\theta_t}(a)$.
Since $\pi_{\theta_t}(a^*) \ge \pi_{\theta_t}(a)$, we also have $\theta_t(a^*) \ge \theta_t(a)$. 
After an update of the parameters, \begin{align}
\theta_{t+1}(a^{*}) &=\theta_{t}(a^{*}) +\frac{\eta}{\Big\| \frac{d \pi_{\theta_t}^\top r}{d {\theta_t}} \Big\|_2}\cdot\frac{d\pi_{\theta_{t}}^\top r}{d\theta_{t}(a^{*})}\\
&\geq \theta_{t}(a) +\frac{\eta}{\Big\| \frac{d \pi_{\theta_t}^\top r}{d {\theta_t}} \Big\|_2}\cdot\frac{d\pi_{\theta_{t}}^\top r}{d\theta_{t}(a)}\\
&= \theta_{t+1}(a),
\end{align}
which implies that $\pi_{\theta_{t+1}}(a^*) \ge \pi_{\theta_{t+1}}(a)$. Since $r(a^*) - \pi_{\theta_{t+1}}^\top r > 0$ and $r(a^*) > r(a)$,
\begin{align}
    \pi_{\theta_{t+1}}(a^*) \cdot \left( r(a^*) - \pi_{\theta_{t+1}}^\top r \right) \ge \pi_{\theta_{t+1}}(a) \cdot \left( r(a) - \pi_{\theta_{t+1}}^\top r \right), 
\end{align}
which is equivalent to $\frac{d \pi_{\theta_{t+1}}^\top r}{d \theta_{t+1}(a^*)} \ge \frac{d \pi_{\theta_{t+1}}^\top r}{d \theta_{t+1}(a)}$, i.e., $\theta_{t+1} \in \gR_1(a)$.

\noindent Case (b): Suppose now that $\pi_{\theta_t}(a^*) < \pi_{\theta_t}(a)$.
First note that for any $\theta$ and $a\ne a^*$, $\theta \in \gR_1(a)$ holds if and only if
\begin{align}
\label{eq:r1acond}
    r(a^*) - r(a) &\ge \left( 1 - \frac{\pi_{\theta}(a^*)}{\pi_{\theta}(a)} \right) \cdot \left( r(a^*) - \pi_{\theta}^\top r \right).
\end{align}
Indeed, from the condition $\frac{d \pi_{\theta}^\top r}{d \theta(a^*)} \ge \frac{d \pi_{\theta}^\top r}{d \theta(a)}$, we get
\begin{align}
    \pi_{\theta}(a^*) \cdot \left( r(a^*) - \pi_{\theta}^\top r \right) &\ge \pi_{\theta}(a) \cdot \left( r(a) - \pi_{\theta}^\top r \right) \\
    &= \pi_{\theta}(a) \cdot \left( r(a^*) - \pi_{\theta}^\top r \right) - \pi_{\theta}(a) \cdot \left( r(a^*) - r(a) \right),
\end{align}
which, after rearranging, is equivalent to \cref{eq:r1acond}.
Hence, it suffices to show that  \cref{eq:r1acond} holds for $\theta_{t+1}$ provided it holds for $\theta_t$. From the latter condition, we get
\begin{align}
    r(a^*) - r(a) \ge 
     \left( 1 - \exp\left\{ \theta_{t}(a^*) - \theta_{t}(a) \right\} \right) \cdot \left( r(a^*) - \pi_{\theta_{t}}^\top r \right).
\end{align}
After an update of the parameters, according to \cref{lem:descent_lemma_NS_function} (or \cref{eq:non_uniform_smoothness_progress_special} below), $\pi_{\theta_{t+1}}^\top r \ge \pi_{\theta_{t}}^\top r$, i.e.,
\begin{align}
    0 < r(a^*) - \pi_{\theta_{t+1}}^\top r \le r(a^*) - \pi_{\theta_{t}}^\top r\,.
\end{align}
On the other hand,
\begin{align}
\theta_{t+1}(a^{*})-\theta_{t+1}(a) 
&= \theta_{t}(a^{*}) +\frac{\eta}{\Big\| \frac{d \pi_{\theta_t}^\top r}{d {\theta_t}} \Big\|_2}\cdot\frac{d\pi_{\theta_{t}}^\top r}{d\theta_{t}(a^{*})} - \theta_{t}(a) - \frac{\eta}{\Big\| \frac{d \pi_{\theta_t}^\top r}{d {\theta_t}} \Big\|_2}\cdot\frac{d\pi_{\theta_{t}}^\top r}{d\theta_{t}(a)}\\
&\geq \theta_{t}(a^{*}) - \theta_{t}(a),
\end{align}
which implies that \begin{align}
    1 - \exp\left\{ \theta_{t+1}(a^*) - \theta_{t+1}(a) \right\} \le 1 - \exp\left\{ \theta_{t}(a^*) - \theta_{t}(a) \right\}.
\end{align}
Furthermore, 
by our assumption that $\pi_{\theta_t}(a^*) < \pi_{\theta_t}(a)$,
we have
$1 - \exp\left\{ \theta_{t}(a^*) - \theta_{t}(a) \right\} = 1 - \frac{\pi_{\theta_{t}}(a^*)}{\pi_{\theta_{t}}(a)} > 0$. Putting things together, we get
\begin{align}
    \left( 1 - \exp\left\{ \theta_{t+1}(a^*) - \theta_{t+1}(a) \right\} \right) \cdot \left( r(a^*) - \pi_{\theta_{t+1}}^\top r \right) &\le \left( 1 - \exp\left\{ \theta_{t}(a^*) - \theta_{t}(a) \right\} \right) \cdot \left( r(a^*) - \pi_{\theta_{t}}^\top r \right) \\
    &\le r(a^*) - r(a),
\end{align}
which is equivalent to
\begin{align}
     \left( 1 - \frac{\pi_{\theta_{t+1}}(a^*)}{\pi_{\theta_{t+1}}(a)} \right) \cdot \left( r(a^*) - \pi_{\theta_{t+1}}^\top r \right) \le r(a^*) - r(a),
\end{align}
and thus by our previous remark, $\theta_{t+1}\in \gR_1(a)$, thus, finishing the proof of part~(i).

Part~(ii): 
Assume again that $\theta_t \in \gR_1$. We want to show that
$\pi_{\theta_{t+1}}(a^*) \ge \pi_{\theta_{t}}(a^*)$. 
Since $\theta_t \in \gR_1$, we have $\frac{d \pi_{\theta_t}^\top r}{d \theta_t(a^*)} \ge \frac{d \pi_{\theta_t}^\top r}{d \theta_t(a)}, \ \forall a \not= a^*$. Hence,
\begin{align}
\MoveEqLeft
    \pi_{\theta_{t+1}}(a^*) = \frac{\exp\left\{ \theta_{t+1}(a^*) \right\}}{ \sum_{a}{ \exp\left\{ \theta_{t+1}(a) \right\}} } \\
    &= \frac{\exp\left\{ \theta_{t}(a^*) + \eta \cdot \frac{d \pi_{\theta_t}^\top r}{d \theta_t(a^*)} \Big/ \Big\| \frac{d \pi_{\theta_t}^\top r}{d {\theta_t}} \Big\|_2 \right\}}{ \sum_{a}{ \exp\left\{ \theta_{t}(a) + \eta \cdot \frac{d \pi_{\theta_t}^\top r}{d \theta_t(a)} \Big/ \Big\| \frac{d \pi_{\theta_t}^\top r}{d {\theta_t}} \Big\|_2 \right\}} } \\
    &\ge \frac{\exp\left\{ \theta_{t}(a^*) + \eta \cdot \frac{d \pi_{\theta_t}^\top r}{d \theta_t(a^*)} \Big/ \Big\| \frac{d \pi_{\theta_t}^\top r}{d {\theta_t}} \Big\|_2 \right\}}{ \sum_{a}{ \exp\left\{ \theta_{t}(a) + \eta \cdot \frac{d \pi_{\theta_t}^\top r}{d \theta_t(a^*)} \Big/ \Big\| \frac{d \pi_{\theta_t}^\top r}{d {\theta_t}} \Big\|_2 \right\}} } \qquad\left(\text{using } \frac{d \pi_{\theta_t}^\top r}{d \theta_t(a^*)} \ge \frac{d \pi_{\theta_t}^\top r}{d \theta_t(a)} \right) \\
    &= \frac{\exp\left\{ \theta_{t}(a^*) \right\}}{ \sum_{a}{ \exp\left\{ \theta_{t}(a) \right\}} } = \pi_{\theta_t}(a^*).
\end{align}

\paragraph{Claim b); Claim c)} 
The proof of those claims are exactly the same as \citet[Lemma 5]{mei2020global}, since they do not involve the update rule.
\end{proof}

\textbf{\cref{thm:final_rates_normalized_softmax_pg_special}.}
Using NPG $\theta_{t+1} = \theta_t + \eta \cdot \frac{d \pi_{\theta_t}^\top r}{d {\theta_t}} \Big/ \Big\| \frac{d \pi_{\theta_t}^\top r}{d {\theta_t}} \Big\|_2$, with $\eta = 1/6$, for all $t \ge 1$, we have, 
\begin{align}
    ( \pi^* - \pi_{\theta_t} )^\top r \le e^{ - \frac{  c \cdot (t-1) }{12} } \cdot \left( \pi^* - \pi_{\theta_{1}}\right)^\top r,
\end{align}
where $c = \inf_{t\ge 1} \pi_{\theta_t}(a^*) > 0$ is from \cref{lem:lower_bound_cT_softmax_special}, and $c$ is a constant that depends on $r$ and $\theta_1$, but not on the time $t$.
\begin{proof}
Denote $\theta_{\zeta_t} \coloneqq \theta_t + \zeta_t \cdot (\theta_{t+1} - \theta_t)$ with some $\zeta_t \in [0,1]$. According to \cref{lem:non_uniform_smoothness_softmax_special},
\begin{align}
\MoveEqLeft
    \left| ( \pi_{\theta_{t+1}} - \pi_{\theta_t})^\top r - \Big\langle \frac{d \pi_{\theta_t}^\top r}{d \theta_t}, \theta_{t+1} - \theta_t \Big\rangle \right| \le \frac{3 }{2} \cdot \bigg\| \frac{d \pi_{\theta_{\zeta_t}}^\top r}{d {\theta_{\zeta_t}}} \bigg\|_2 \cdot \| \theta_{t+1} - \theta_t \|_2^2 \\
    &\le \frac{3 }{2} \cdot \frac{1}{1 - 3 \eta} \cdot \left\| \frac{d \pi_{\theta_t}^\top r}{d {\theta_t}} \right\|_2 \cdot \| \theta_{t+1} - \theta_t \|_2^2, \qquad \left( \eta = 1/6, \text{ by \cref{lem:non_uniform_smoothness_intermediate_policy_gradient_norm_special}} \right)
\end{align}
which implies,
\begin{align}
\label{eq:non_uniform_smoothness_progress_special}
\MoveEqLeft
    \pi_{\theta_t}^\top r - \pi_{\theta_{t+1}}^\top r \le - \Big\langle \frac{d \pi_{\theta_t}^\top r}{d \theta_t}, \theta_{t+1} - \theta_{t} \Big\rangle + \frac{3}{2 \cdot (1 - 3 \eta ) } \cdot \left\| \frac{d \pi_{\theta_t}^\top r}{d {\theta_t}} \right\|_2 \cdot \| \theta_{t+1} - \theta_{t} \|_2^2 \\
    &= - \eta \cdot \left\| \frac{d \pi_{\theta_t}^\top r}{d \theta_t} \right\|_2 +  \frac{3 \cdot \eta^2}{ 2 \cdot (1 - 3 \eta ) }  \cdot \left\| \frac{d \pi_{\theta_t}^\top r}{d \theta_t} \right\|_2 
    \qquad\left(\text{using } \theta_{t+1} = \theta_t + \eta \cdot \frac{d \pi_{\theta_t}^\top r}{d \theta_t} \bigg/ \left\| \frac{d \pi_{\theta_t}^\top r}{d {\theta_t}} \right\|_2 \right) \\
    &= - \frac{1}{12} \cdot \left\| \frac{d \pi_{\theta_t}^\top r}{d \theta_t} \right\|_2 
    \qquad\qquad \left(\text{using } \eta = 1/6 \right) \\
    &\le - \frac{1}{12} \cdot \pi_{\theta_t}(a^*) \cdot ( \pi^* - \pi_{\theta_t} )^\top r 
    \qquad\left(\text{by \cref{lem:non_uniform_lojasiewicz_softmax_special}} \right) \\
    &\le - \frac{1}{12} \cdot \inf_{t\ge 1} \pi_{\theta_t}(a^*) \cdot ( \pi^* - \pi_{\theta_t} )^\top r.
\end{align}
According to \cref{eq:non_uniform_smoothness_progress_special}, we have,
\begin{align}
    \left( \pi^* - \pi_{\theta_{t}}\right)^\top r &\le \left( 1 - \frac{  c  }{12} \right) \cdot \left( \pi^* - \pi_{\theta_{t-1}}\right)^\top r \qquad \left( c \coloneqq \inf_{t\ge 1} \pi_{\theta_t}(a^*) > 0 \right) \\
    &\le \exp{ \left\{ -  c /12 \right\} } \cdot \left( \pi^* - \pi_{\theta_{t-1}}\right)^\top r \\
    &\le \exp{ \left\{ -  (t-1) \cdot c / 12 \right\} } \cdot \left( \pi^* - \pi_{\theta_{1}}\right)^\top r. \qedhere
\end{align}
\end{proof}

\subsection{General MDPs}
\label{sec:proofs_policy_gradient_general_mdps}

\textbf{\cref{lem:non_uniform_lojasiewicz_softmax_general}} (N\L{}) \textbf{.}
Denote $S \coloneqq | \gS| $ as the total number of states. We have, for all $\theta \in \sR^{\gS \times \gA}$,
\begin{align}
    \left\| \frac{\partial V^{\pi_\theta}(\mu)}{\partial \theta }\right\|_2 \ge \frac{ \min_s{ \pi_\theta(a^*(s)|s) } }{ \sqrt{S} \cdot  \left\| d_{\rho}^{\pi^*} / d_{\mu}^{\pi_\theta} \right\|_\infty } \cdot \left( V^*(\rho) - V^{\pi_\theta}(\rho) \right),
\end{align}
where $a^*(s)$ is the action that $\pi^*$ selects in state $s$.
\begin{proof}
See the proof in \citep[Lemma 8]{mei2020global}. We include a proof for completeness. We have,
\begin{align}
\MoveEqLeft
    \left\| \frac{\partial V^{\pi_\theta}(\mu)}{\partial \theta }\right\|_2 = \left[ \sum_{s,a} \left( \frac{\partial V^{\pi_\theta}(\mu)}{\partial \theta(s,a)} \right)^2 \right]^{\frac{1}{2}} \\
    &\ge \left[ \sum_{s} \left( \frac{\partial V^{\pi_\theta}(\mu)}{\partial \theta(s,a^*(s))} \right)^2 \right]^{\frac{1}{2}} \\
    &\ge \frac{1}{\sqrt{S}} \sum_{s} \left| \frac{\partial V^{\pi_\theta}(\mu)}{\partial \theta(s,a^*(s))} \right| \qquad \left(
    \text{by Cauchy-Schwarz, } \| x \|_1 = | \langle \rvone, \ |x| \rangle | \le \| \rvone \|_2 \cdot \| x \|_2 \right) \\
    &= \frac{1}{1-\gamma} \cdot \frac{1}{\sqrt{S}} \sum_{s} \left| d_{\mu}^{\pi_\theta}(s) \cdot \pi_\theta(a^*(s)|s) \cdot A^{\pi_\theta}(s,a^*(s)) \right| \qquad \left( \text{by \cref{lem:policy_gradient_softmax}} \right) \\
    &= \frac{1}{1-\gamma} \cdot \frac{1}{\sqrt{S}} \sum_{s} d_{\mu}^{\pi_\theta}(s) \cdot  \pi_\theta(a^*(s)|s) \cdot \left| A^{\pi_\theta}(s,a^*(s)) \right|. 
    \qquad \left( \text{because } d_{\mu}^{\pi_\theta}(s) \ge 0 \text{ and } \pi_\theta(a^*(s)|s) \ge 0 \right)
\end{align}
Define the distribution mismatch coefficient as $\bigg\| \frac{d_{\rho}^{\pi^*}}{d_{\mu}^{\pi_\theta}} \bigg\|_\infty = \max_{s}{ \frac{d_{\rho}^{\pi^*}(s)}{d_{\mu}^{\pi_\theta}(s)} }$. We have,
\begin{align}
    \left\| \frac{\partial V^{\pi_\theta}(\mu)}{\partial \theta }\right\|_2 &\ge \frac{1}{1-\gamma} \cdot \frac{1}{\sqrt{S}} \sum_{s} \frac{ d_{\mu}^{\pi_\theta}(s) }{ d_{\rho}^{\pi^*}(s) } \cdot  d_{\rho}^{\pi^*}(s) \cdot \pi_\theta(a^*(s)|s) \cdot \left| A^{\pi_\theta}(s,a^*(s)) \right| \\
    &\ge \frac{1}{1-\gamma} \cdot  \frac{1}{\sqrt{S}} \cdot \left\| \frac{d_{\rho}^{\pi^*}}{d_{\mu}^{\pi_\theta}} \right\|_\infty^{-1} \cdot \min_s{ \pi_\theta(a^*(s)|s) } \cdot \sum_s{ d_{\rho}^{\pi^*}(s) \cdot \left| A^{\pi_\theta}(s,a^*(s)) \right| } \\
    &\ge \frac{1}{1-\gamma} \cdot  \frac{1}{\sqrt{S}} \cdot \left\| \frac{d_{\rho}^{\pi^*}}{d_{\mu}^{\pi_\theta}} \right\|_\infty^{-1} \cdot \min_s{ \pi_\theta(a^*(s)|s) } \cdot \sum_s{ d_{\rho}^{\pi^*}(s) \cdot A^{\pi_\theta}(s,a^*(s)) } \\
    &= \frac{1}{\sqrt{S}} \cdot \left\| \frac{d_{\rho}^{\pi^*}}{d_{\mu}^{\pi_\theta}} \right\|_\infty^{-1} \cdot \min_s{ \pi_\theta(a^*(s)|s) } \cdot \frac{1}{1-\gamma} \sum_{s}{ d_{\rho}^{\pi^*}(s) \sum_{a}{\pi^*(a|s) \cdot A^{\pi_\theta}(s,a) } } 
    \\
    &= \frac{1}{\sqrt{S}} \cdot \left\| \frac{d_{\rho}^{\pi^*}}{d_{\mu}^{\pi_\theta}} \right\|_\infty^{-1} \cdot \min_s{ \pi_\theta(a^*(s)|s) } \cdot \left[ V^*(\rho) - V^{\pi_\theta}(\rho) \right],
\end{align}
where the
one but last equality used that $\pi^*$ is deterministic and in state $s$ chooses $a^*(s)$ with probability one,
and the last equality uses the performance difference formula (\cref{lem:performance_difference_general}).
\end{proof}

\textbf{\cref{lem:non_uniform_smoothness_softmax_general}} (NS) \textbf{.}
Let \cref{ass:posinit} hold and denote $\theta_\zeta \coloneqq \theta + \zeta \cdot (\theta^\prime - \theta)$ with some $\zeta \in [0,1]$. $\theta \mapsto V^{\pi_\theta}(\mu)$ satisfies $\beta(\theta_\zeta)$ non-uniform smoothness with
\begin{align}
    \beta(\theta_\zeta) = \left[ 3 + \frac{ 4 \cdot \left( C_\infty - (1 - \gamma) \right) }{ 1 - \gamma } \right] \cdot \sqrt{S} \cdot \left\| \frac{\partial V^{\pi_{\theta_\zeta}}(\mu)}{\partial {\theta_\zeta} }\right\|_2,
\end{align}
where $C_\infty \coloneqq \max_{\pi}{ \left\| \frac{d_{\mu}^{\pi}}{ \mu} \right\|_\infty} \le \frac{1}{ \min_s \mu(s) } < \infty$.
\begin{proof}
The main part is to prove that for all $y \in \sR^{S A}$ and $\theta$,
\begin{align}
\MoveEqLeft
    \left| y^\top \frac{\partial^2 V^{\pi_\theta}(\mu)}{\partial \theta^2} y \right| \le \left[ 3 + \frac{ 4 \cdot  \left( C_\infty - (1 - \gamma) \right) }{ 1 - \gamma } \right] \cdot \sqrt{S} \cdot \left\| \frac{\partial V^{\pi_\theta}(\mu)}{\partial \theta }\right\|_2 \cdot \| y \|_2^2.
\end{align}
We first calculate the second order derivative of $V^{\pi_\theta}(\mu)$ w.r.t. $\theta$.

Denote $\theta_\alpha = \theta + \alpha u$, where $\alpha \in \sR$ and $u \in \sR^{SA}$. For any $(s, a) \in \gS \times \gA$,
\begin{align}
\label{eq:non_uniform_smoothness_softmax_general_intermediate_pi_first_derivative}
    \frac{\partial \pi_{\theta_\alpha}(a | s)}{\partial \alpha} \Big|_{\alpha=0} &= \Big\langle \frac{\partial \pi_{\theta_\alpha}(a | s)}{\partial \theta_\alpha} \Big|_{\alpha=0}, \frac{\partial \theta_\alpha}{\partial \alpha} \Big\rangle \\
    &= \Big\langle \frac{\partial \pi_{\theta}(a | s)}{\partial \theta}, u \Big\rangle \\
    &= \Big\langle \frac{\partial \pi_{\theta}(a | s)}{\partial \theta(s, \cdot)}, u(s, \cdot) \Big\rangle. \qquad \left( \frac{\partial \pi_\theta(a | s)}{\partial \theta(s^\prime, \cdot)} = \rvzero, \ \forall s^\prime \not= s \right)
\end{align}
Similarly, for any $(s, a) \in \gS \times \gA$,
\begin{align}
\label{eq:non_uniform_smoothness_softmax_general_intermediate_pi_second_derivative}
    \frac{\partial^2 \pi_{\theta_\alpha}(a | s)}{\partial \alpha^2} \Big|_{\alpha=0} &= \Big\langle \frac{\partial}{\partial \theta_\alpha} \left\{ \frac{\partial \pi_{\theta_\alpha}(a|s)}{\partial \alpha} \right\} \Big|_{\alpha=0}, \frac{\partial \theta_\alpha}{\partial \alpha} \Big\rangle \\
    &= \Big\langle \frac{\partial^2 \pi_{\theta_\alpha}(a|s)}{\partial \theta_\alpha^2} \Big|_{\alpha=0} \frac{\partial \theta_\alpha}{\partial \alpha}, \frac{\partial \theta_\alpha}{\partial \alpha} \Big\rangle \\
    &= \Big\langle \frac{\partial^2 \pi_{\theta}(a | s)}{\partial \theta^2(s, \cdot)} u(s, \cdot), u(s, \cdot) \Big\rangle.
\end{align}
Define $\Pi(\alpha) \in \sR^{S \times SA}$ as follows,
\begin{align}
    \Pi(\alpha) \coloneqq \begin{bmatrix} 
    \pi_{\theta_\alpha}( \cdot | 1)^\top & \rvzero^\top & \cdots & \rvzero^\top \\
    \rvzero^\top & \pi_{\theta_\alpha}( \cdot | 2)^\top & \cdots & \rvzero^\top \\
    \vdots & \vdots & \ddots & \vdots \\
    \rvzero^\top & \rvzero^\top & \cdots & \pi_{\theta_\alpha}( \cdot | S)^\top 
    \end{bmatrix}.
\end{align}
Denote $\gP \in \sR^{SA \times S}$ such that,
\begin{align}
    \gP_{{(sa,s^\prime)}} \coloneqq \gP(s^\prime | s, a).
\end{align}
Define $P(\alpha) \coloneqq \Pi(\alpha) \gP \in \sR^{S \times S}$, where $\forall ( s, s^\prime )$,
\begin{align}
\label{eq:non_uniform_smoothness_softmax_general_intermediate_P_PI_def}
   \left[ P(\alpha) \right]_{( s, s^\prime )} = \sum_{a}{ \pi_{\theta_\alpha}(a | s) \cdot \gP(s^\prime | s, a) }.
\end{align}
The derivative w.r.t. $\alpha$ is
\begin{align}
\label{eq:non_uniform_smoothness_softmax_general_intermediate_P_PI_first_derivative_matrix}
    \frac{\partial P(\alpha)}{\partial \alpha} &= \frac{\partial \Pi(\alpha) \gP}{\partial \alpha} = \frac{\partial \Pi(\alpha) }{\partial \alpha} \gP.
\end{align}
And $\forall ( s, s^\prime )$, we have, 
\begin{align}
\label{eq:non_uniform_smoothness_softmax_general_intermediate_P_PI_first_derivative_scalar}
    \left[ \frac{\partial P(\alpha)}{\partial \alpha} \Big|_{\alpha=0} \right]_{( s, s^\prime)} = \sum_{a} \left[ \frac{ \partial \pi_{\theta_\alpha}(a | s) }{\partial \alpha} \Big|_{\alpha=0} \right] \cdot \gP(s^\prime | s, a).
\end{align}
Next, consider the state value function of $\pi_{\theta_\alpha}$,
\begin{align}
    V^{\pi_{\theta_\alpha}}(s) &= \sum_{a}{\pi_{\theta_\alpha}(a | s) \cdot r(s, a) } + \gamma \sum_{a}{ \pi_{\theta_\alpha}(a | s) \sum_{s^\prime}{ \gP(s^\prime | s, a) \cdot V^{\pi_{\theta_\alpha}}(s^\prime) }  },
\end{align}
which implies,
\begin{align}
\label{eq:state_value_bellman_equation_s}
    V^{\pi_{\theta_\alpha}}(s) &= e_{s}^\top M(\alpha) r_{\theta_\alpha} \\
\label{eq:state_value_bellman_equation_rho}
    V^{\pi_{\theta_\alpha}}(\mu) &= \mu^\top M(\alpha) r_{\theta_\alpha},
\end{align}
where
\begin{align}
\label{eq:non_uniform_moothness_softmax_general_intermediate_M_matrix_def}
    M(\alpha) = \left( \identitymatrix - \gamma P(\alpha) \right)^{-1},
\end{align}
and $r_{\theta_\alpha} \in \sR^{S}$ is given by
\begin{align}
\label{eq:non_uniform_moothness_softmax_general_intermediate_reward_def_matrix}
    r_{\theta_\alpha} = \Pi(\alpha) r,
\end{align}
where $r \in \sR^{SA}$.
Taking derivative w.r.t. $\alpha$ in \cref{eq:state_value_bellman_equation_rho},
\begin{align}
    \frac{\partial V^{\pi_{\theta_\alpha}}(\mu)}{\partial \alpha} &= \gamma \cdot \mu^\top M(\alpha) \frac{\partial P(\alpha)}{\partial \alpha} M(\alpha) r_{\theta_\alpha} + \mu^\top M(\alpha) \frac{\partial r_{\theta_\alpha}}{\partial \alpha} \\
    &= \mu^\top M(\alpha) \left[ \gamma \cdot \frac{\partial P(\alpha)}{\partial \alpha} M(\alpha) r_{\theta_\alpha} + \frac{\partial r_{\theta_\alpha}}{\partial \alpha} \right] \\
    &= \mu^\top M(\alpha) \left[ \gamma \cdot \frac{\partial \Pi(\alpha)}{\partial \alpha} \gP M(\alpha) r_{\theta_\alpha} + \frac{\partial \Pi(\alpha)}{\partial \alpha} r \right] \qquad \left( \text{by \cref{eq:non_uniform_smoothness_softmax_general_intermediate_P_PI_first_derivative_matrix,eq:non_uniform_moothness_softmax_general_intermediate_reward_def_matrix}} \right) \\
    &= \mu^\top M(\alpha) \frac{\partial \Pi(\alpha)}{\partial \alpha} Q^{\pi_{\theta_\alpha}},
\end{align}
where $Q^{\pi_{\theta_\alpha}} \in \sR^{SA}$ is the state-action value and it satisfies,
\begin{align}
    Q^{\pi_{\theta_\alpha}} &= r + \gamma \cdot \gP M(\alpha) r_{\theta_\alpha} \\
    &= r + \gamma \cdot \gP V^{\pi_{\theta_\alpha}}. \qquad \left( \text{by \cref{eq:state_value_bellman_equation_s}} \right)
\end{align}
Similarly, taking second derivative w.r.t. $\alpha$,
\begin{align}
\label{eq:non_uniform_smoothness_softmax_general_intermediate_V_second_derivative_def}
\MoveEqLeft
    \frac{\partial^2 V^{\pi_{\theta_\alpha}}(\mu)}{\partial \alpha^2} = 2 \gamma^2 \cdot \mu^\top M(\alpha) \frac{\partial P(\alpha)}{\partial \alpha} M(\alpha) \frac{\partial P(\alpha)}{\partial \alpha} M(\alpha) r_{\theta_\alpha} + \gamma \cdot \mu^\top M(\alpha) \frac{\partial^2 P(\alpha)}{\partial \alpha^2} M(\alpha) r_{\theta_\alpha} \\
    &\qquad + 2 \gamma \cdot \mu^\top M(\alpha) \frac{\partial P(\alpha)}{\partial \alpha} M(\alpha) \frac{\partial r_{\theta_\alpha}}{\partial \alpha} + \mu^\top M(\alpha) \frac{\partial^2 r_{\theta_\alpha}}{\partial \alpha^2} \\
    &= 2 \gamma \cdot \mu^\top M(\alpha) \frac{\partial P(\alpha)}{\partial \alpha} M(\alpha) \frac{\partial \Pi(\alpha)}{\partial \alpha} \left( \gamma \cdot \gP M(\alpha) r_{\theta_\alpha} + r \right) + \mu^\top M(\alpha) \frac{\partial^2 \Pi(\alpha)}{\partial \alpha^2} \left( \gamma \cdot \gP M(\alpha) r_{\theta_\alpha} + r \right) \\
    &= 2 \gamma \cdot \mu^\top M(\alpha) \frac{\partial P(\alpha)}{\partial \alpha} M(\alpha) \frac{\partial \Pi(\alpha)}{\partial \alpha} Q^{\pi_{\theta_\alpha}} + \mu^\top M(\alpha) \frac{\partial^2 \Pi(\alpha)}{\partial \alpha^2} Q^{\pi_{\theta_\alpha}}.
\end{align}
For the last term, we have,
\begin{align}
    \left[ \frac{\partial^2 \Pi(\alpha)}{\partial \alpha^2} Q^{\pi_{\theta_\alpha}} \Big|_{\alpha=0} \right]_{(s)} &= \sum_{a}{ \frac{\partial^2 \pi_{\theta_\alpha}(a | s)}{\partial \alpha^2} \Big|_{\alpha=0} \cdot Q^{\pi_{\theta}}(s, a) } \\
    &= \sum_{a}{ \Big\langle \frac{\partial^2 \pi_{\theta}(a | s)}{\partial \theta^2(s, \cdot)} u(s, \cdot), u(s, \cdot) \Big\rangle \cdot Q^{\pi_{\theta}}(s, a) } \qquad \left( \text{by \cref{eq:non_uniform_smoothness_softmax_general_intermediate_pi_second_derivative}} \right) \\
    &= u(s, \cdot)^\top \left[ \sum_{a}{ \frac{\partial^2 \pi_{\theta}(a | s)}{\partial \theta^2(s, \cdot)} \cdot Q^{\pi_{\theta}}(s, a) } \right] u(s, \cdot).
\end{align}
Let $S(a, \theta) = \frac{\partial^2 \pi_{\theta}(a | s)}{\partial \theta^2(s, \cdot)} \in \sR^{A \times A}$. $\forall i, j \in [A]$, the value of $S(a, \theta)$ is,
\begin{align}
\MoveEqLeft
    S_{(i,j)} = \frac{\partial \{ \delta_{ia} \pi_{\theta}(a | s) - \pi_{\theta}(a | s) \pi_{\theta}(i | s) \}}{\partial \theta(s,j)} \\
    &= \delta_{ia} \cdot \left[ \delta_{ja} \pi_{\theta}(a | s) -  \pi_{\theta}(a | s) \pi_{\theta}(j | s) \right] - \pi_{\theta}(a | s) \cdot \left[\delta_{ij} \pi_{\theta}(j | s) - \pi_{\theta}(i | s) \pi_{\theta}(j | s) \right] - \pi_{\theta}(i | s) \cdot \left[ \delta_{ja} \pi_{\theta}(a | s) - \pi_{\theta}(a | s) \pi_{\theta}(j | s) \right],
\end{align}
where the $\delta$ notation is as defined in \cref{eq:delta_ij_notation}. Then we have,
\begin{align}
\MoveEqLeft
    \left[ \sum_{a}{ \frac{\partial^2 \pi_{\theta}(a | s)}{\partial \theta^2(s, \cdot)} \cdot Q^{\pi_{\theta}}(s, a) } \right]_{(i,j)} = \sum_{ a }{ S_{(i,j)} \cdot Q^{\pi_{\theta}}(s, a) }  \\
    &= \delta_{ij} \cdot \pi_{\theta}(i | s) \cdot \left[ Q^{\pi_{\theta}}(s, i) - V^{\pi_{\theta}}(s) \right] - \pi_{\theta}(i | s) \cdot \pi_{\theta}(j | s) \cdot \left[ Q^{\pi_{\theta}}(s, i) - V^{\pi_{\theta}}(s) \right]  - \pi_{\theta}(i | s) \cdot \pi_{\theta}(j | s) \cdot \left[ Q^{\pi_{\theta}}(s, j) - V^{\pi_{\theta}}(s) \right].
\end{align}
Therefore we have,
\begin{align}
\MoveEqLeft
     \left[ \frac{\partial^2 \Pi(\alpha)}{\partial \alpha^2} Q^{\pi_{\theta_\alpha}} \Big|_{\alpha=0} \right]_{(s)} = \sum_{i=1}^{A}{ \sum_{j=1}^{A}{ u(s, i) \cdot u(s, j) \cdot \left[ \sum_{a}{ \frac{\partial^2 \pi_{\theta}(a | s)}{\partial \theta^2(s, \cdot)} \cdot Q^{\pi_{\theta}}(s, a) } \right]_{(i,j)}  } } \\
     &= \left( H(\pi_{\theta}(\cdot | s)) Q^{\pi_\theta}(s, \cdot ) \right)^\top \left( u(s, \cdot) \odot u(s, \cdot) \right) - 2 \cdot \left[ \left( H(\pi_{\theta}(\cdot | s)) Q^{\pi_\theta}(s, \cdot) \right)^\top u(s, \cdot)\right] \cdot \left( \pi_{\theta}(\cdot | s)^\top u(s, \cdot) \right),
\end{align}
where $H(\pi) \coloneqq \diagonalmatrix(\pi) - \pi \pi^\top$. Combining the above results with \cref{eq:non_uniform_smoothness_softmax_general_intermediate_V_second_derivative_def}, we have,
\begin{align}
\label{eq:non_uniform_smoothness_softmax_general_intermediate_main_1}
\MoveEqLeft
    \left| \mu^\top M(\alpha) \frac{\partial^2 \Pi(\alpha)}{\partial \alpha^2} Q^{\pi_{\theta_\alpha}} \Big|_{\alpha=0} \right| \le \frac{1}{ 1 - \gamma } \cdot \sum_{s}{ d_{\mu}^{\pi_{\theta}}(s) \cdot \left| \left[ \frac{\partial^2 \Pi(\alpha)}{\partial \alpha^2} Q^{\pi_{\theta_\alpha}} \Big|_{\alpha=0} \right]_{(s)} \right| } \qquad \left( \text{by triangle inequality} \right) \\
    &\le \frac{1}{ 1 - \gamma } \cdot \sum_{s}{ d_{\mu}^{\pi_{\theta}}(s) \cdot 3 \cdot \left\| H(\pi_\theta(\cdot | s)) Q^{\pi_\theta}(s,\cdot) \right\|_2 \cdot \left\| u \right\|_2^2 }  \qquad \left( \text{by H{\" o}lder's inequality} \right) \\
    &\le \frac{3 \cdot \sqrt{S} }{ 1 - \gamma }  \cdot \left[ \sum_{s}{d_{\mu}^{\pi_\theta}(s)^2 \cdot \left\| H(\pi_\theta(\cdot | s)) Q^{\pi_\theta}(s,\cdot) \right\|_2^2 } \right]^\frac{1}{2} \cdot \left\| u \right\|_2^2 \qquad \left( \text{by Cauchy-Schwarz} \right) \\
    &= 3 \cdot \sqrt{S} \cdot \left\| \frac{\partial V^{\pi_\theta}(\mu)}{\partial \theta }\right\|_2 \cdot \left\| u \right\|_2^2. \qquad \left( \text{by \cref{lem:policy_gradient_norm_softmax}} \right)
\end{align}
For the first term in \cref{eq:non_uniform_smoothness_softmax_general_intermediate_V_second_derivative_def}, we have,
\begin{align}
\label{eq:non_uniform_smoothness_softmax_general_intermediate_main_2_a}
    \mu^\top M(\alpha) \frac{\partial P(\alpha)}{\partial \alpha} M(\alpha) \frac{\partial \Pi(\alpha)}{\partial \alpha} Q^{\pi_{\theta_\alpha}} \Big|_{\alpha=0} &= \sum_{s^\prime}{ \left[ \mu^\top M(\alpha) \frac{\partial P(\alpha)}{\partial \alpha} \Big|_{\alpha=0} \right]_{(s^\prime)} \cdot \left[ M(\alpha) \frac{\partial \Pi(\alpha)}{\partial \alpha} Q^{\pi_{\theta_\alpha}} \Big|_{\alpha=0} \right]_{(s^\prime)}  },
\end{align}
since,
\begin{align}
    \left( \mu^\top M(\alpha) \frac{\partial P(\alpha)}{\partial \alpha} \right)^\top \in \sR^S, \quad \text{and} \quad M(\alpha) \frac{\partial \Pi(\alpha)}{\partial \alpha} Q^{\pi_{\theta_\alpha}} \in \sR^S.
\end{align}
Next we have,
\begin{align}
\MoveEqLeft
    \left[ M(\alpha) \frac{\partial \Pi(\alpha)}{\partial \alpha} Q^{\pi_{\theta_\alpha}} \Big|_{\alpha=0} \right]_{(s^\prime)} = \frac{1}{ 1 - \gamma } \cdot  \sum_{s}{ d_{s^{\prime}}^{\pi_\theta}(s) \cdot \left[ \frac{\partial \Pi(\alpha)}{\partial \alpha} Q^{\pi_{\theta_\alpha}} \Big|_{\alpha=0} \right]_{(s)} } \qquad \left( \frac{\partial \Pi(\alpha)}{\partial \alpha} Q^{\pi_{\theta_\alpha}} \in \sR^{S} \right) \\
    &= \frac{1}{ 1 - \gamma } \cdot  \sum_{s}{ d_{s^{\prime}}^{\pi_\theta}(s) \cdot \sum_{a}{  \frac{\partial \pi_{\theta_\alpha}(a | s)}{\partial \alpha} \Big|_{\alpha=0} \cdot Q^{\pi_{\theta}}(s, a)  } } \\
    &= \frac{1}{ 1 - \gamma } \cdot  \sum_{s}{ d_{s^{\prime}}^{\pi_\theta}(s) \cdot \sum_{a}{  \Big\langle \frac{\partial \pi_{\theta}(a | s)}{\partial \theta(s, \cdot)}, u(s, \cdot) \Big\rangle \cdot Q^{\pi_{\theta}}(s, a)  } } \qquad \left( \text{by \cref{eq:non_uniform_smoothness_softmax_general_intermediate_pi_first_derivative}} \right) \\
    &= \frac{1}{ 1 - \gamma } \cdot  \sum_{s}{ d_{s^{\prime}}^{\pi_\theta}(s) \cdot  \Big\langle \sum_{a}{  \frac{\partial \pi_{\theta}(a | s)}{\partial \theta(s, \cdot)} \cdot Q^{\pi_{\theta}}(s, a) }, u(s, \cdot) \Big\rangle  } \\
    &= \frac{1}{ 1 - \gamma } \cdot  \sum_{s}{ d_{s^{\prime}}^{\pi_\theta}(s) \cdot  \left( H(\pi_{\theta}(\cdot | s)) Q^{\pi_\theta}(s, \cdot) \right)^\top u(s, \cdot)  }, \qquad \left( H(\pi_\theta) \text{ is the Jacobian of } \theta \mapsto \softmax(\theta) \right)
\end{align}
which implies,
\begin{align}
\label{eq:non_uniform_smoothness_softmax_general_intermediate_main_2_b}
    \left| \left[ M(\alpha) \frac{\partial \Pi(\alpha)}{\partial \alpha} Q^{\pi_{\theta_\alpha}} \Big|_{\alpha=0} \right]_{(s^\prime)} \right| &\le \frac{1}{ 1 - \gamma } \cdot  \sum_{s}{ d_{s^{\prime}}^{\pi_\theta}(s) \cdot  \left\| H(\pi_{\theta}(\cdot | s)) Q^{\pi_\theta}(s, \cdot) \right\|_2 \cdot \left\| u(s, \cdot) \right\|_2  } \\
    &\le \frac{\left\| u \right\|_2}{ 1 - \gamma } \cdot  \sum_{s}{ d_{s^{\prime}}^{\pi_\theta}(s) \cdot  \left\| H(\pi_{\theta}(\cdot | s)) Q^{\pi_\theta}(s, \cdot) \right\|_2   }.
\end{align}
On the other hand,
\begin{align}
\MoveEqLeft
    \left[ \mu^\top M(\alpha) \frac{\partial P(\alpha)}{\partial \alpha} \Big|_{\alpha=0} \right]_{(s^\prime)} = \frac{1}{ 1 - \gamma } \cdot  \sum_{s}{ d_{\mu}^{\pi_\theta}(s) \cdot \left[ \frac{\partial P(\alpha)}{\partial \alpha} \Big|_{\alpha=0} \right]_{(s,s^\prime)} } \qquad \left( \frac{\partial P(\alpha)}{\partial \alpha} \in \sR^{S \times S} \right) \\
    &= \frac{1}{ 1 - \gamma } \cdot \sum_{s}{ d_{\mu}^{\pi_\theta}(s) \cdot \sum_{a} \left[ \frac{ \partial \pi_{\theta_\alpha}(a | s) }{\partial \alpha} \Big|_{\alpha=0} \right] \cdot \gP(s^\prime | s, a) } \qquad \left( \text{by \cref{eq:non_uniform_smoothness_softmax_general_intermediate_P_PI_first_derivative_scalar}} \right) \\
    &= \frac{1}{ 1 - \gamma } \cdot \sum_{s}{ d_{\mu}^{\pi_\theta}(s) \cdot \sum_{a} \Big\langle \frac{\partial \pi_{\theta}(a | s)}{\partial \theta(s, \cdot)}, u(s, \cdot) \Big\rangle \cdot \gP(s^\prime | s, a) } \qquad \left( \text{by \cref{eq:non_uniform_smoothness_softmax_general_intermediate_pi_first_derivative}} \right) \\
    &= \frac{1}{ 1 - \gamma } \cdot \sum_{s}{ d_{\mu}^{\pi_\theta}(s) \cdot \sum_{a}{ \pi_{\theta}(a | s) \cdot \gP(s^\prime | s, a) \cdot \left[ u(s, a) -  \pi_{\theta}(\cdot | s)^\top u(s, \cdot) \right] } },
\end{align}
which implies,
\begin{align}
    \left| \left[ \mu^\top M(\alpha) \frac{\partial P(\alpha)}{\partial \alpha} \Big|_{\alpha=0} \right]_{(s^\prime)} \right| &\le \frac{1}{ 1 - \gamma } \cdot \sum_{s}{ d_{\mu}^{\pi_\theta}(s) \cdot \sum_{a}{ \pi_{\theta}(a | s) \cdot \gP(s^\prime | s, a) \cdot 2 \cdot \left\| u(s, \cdot) \right\|_\infty } } \\
    &\le \frac{ 2 \cdot \left\| u \right\|_2 }{ 1 - \gamma } \cdot  \sum_{s}{ d_{\mu}^{\pi_\theta}(s) \cdot \sum_{a}{ \pi_{\theta}(a | s) \cdot \gP(s^\prime | s, a) } }.
\end{align}
According to
\begin{align}
    d_{\mu}^{\pi_\theta}(s^\prime) = (1 - \gamma) \cdot \mu(s^\prime) + \gamma \cdot \sum_{s}{ d_{\mu}^{\pi_\theta}(s) \cdot \sum_{a}{ \pi_{\theta}(a | s) \cdot \gP(s^\prime | s, a) } }, \quad \forall s^\prime \in \gS
\end{align}
we have,
\begin{align}
\label{eq:non_uniform_smoothness_softmax_general_intermediate_main_2_c}
\MoveEqLeft
    \left| \left[ \mu^\top M(\alpha) \frac{\partial P(\alpha)}{\partial \alpha} \Big|_{\alpha=0} \right]_{(s^\prime)} \right| \le \frac{ 2 \cdot \left\| u \right\|_2 }{ (1 - \gamma) \cdot \gamma } \cdot \left[ d_{\mu}^{\pi_\theta}(s^\prime) - (1 - \gamma) \cdot \mu(s^\prime) \right] \\
    &= \frac{ 2 \cdot \left\| u \right\|_2 }{ (1 - \gamma) \cdot \gamma } \cdot \left[ \frac{d_{\mu}^{\pi_\theta}(s^\prime)}{ \mu(s^\prime)} \cdot \mu(s^\prime) - (1 - \gamma) \cdot \mu(s^\prime) \right] \\
    &\le \frac{ 2 \cdot \left\| u \right\|_2 }{ (1 - \gamma) \cdot \gamma } \cdot \left( C_\infty - (1 - \gamma) \right) \cdot \mu(s^\prime). \qquad \left( C_\infty \coloneqq \max_{\pi}{ \left\| \frac{d_{\mu}^{\pi}}{ \mu} \right\|_\infty} < \left\| \frac{1}{\mu} \right\|_\infty < \infty \right)
\end{align}
Combining \cref{eq:non_uniform_smoothness_softmax_general_intermediate_main_2_a,eq:non_uniform_smoothness_softmax_general_intermediate_main_2_b,eq:non_uniform_smoothness_softmax_general_intermediate_main_2_c}, we have,
\begin{align}
\label{eq:non_uniform_smoothness_softmax_general_intermediate_main_2}
\MoveEqLeft
    \left| \mu^\top M(\alpha) \frac{\partial P(\alpha)}{\partial \alpha} M(\alpha) \frac{\partial \Pi(\alpha)}{\partial \alpha} Q^{\pi_{\theta_\alpha}} \Big|_{\alpha=0} \right| \\
    &\le \sum_{s^\prime}{ \frac{ 2 \cdot \left\| u \right\|_2 }{ (1 - \gamma) \cdot \gamma } \cdot \left( C_\infty - (1 - \gamma) \right) \cdot \mu(s^\prime) } \cdot \frac{\left\| u \right\|_2}{ 1 - \gamma } \cdot  \sum_{s}{ d_{s^{\prime}}^{\pi_\theta}(s) \cdot  \left\| H(\pi_{\theta}(\cdot | s)) Q^{\pi_\theta}(s, \cdot) \right\|_2   } \\
    &=  \frac{ 2 \cdot \left( C_\infty - (1 - \gamma) \right) }{ (1 - \gamma)^2 \cdot \gamma } \cdot \sum_{s}{ d_{\mu}^{\pi_\theta}(s) \cdot  \left\| H(\pi_{\theta}(\cdot | s)) Q^{\pi_\theta}(s, \cdot) \right\|_2   } \cdot \left\| u \right\|_2^2  \\
    &\le \frac{ 2 \cdot  \left( C_\infty - (1 - \gamma) \right) \cdot \sqrt{S} }{ (1 - \gamma)^2 \cdot \gamma } \cdot \left[ \sum_{s}{d_{\mu}^{\pi_\theta}(s)^2 \cdot \left\| H(\pi_\theta(\cdot | s)) Q^{\pi_\theta}(s,\cdot) \right\|_2^2 } \right]^\frac{1}{2} \cdot \left\| u \right\|_2^2  \qquad \left( \text{by Cauchy-Schwarz} \right) \\
    &=  \frac{ 2 \cdot \left( C_\infty - (1 - \gamma) \right) \cdot \sqrt{S} }{ (1 - \gamma)\cdot \gamma } \cdot \left\| \frac{\partial V^{\pi_\theta}(\mu)}{\partial \theta }\right\|_2 \cdot \left\| u \right\|_2^2. \qquad \left( \text{by \cref{lem:policy_gradient_norm_softmax}} \right)
\end{align}
Combining \cref{eq:non_uniform_smoothness_softmax_general_intermediate_V_second_derivative_def,eq:non_uniform_smoothness_softmax_general_intermediate_main_1,eq:non_uniform_smoothness_softmax_general_intermediate_main_2},
\begin{align}
\label{eq:non_uniform_smoothness_softmax_general_intermediate_main_3}
    \left| \frac{\partial^2 V^{\pi_{\theta_\alpha}}(\mu)}{\partial \alpha^2} \Big|_{\alpha=0} \right| &\le \left[ 3 + 2 \gamma \cdot \frac{ 2 \cdot  \left( C_\infty - (1 - \gamma) \right) }{ (1 - \gamma) \cdot \gamma } \right] \cdot \sqrt{S} \cdot \left\| \frac{\partial V^{\pi_\theta}(\mu)}{\partial \theta }\right\|_2 \cdot \left\| u \right\|_2^2,
\end{align}
which implies for all $y \in \sR^{S A}$ and $\theta$,
\begin{align}
\label{eq:non_uniform_smoothness_softmax_general_intermediate_main_4}
\MoveEqLeft
    \left| y^\top \frac{\partial^2 V^{\pi_\theta}(\mu)}{\partial \theta^2} y \right| = \left| \left(\frac{y}{ \| y \|_2 } \right)^\top \frac{\partial^2 V^{\pi_\theta}(\mu)}{\partial \theta^2} \left(\frac{y}{ \| y \|_2 } \right) \right| \cdot \| y \|_2^2 \\
    &\le \max_{\| u \|_2 = 1}{ \left| \Big\langle \frac{\partial^2 V^{\pi_\theta}(\mu)}{\partial \theta^2} u , u \Big\rangle \right| } \cdot \| y \|_2^2 \\
    &= \max_{\| u \|_2 = 1}{ \left| \Big\langle\frac{\partial^2 V^{\pi_{\theta_\alpha}}(\mu)}{\partial {\theta_\alpha^2}} \Big|_{\alpha=0} \frac{\partial \theta_\alpha}{\partial \alpha}, \frac{\partial \theta_\alpha}{\partial \alpha} \Big\rangle \right| } \cdot \| y \|_2^2 \\
    &= \max_{\| u \|_2 = 1}{ \left| \Big\langle \frac{\partial}{\partial \theta_\alpha} \left\{ \frac{\partial V^{\pi_{\theta_\alpha}}(\mu)}{\partial \alpha} \right\} \Big|_{\alpha = 0}, \frac{\partial \theta_\alpha}{\partial \alpha} \Big\rangle \right| } \cdot \| y \|_2^2 \\
    &= \max_{\| u \|_2 = 1}{ \left| \frac{\partial^2 V^{\pi_{\theta_\alpha}}(\mu)}{\partial \alpha^2  } \Big|_{\alpha=0} \right| } \cdot \| y \|_2^2 \\
    &\le \left[ 3 + \frac{ 4 \cdot  \left( C_\infty - (1 - \gamma) \right) }{ 1 - \gamma } \right] \cdot \sqrt{S} \cdot \left\| \frac{\partial V^{\pi_\theta}(\mu)}{\partial \theta }\right\|_2 \cdot \| y \|_2^2. \qquad \left(
    \text{by \cref{eq:non_uniform_smoothness_softmax_general_intermediate_main_3}} \right)
\end{align}
Denote $\theta_{\zeta} = \theta + \zeta ( \theta^\prime - \theta )$, where $\zeta \in [0,1]$. According to Taylor's theorem, $\forall s$, $\forall \theta, \ \theta^\prime$,
\begin{align}
\MoveEqLeft
    \left| V^{\pi_{\theta^\prime}}(\mu) - V^{\pi_{\theta}}(\mu) - \Big\langle \frac{\partial V^{\pi_\theta}(\mu)}{\partial \theta}, \theta^\prime - \theta \Big\rangle \right| = \frac{1}{2} \cdot \left| \left( \theta^\prime - \theta \right)^\top \frac{\partial^2 V^{\pi_{\theta_\zeta}}(\mu)}{\partial \theta_\zeta^2} \left( \theta^\prime - \theta \right) \right| \\
    &\le \frac{ 3 \cdot ( 1 - \gamma)  + 4 \cdot \left( C_\infty - (1 - \gamma) \right) }{ 2 \cdot ( 1 - \gamma) } \cdot \sqrt{S} \cdot \left\| \frac{\partial V^{\pi_{\theta_\zeta}}(\mu)}{\partial {\theta_\zeta} }\right\|_2 \cdot \| \theta^\prime - \theta \|_2^2. \qquad \left( \text{by \cref{eq:non_uniform_smoothness_softmax_general_intermediate_main_4}} \right) \qedhere
\end{align}
\end{proof}

\textbf{\cref{lem:non_uniform_smoothness_intermediate_policy_gradient_norm_general}.}
Let $\eta = \frac{ 1 - \gamma }{ 6 \cdot ( 1 - \gamma)  + 8 \cdot \left( C_\infty - (1 - \gamma) \right) } \cdot \frac{ 1}{ \sqrt{S} }$ and
\begin{align}
    \theta^\prime = \theta + \eta \cdot \frac{\partial V^{\pi_\theta}(\mu)}{\partial \theta} \bigg/ \left\| \frac{\partial V^{\pi_\theta}(\mu)}{\partial \theta} \right\|_2.
\end{align}
Denote $\theta_\zeta \coloneqq \theta + \zeta \cdot (\theta^\prime - \theta)$ with some $\zeta \in [0,1]$. We have,
\begin{align}
    \left\| \frac{\partial V^{\pi_{\theta_\zeta}}(\mu)}{\partial \theta_\zeta} \right\|_2 \le 2 \cdot \left\| \frac{\partial V^{\pi_\theta}(\mu)}{\partial \theta} \right\|_2.
\end{align}
\begin{proof}
Using the similar arguments of \cref{lem:non_uniform_smoothness_intermediate_policy_gradient_norm_special} (replacing $3$ in \cref{lem:non_uniform_lojasiewicz_softmax_special} with $\frac{ 3 \cdot ( 1 - \gamma)  + 4 \cdot \left( C_\infty - (1 - \gamma) \right) }{ 1 - \gamma } \cdot \sqrt{S}$ in \cref{lem:non_uniform_lojasiewicz_softmax_general}), we have the results.
\end{proof}

\textbf{\cref{lem:lower_bound_cT_softmax_general}} (Non-vanishing N\L{} coefficient) \textbf{.}
Let \cref{ass:posinit} hold. We have,
$c:=\inf_{s\in \cS,t\ge 1} \pi_{\theta_t}(a^*(s)|s) > 0$, where $\left\{ \theta_t \right\}_{t\ge 1}$ is generated by \cref{alg:normalized_policy_gradient_softmax}.
\begin{proof}
The proof is similar to \citet[Lemma 9]{mei2020global} and is an extension of the proof for  \cref{lem:lower_bound_cT_softmax_special}. Denote $\Delta^*(s) = Q^*(s, a^*(s)) - \max_{a \not= a^*(s)}{ Q^*(s, a) } > 0$ as the optimal value gap of state $s$, where $a^*(s)$ is the action that the optimal policy selects under state $s$, and $\Delta^* = \min_{s \in \gS}{ \Delta^*(s) } > 0$ as the optimal value gap of the MDP. For each state $s \in \gS$, define the following sets:
\begin{align}
    \gR_1(s) &= \left\{ \theta : \frac{\partial V^{\pi_\theta}(\mu)}{\partial \theta(s, a^*(s))} \ge \frac{\partial V^{\pi_\theta}(\mu)}{\partial \theta(s, a)}, \ \forall a \not= a^* \right\}, \\
    \gR_2(s) &= \left\{ \theta :  Q^{\pi_\theta}(s,a^*(s)) \ge Q^*(s,a^*(s)) - \Delta^*(s) / 2 \right\}, \\
    \gR_3(s) &= \left\{ \theta_t: V^{\pi_{\theta_t}}(s) \ge Q^{\pi_{\theta_t}}(s, a^*(s)) - \Delta^*(s) / 2, 
    \text{ for all } t \ge 1 \text{ large enough} \right\}, \\
    \gN_c(s) &= \left\{ \theta : \pi_\theta(a^*(s) | s) \ge \frac{c(s)}{c(s)+1} \right\}, \text{ where } c(s) = \frac{A}{(1-\gamma) \cdot \Delta^*(s)} - 1.
\end{align}
Similarly to the previous proof, we have the following claims: 
\begin{description}
    \item[Claim I.] $\gR_1(s) \cap \gR_2(s) \cap \gR_3(s)$ is a ``nice" region, in the sense that, following a gradient update, (i) if $\theta_{t} \in \gR_1(s) \cap \gR_2(s) \cap \gR_3(s)$, then $\theta_{t+1} \in \gR_1(s) \cap \gR_2(s) \cap \gR_3(s)$;
    while we also have (ii) $\pi_{\theta_{t+1}}(a^*(s) | s) \ge \pi_{\theta_{t}}(a^*(s) | s)$.
    \item[Claim II.] $\gN_c(s) \cap \gR_2(s) \cap \gR_3(s) \subset \gR_1(s) \cap \gR_2(s) \cap \gR_3(s)$.
    \item[Claim III.] There exists a finite time $t_0(s) \ge 1$, such that $\theta_{t_0(s)} \in \gN_c(s) \cap \gR_2(s) \cap \gR_3(s)$, and thus $\theta_{t_0(s)} \in \gR_1(s) \cap \gR_2(s) \cap \gR_3(s)$, which implies $\inf_{t\ge 1} \pi_{\theta_t}(a^*(s) | s) = \min_{1 \le t \le t_0(s)}{ \pi_{\theta_t}(a^*(s) | s) }$.
    \item[Claim IV.] Define $t_0 = \max_{s}{ t_0(s)}$. Then, we have $\inf_{s\in \cS, t\ge 1} \pi_{\theta_t}(a^*(s)|s) = \min_{1 \le t \le t_0}{ \min_{s} \pi_{\theta_t}(a^*(s) | s) }$. 
\end{description}
Clearly, claim IV suffices to prove the lemma since for any $\theta$, $\min_{s,a}\pi_{\theta}(a|s)>0$.
In what follows we provide the proofs of these four claims.
\paragraph{Claim I.}
First we prove part (i) of the claim. If $\theta_{t} \in \gR_1(s) \cap \gR_2(s) \cap \gR_3(s)$, then $\theta_{t+1} \in \gR_1(s) \cap \gR_2(s) \cap \gR_3(s)$. Suppose $\theta_{t} \in \gR_1(s) \cap \gR_2(s) \cap \gR_3(s)$. We have $\theta_{t+1} \in \gR_3(s)$ by the definition of $\gR_3(s)$. We have,\begin{align}
    Q^{\pi_{\theta_t}}(s,a^*(s)) \ge Q^*(s,a^*(s)) - \Delta^*(s) / 2.
\end{align}
According to monotonic improvement of \cref{eq:ascent_lemma_MDP}, we have $V^{\pi_{\theta_{t+1}}}(s^\prime) \ge V^{\pi_{\theta_t}}(s^\prime)$, and 
\begin{align}
    Q^{\pi_{\theta_{t+1}}}(s,a^*(s)) &= Q^{\pi_{\theta_t}}(s,a^*(s)) + Q^{\pi_{\theta_{t+1}}}(s,a^*(s)) - Q^{\pi_{\theta_t}}(s,a^*(s)) \\
    &= Q^{\pi_{\theta_t}}(s,a^*(s)) + \gamma \sum_{s^\prime}{ \gP(s^\prime | s, a^*(s))} \cdot \left[ V^{\pi_{\theta_{t+1}}}(s^\prime) - V^{\pi_{\theta_t}}(s^\prime) \right] \\
    &\ge  Q^{\pi_{\theta_t}}(s,a^*(s)) + 0\\
    & \ge Q^*(s,a^*(s)) - \Delta^*(s) / 2,
\end{align}
which means $\theta_{t+1} \in \gR_2(s)$. Next we prove $\theta_{t+1} \in \gR_1(s)$. Note that $\forall a \not= a^*(s)$,
\begin{align}
\label{eq:lower_bound_cT_softmax_general_intermediate_1}
    Q^{\pi_{\theta_t}}(s,a^*(s)) - Q^{\pi_{\theta_t}}(s,a) &= Q^{\pi_{\theta_t}}(s,a^*(s)) - Q^*(s,a^*(s)) + Q^*(s,a^*(s)) - Q^{\pi_{\theta_t}}(s,a) \\
    &\ge - \Delta^*(s) / 2 + Q^*(s,a^*(s)) - Q^*(s,a) + Q^*(s,a) - Q^{\pi_{\theta_t}}(s,a) \\
    &\ge - \Delta^*(s) / 2 + Q^*(s,a^*(s)) - \max_{a \not= a^*(s)}{ Q^*(s,a) } + Q^*(s,a) - Q^{\pi_{\theta_t}}(s,a) \\
    &= - \Delta^*(s) / 2 + \Delta^*(s) + \gamma \sum_{s^\prime}{ \gP(s^\prime | s, a)} \cdot \left[ V^*(s^\prime) - V^{\pi_{\theta_t}}(s^\prime) \right] \\
    &\ge - \Delta^*(s) / 2 + \Delta^*(s) + 0\\
    & = \Delta^*(s) / 2.
\end{align}
Using similar arguments we also have $Q^{\pi_{\theta_{t+1}}}(s,a^*(s)) - Q^{\pi_{\theta_{t+1}}}(s,a) \ge \Delta^*(s) / 2$.
According to \cref{lem:policy_gradient_softmax},
\begin{align}
    \frac{\partial V^{\pi_{\theta_t}}(\mu)}{\partial \theta_t(s, a)} &= \frac{1}{1-\gamma} \cdot d_{\mu}^{\pi_{\theta_t}}(s) \cdot \pi_{\theta_t}(a | s) \cdot A^{\pi_{\theta_t}}(s, a) \\
    &= \frac{1}{1-\gamma} \cdot  d_{\mu}^{\pi_{\theta_t}}(s) \cdot \pi_{\theta_t}(a | s) \cdot \left[ Q^{\pi_{\theta_t}}(s, a) - V^{\pi_{\theta_t}}(s) \right].
\end{align}
Furthermore, since $\frac{\partial V^{\pi_{\theta_t}}(\mu)}{\partial {\theta_t}(s, a^*(s))} \ge \frac{\partial V^{\pi_{\theta_t}}(\mu)}{\partial {\theta_t}(s, a)}$, we have
\begin{align}
    \pi_{\theta_t}(a^*(s) | s) \cdot \left[Q^{\pi_{\theta_t}}(s, a^*(s)) - V^{\pi_{\theta_t}}(s) \right] \ge \pi_{\theta_t}(a | s) \cdot  \left[Q^{\pi_{\theta_t}}(s, a) - V^{\pi_{\theta_t}}(s) \right].
\end{align}
Similarly to the first part in the proof for \cref{lem:lower_bound_cT_softmax_special}. There are two cases.
Case (a): If $\pi_{\theta_t}(a^*(s) | s) \ge \pi_{\theta_t}(a | s)$, then $\theta_t(s, a^*(s)) \ge \theta_t(s, a)$.  
After an update of the parameters,
\begin{align}
    \theta_{t+1}(s, a^*(s)) &= \theta_{t}(s, a^*(s)) + \eta \cdot \frac{\partial V^{\pi_{\theta_t}}(\mu)}{\partial {\theta_t}(s, a^*(s))} \Big/ \left\|\frac{\partial V^{\pi_{\theta_t}}(\mu)}{\partial {\theta_t}}\right\|_2 \\ 
    &\ge \theta_t(s, a) + \eta \cdot \frac{\partial V^{\pi_{\theta_t}}(\mu)}{\partial {\theta_t}(s, a)} \Big/ \left\|\frac{\partial V^{\pi_{\theta_t}}(\mu)}{\partial {\theta_t} }\right\|_2 = \theta_{t+1}(s, a),
\end{align}
which implies $\pi_{\theta_{t+1}}(a^*(s) | s) \ge \pi_{\theta_{t+1}}(a | s)$. Since $Q^{\pi_{\theta_{t+1}}}(s,a^*(s)) - Q^{\pi_{\theta_{t+1}}}(s,a) \ge \Delta^*(s) / 2 \ge 0$, $\forall a$, we have $Q^{\pi_{\theta_{t+1}}}(s,a^*(s)) - V^{\pi_{\theta_{t+1}}}(s) = Q^{\pi_{\theta_{t+1}}}(s,a^*(s)) - \sum_{a}{ \pi_{\theta_{t+1}}(a | s) \cdot Q^{\pi_{\theta_{t+1}}}(s,a) } \ge 0$, and
\begin{align}
    \pi_{\theta_{t+1}}(a^*(s) | s) \cdot \left[Q^{\pi_{\theta_{t+1}}}(s, a^*(s)) - V^{\pi_{\theta_{t+1}}}(s) \right] \ge \pi_{\theta_{t+1}}(a | s) \cdot  \left[Q^{\pi_{\theta_{t+1}}}(s, a) - V^{\pi_{\theta_{t+1}}}(s) \right],
\end{align}
which is equivalent to $\frac{\partial V^{\pi_{\theta_{t+1}}}(\mu)}{\partial {\theta_{t+1}}(s, a^*(s))} \ge \frac{\partial V^{\pi_{\theta_{t+1}}}(\mu)}{\partial {\theta_{t+1}}(s, a)}$, i.e., $\theta_{t+1} \in \gR_1(s)$.\\

Case (b): If $\pi_{\theta_t}(a^*(s) | s) < \pi_{\theta_t}(a | s)$, then by $\frac{\partial V^{\pi_{\theta_t}}(\mu)}{\partial {\theta_t}(s, a^*(s))} \ge \frac{\partial V^{\pi_{\theta_t}}(\mu)}{\partial {\theta_t}(s, a)}$,
\begin{align}
\MoveEqLeft
    \pi_{\theta_t}(a^*(s) | s) \cdot \left[Q^{\pi_{\theta_t}}(s, a^*(s)) - V^{\pi_{\theta_t}}(s) \right] \ge \pi_{\theta_t}(a | s) \cdot \left[Q^{\pi_{\theta_t}}(s, a) - V^{\pi_{\theta_t}}(s) \right] \\
    &= \pi_{\theta_t}(a | s) \cdot \left[Q^{\pi_{\theta_t}}(s, a^*(s)) - V^{\pi_{\theta_t}}(s) + Q^{\pi_{\theta_t}}(s, a) - Q^{\pi_{\theta_t}}(s, a^*(s)) \right],
\end{align}
which, after rearranging, is equivalent to
\begin{align}
    Q^{\pi_{\theta_t}}(s, a^*(s)) - Q^{\pi_{\theta_t}}(s, a) &\ge \left( 1 - \frac{\pi_{\theta_t}(a^*(s) | s)}{\pi_{\theta_t}(a | s)} \right) \cdot \left[ Q^{\pi_{\theta_t}}(s, a^*(s)) - V^{\pi_{\theta_t}}(s)  \right] \\
    &= \left( 1 - \exp\left\{ \theta_{t}(s, a^*(s)) - \theta_{t}(s, a) \right\} \right) \cdot \left[ Q^{\pi_{\theta_t}}(s, a^*(s)) - V^{\pi_{\theta_t}}(s) \right].
\end{align}
Since $\theta_{t+1} \in \gR_3(s)$, we have, 
\begin{align}
    Q^{\pi_{\theta_{t+1}}}(s, a^*(s)) - V^{\pi_{\theta_{t+1}}}(s) \le \Delta^*(s) / 2 \le Q^{\pi_{\theta_{t+1}}}(s,a^*(s)) - Q^{\pi_{\theta_{t+1}}}(s,a).
\end{align}
On the other hand,
\begin{align}
    &\theta_{t+1}(s, a^*(s)) - \theta_{t+1}(s, a) \\
    &= \theta_{t}(s, a^*(s)) + \eta \cdot \frac{\partial V^{\pi_{\theta_t}}(\mu)}{\partial {\theta_t}(s, a^*(s))} \Big/ \left\|\frac{\partial V^{\pi_{\theta_t}}(\mu)}{\partial {\theta_t} }\right\|_2 - \theta_{t}(s, a) - \eta \cdot \frac{\partial V^{\pi_{\theta_t}}(\mu)}{\partial {\theta_t}(s, a)} \Big/ \left\|\frac{\partial V^{\pi_{\theta_t}}(\mu)}{\partial {\theta_t} }\right\|_2 \\
    &\ge \theta_{t}(s, a^*(s)) - \theta_{t}(s, a),
\end{align}
which implies
\begin{align}
    1 - \exp\left\{ \theta_{t+1}(s, a^*(s)) - \theta_{t+1}(s, a) \right\} \le 1 - \exp\left\{ \theta_{t}(s, a^*(s)) - \theta_{t}(s, a) \right\}.
\end{align}
Furthermore, since $1 - \exp\left\{ \theta_{t}(s, a^*(s)) - \theta_{t}(s, a) \right\} = 1 - \frac{\pi_{\theta_{t}}(a^*(s) | s)}{\pi_{\theta_{t}}(a | s)} > 0$ (in this case $\pi_{\theta_t}(a^*(s) |s) < \pi_{\theta_t}(a | s))$,
\begin{align}
    \left( 1 - \exp\left\{ \theta_{t+1}(s, a^*(s)) - \theta_{t+1}(s, a) \right\} \right) \cdot \left[ Q^{\pi_{\theta_{t+1}}}(s, a^*(s)) - V^{\pi_{\theta_{t+1}}}(s) \right] \le Q^{\pi_{\theta_{t+1}}}(s, a^*(s)) - Q^{\pi_{\theta_{t+1}}}(s, a),
\end{align}
which after rearranging is equivalent to
\begin{align}
    \pi_{\theta_{t+1}}(a^*(s) | s) \cdot \left[Q^{\pi_{\theta_{t+1}}}(s, a^*(s)) - V^{\pi_{\theta_{t+1}}}(s) \right] \ge \pi_{\theta_{t+1}}(a | s) \cdot  \left[Q^{\pi_{\theta_{t+1}}}(s, a) - V^{\pi_{\theta_{t+1}}}(s) \right],
\end{align}
which means $\frac{\partial V^{\pi_{\theta_{t+1}}}(\mu)}{\partial {\theta_{t+1}}(s, a^*(s))} \ge \frac{\partial V^{\pi_{\theta_{t+1}}}(\mu)}{\partial {\theta_{t+1}}(s, a)}$ i.e., $\theta_{t+1} \in \gR_1(s)$. Now we have (i) if $\theta_{t} \in \gR_1(s) \cap \gR_2(s) \cap \gR_3(s)$, then $\theta_{t+1} \in \gR_1(s) \cap \gR_2(s) \cap \gR_3(s)$.\\

Let us now turn to proving part~(ii).
We have $\pi_{\theta_{t+1}}(a^*(s) | s) \ge \pi_{\theta_{t}}(a^*(s) | s)$. If $\theta_{t} \in \gR_1(s) \cap \gR_2(s) \cap \gR_3(s)$, then $\frac{\partial V^{\pi_{\theta_t}}(\mu)}{\partial {\theta_t}(s, a^*(s))} \ge \frac{\partial V^{\pi_{\theta_t}}(\mu)}{\partial {\theta_t}(s, a)}, \ \forall a \not= a^*$. After an update of the parameters,
\begin{align}
\MoveEqLeft
    \pi_{\theta_{t+1}}(a^*(s) | s) = \frac{\exp\left\{ \theta_{t+1}(s, a^*(s)) \right\}}{ \sum_{a}{ \exp\left\{ \theta_{t+1}(s, a) \right\}} } \\
    &= \frac{\exp\left\{ \theta_{t}(s, a^*(s)) + \eta \cdot \frac{\partial V^{\pi_{\theta_t}}(\mu)}{\partial {\theta_t}(s, a^*(s))} \Big/ \left\|\frac{\partial V^{\pi_{\theta_t}}(\mu)}{\partial {\theta_t}}\right\|_2 \right\}}{ \sum_{a}{ \exp\left\{ \theta_{t}(s, a) + \eta \cdot \frac{\partial V^{\pi_{\theta_t}}(\mu)}{\partial {\theta_t}(s, a)} \Big/ \left\|\frac{\partial V^{\pi_{\theta_t}}(\mu)}{\partial {\theta_t} }\right\|_2 \right\}} } \\
    &\ge \frac{\exp\left\{ \theta_{t}(s, a^*(s)) + \eta \cdot \frac{\partial V^{\pi_{\theta_t}}(\mu)}{\partial {\theta_t}(s, a^*(s))} \Big/ \left\|\frac{\partial V^{\pi_{\theta_t}}(\mu)}{\partial {\theta_t}}\right\|_2 \right\}}{ \sum_{a}{ \exp\left\{ \theta_{t}(s, a) + \eta \cdot \frac{\partial V^{\pi_{\theta_t}}(\mu)}{\partial {\theta_t}(s, a^*(s))} \Big/ \left\|\frac{\partial V^{\pi_{\theta_t}}(\mu)}{\partial {\theta_t}}\right\|_2 \right\}} } 
    \qquad \left( \text{because } \frac{\partial V^{\pi_{\theta_t}}(\mu)}{\partial {\theta_t}(s, a^*(s))} \ge \frac{\partial V^{\pi_{\theta_t}}(\mu)}{\partial {\theta_t}(s, a)} \right) \\
    &= \frac{\exp\left\{ {\theta_t}(s, a^*(s)) \right\}}{ \sum_{a}{ \exp\left\{ {\theta_t}(s, a) \right\}} } = \pi_{\theta_t}(a^*(s) | s).
\end{align}

\paragraph{Claim II, Claim III, Claim IV.} The proof of those claims are exactly the same as \citet[Lemma 9]{mei2020global}, since they do not involve the update rule.
\end{proof}

\textbf{\cref{thm:final_rates_normalized_softmax_pg_general}.}
Let \cref{ass:posinit} hold and
let $\{\theta_t\}_{t\ge 1}$ be generated using \cref{alg:normalized_policy_gradient_softmax} with 
\begin{align}
    \eta = \frac{ 1 - \gamma }{ 6 \cdot ( 1 - \gamma)  + 8 \cdot \left( C_\infty - (1 - \gamma) \right) } \cdot \frac{ 1}{ \sqrt{S} },
\end{align}
where $C_\infty \coloneqq \max_{\pi}{ \left\| \frac{d_{\mu}^{\pi}}{ \mu} \right\|_\infty} < \infty$. Denote $C_\infty^\prime \coloneqq \max_{\pi}{ \left\| \frac{d_{\rho}^{\pi}}{ \mu} \right\|_\infty}$. Let $c$ be the positive constant from \cref{lem:lower_bound_cT_softmax_general}.
We have, for all $t\ge 1$,
\begin{align}
    V^*(\rho) - V^{\pi_{\theta_t}}(\rho) \le \frac{ \left( V^*(\mu) - V^{\pi_{\theta_1}}(\mu) \right) \cdot C_\infty^\prime }{ 1 - \gamma} \cdot e^{ - C \cdot (t-1)},
\end{align}
where
\begin{align}
    C = \frac{ ( 1 - \gamma)^2 \cdot c }{ 12 \cdot ( 1 - \gamma)  + 16 \cdot \left( C_\infty - (1 - \gamma) \right) } \cdot \frac{ 1}{ S } \cdot \left\| \frac{d_{\mu}^{\pi^*}}{\mu} \right\|_\infty^{-1}.
\end{align}
\begin{proof}
First note that for any $\theta$ and $\mu$,
\begin{align}
\label{eq:stationary_distribution_dominate_initial_state_distribution}
    d_{\mu}^{\pi_\theta}(s) &= \expectation_{s_0 \sim \mu}{ \left[ d_{\mu}^{\pi_\theta}(s) \right] } \\
    &= \expectation_{s_0 \sim \mu}{ \left[ (1 - \gamma) \cdot  \sum_{t=0}^{\infty}{ \gamma^t \probability(s_t = s | s_0, \pi_\theta, \gP) } \right] } \\
    &\ge \expectation_{s_0 \sim \mu}{ \left[ (1 - \gamma) \cdot  \probability(s_0 = s | s_0)  \right] } \\
    &= (1 - \gamma) \cdot \mu(s).
\end{align}
Next, according to \cref{lem:value_suboptimality}, we have,
\begin{align}
\MoveEqLeft
    V^*(\rho) - V^{\pi_\theta}(\rho) = \frac{1}{1 - \gamma} \sum_{s}{ d_{\rho}^{\pi_\theta}(s)  \sum_{a}{ \left( \pi^*(a | s) - \pi_\theta(a | s) \right) \cdot Q^*(s,a) } } \\
    &= \frac{1}{1 - \gamma} \sum_{s} \frac{d_{\rho}^{\pi_\theta}(s)}{d_{\mu}^{\pi_\theta}(s)} \cdot d_{\mu}^{\pi_\theta}(s) \sum_{a}{ \left( \pi^*(a | s) - \pi_\theta(a | s) \right) \cdot Q^*(s,a) } \\
    &\le \frac{1}{1 - \gamma} \cdot \left\| \frac{d_{\rho}^{\pi_\theta}}{d_{\mu}^{\pi_\theta}} \right\|_\infty \sum_{s} d_{\mu}^{\pi_\theta}(s) \sum_{a}{ \left( \pi^*(a | s) - \pi_\theta(a | s) \right) \cdot Q^*(s,a) } \qquad \left( \sum_{a}{ \left( \pi^*(a | s) - \pi_\theta(a | s) \right) \cdot Q^*(s,a) } \ge 0 \right) \\
    &\le \frac{1}{(1 - \gamma)^2} \cdot \left\| \frac{d_{\rho}^{\pi_\theta}}{\mu} \right\|_\infty \sum_{s} d_{\mu}^{\pi_\theta}(s) \sum_{a}{ \left( \pi^*(a | s) - \pi_\theta(a | s) \right) \cdot Q^*(s,a) } 
    \qquad \left( \text{by \cref{eq:stationary_distribution_dominate_initial_state_distribution} and } \min_s\mu(s)>0 \right)
    \\
    &\le \frac{1}{(1 - \gamma)^2} \cdot C_\infty^\prime \cdot \sum_{s} d_{\mu}^{\pi_\theta}(s) \sum_{a}{ \left( \pi^*(a | s) - \pi_\theta(a | s) \right) \cdot Q^*(s,a) } \\
    &= \frac{1}{1 - \gamma} \cdot C_\infty^\prime \cdot \left[ V^*(\mu) - V^{\pi_\theta}(\mu) \right]. \qquad \left ( \text{by \cref{lem:value_suboptimality}} \right)
\end{align}
Denote $\theta_{\zeta_t} \coloneqq \theta_t + \zeta_t \cdot (\theta_{t+1} - \theta_t)$ with some $\zeta_t \in [0,1]$. And note $ \eta = \frac{ 1 - \gamma }{ 6 \cdot ( 1 - \gamma)  + 8 \cdot \left( C_\infty - (1 - \gamma) \right) } \cdot \frac{ 1}{ \sqrt{S} } $.
According to \cref{lem:non_uniform_smoothness_softmax_general}, we have,
\begin{align}
\MoveEqLeft
    \left| V^{\pi_{\theta_{t+1}}}(\mu) - V^{\pi_{\theta_t}}(\mu) - \Big\langle \frac{\partial V^{\pi_{\theta_t}}(\mu)}{\partial \theta_t}, \theta_{t+1} - \theta_t \Big\rangle \right| \\
    &\le \frac{ 3 \cdot ( 1 - \gamma)  + 4 \cdot \left( C_\infty - (1 - \gamma) \right) }{ 2 \cdot ( 1 - \gamma) } \cdot \sqrt{S} \cdot \left\| \frac{\partial V^{\pi_{\theta_{\zeta_t}}}(\mu)}{\partial {\theta_{\zeta_t}} }\right\|_2 \cdot \| \theta_{t+1} - \theta_t \|_2^2 \\
    &\le \frac{ 3 \cdot ( 1 - \gamma)  + 4 \cdot \left( C_\infty - (1 - \gamma) \right) }{ 1 - \gamma} \cdot \sqrt{S} \cdot \left\| \frac{\partial V^{\pi_{\theta_t}}(\mu)}{\partial {\theta_t} }\right\|_2 \cdot \| \theta_{t+1} - \theta_t \|_2^2. \qquad \left( \text{by \cref{lem:non_uniform_smoothness_intermediate_policy_gradient_norm_general}} \right)
\end{align}
Denote $\delta_t = V^*(\mu) - V^{\pi_{\theta_{t}}}(\mu)$. We have,
\begin{align}\label{eq:ascent_lemma_MDP}
\MoveEqLeft
    \delta_{t+1} - \delta_t = V^{\pi_{\theta_{t}}}(\mu) - V^{\pi_{\theta_{t+1}}}(\mu) \\
    &\le - \Big\langle \frac{\partial V^{\pi_{\theta_t}}(\mu)}{\partial \theta_t}, \theta_{t+1} - \theta_t \Big\rangle + \frac{ 3 \cdot ( 1 - \gamma)  + 4 \cdot \left( C_\infty - (1 - \gamma) \right) }{ 1 - \gamma } \cdot \sqrt{S} \cdot \left\| \frac{\partial V^{\pi_{\theta_t}}(\mu)}{\partial {\theta_t} }\right\|_2 \cdot \| \theta_{t+1} - \theta_t \|_2^2 \\
    &= - \frac{ 1 - \gamma }{ 12 \cdot ( 1 - \gamma)  + 16 \cdot \left( C_\infty - (1 - \gamma) \right) } \cdot \frac{ 1}{ \sqrt{S} } \cdot \left\| \frac{\partial V^{\pi_{\theta_t}}(\mu)}{\partial {\theta_t} }\right\|_2 \qquad \left( \text{using the value of } \eta \right) \\
    &\le - \frac{ 1 - \gamma }{ 12 \cdot ( 1 - \gamma)  + 16 \cdot \left( C_\infty - (1 - \gamma) \right) } \cdot \frac{ 1}{ \sqrt{S} } \cdot \frac{ \min_s{ \pi_{\theta_t}(a^*(s)|s) } }{ \sqrt{S} \cdot  \left\| d_{\mu}^{\pi^*} / d_{\mu}^{\pi_{\theta_t}} \right\|_\infty } \cdot \delta_t \qquad \left( \text{by \cref{lem:non_uniform_lojasiewicz_softmax_general}} \right) \\
    &\le - \frac{ ( 1 - \gamma)^2 }{ 12 \cdot ( 1 - \gamma)  + 16 \cdot \left( C_\infty - (1 - \gamma) \right) } \cdot \frac{ 1}{ S } \cdot \left\| \frac{d_{\mu}^{\pi^*}}{\mu} \right\|_\infty^{-1} \cdot \inf_{s\in \cS,t\ge 1} \pi_{\theta_t}(a^*(s)|s) \cdot \delta_t,
\end{align}
where the last inequality is by $d_{\mu}^{\pi_{\theta_t}}(s) \ge (1 - \gamma) \cdot \mu(s)$ (cf. \cref{eq:stationary_distribution_dominate_initial_state_distribution}). According to \cref{lem:lower_bound_cT_softmax_general}, $c=\inf_{s\in \cS,t\ge 1} \pi_{\theta_t}(a^*(s)|s) > 0$. Therefore we have,
\begin{align}
    V^*(\mu) - V^{\pi_{\theta_{t}}}(\mu) \le  \left( V^*(\mu) - V^{\pi_{\theta_1}}(\mu) \right) \cdot \exp\left\{ - \frac{ ( 1 - \gamma)^2 \cdot c \cdot (t-1)}{ 12 \cdot ( 1 - \gamma)  + 16 \cdot \left( C_\infty - (1 - \gamma) \right) } \cdot \frac{ 1}{ S } \cdot \left\| \frac{d_{\mu}^{\pi^*}}{\mu} \right\|_\infty^{-1}  \right\},
\end{align}
which leads to the final result,
\begin{align}
    V^*(\rho) - V^{\pi_{\theta_t}}(\rho) \le \frac{ \left( V^*(\mu) - V^{\pi_{\theta_1}}(\mu) \right) \cdot C_\infty^\prime }{ 1 - \gamma} \cdot \exp\left\{ - \frac{ ( 1 - \gamma)^2 \cdot c \cdot (t-1) }{ 12 \cdot ( 1 - \gamma)  + 16 \cdot \left( C_\infty - (1 - \gamma) \right) } \cdot \frac{ 1}{ S } \cdot \left\| \frac{d_{\mu}^{\pi^*}}{\mu} \right\|_\infty^{-1}  \right\},
\end{align}
thus, finishing the proof.
\end{proof}

\section{Proofs for \cref{sec:generalized_linear_model}}
\label{sec:proofs_generalized_linear_model}

\textbf{\cref{lem:non_uniform_lojasiewicz_glm_sigmoid_realizable}} (N\L{}) \textbf{.}
Denote $u(\theta) \coloneqq \min_{i}{ \left\{ \pi_i \cdot \left( 1 - \pi_i \right) \right\} }$, and $v \coloneqq \min_{i}{ \left\{  \pi_i^*\cdot \left( 1 - \pi_i^* \right) \right\} }$. We have, for all $i \in [N]$,
\begin{align}
    \left\| \frac{\partial \gL(\theta)}{\partial \theta} \right\|_2 \ge 8 \cdot u(\theta) \cdot \min\left\{ u(\theta), v \right\} \cdot \sqrt{\lambda_\phi} \cdot \left[ \frac{1}{N} \cdot \sum_{i=1}^{N}{ \left( \pi_i - \pi_i^* \right)^2 } \right]^{\frac{1}{2}},
\end{align}
where $\lambda_\phi$ is the smallest positive eigenvalue of $\frac{1}{N} \cdot \sum_{i=1}^{N}{ \phi_i \phi_i^\top}$.
\begin{proof}
Denote $\pi_i^\prime \coloneqq \sigmoid(z_i^\prime)$, where $z_i^\prime \coloneqq \phi_i^\top \theta + \zeta \cdot \left( \phi_i^\top \theta - \phi_i^\top \theta^* \right)$ for some $\zeta \in [0, 1]$. We have,
\begin{align}
\MoveEqLeft
\label{eq:non_uniform_lojasiewicz_glm_sigmoid_realizable_intermediate_1}
    \left( \pi_i - \pi_i^* \right)^2 = \left( \pi_i - \pi_i^* \right) \cdot \frac{ d  \sigmoid( z_i^\prime ) }{ d z_i^\prime } \cdot \left( \phi_i^\top \theta - \phi_i^\top \theta^* \right) \qquad \left( \text{by the mean value theorem} \right) \\
    &= \pi_i^\prime \cdot \left( 1 - \pi_i^\prime \right) \cdot \left( \pi_i - \pi_i^* \right) \cdot \left( \phi_i^\top \theta - \phi_i^\top \theta^* \right) \\
    &\le \frac{1}{4} \cdot \left( \pi_i - \pi_i^* \right) \cdot \left( \phi_i^\top \theta - \phi_i^\top \theta^* \right). \qquad \left( x \cdot (1 - x ) \le \frac{1}{4}, \ \forall x \in [0, 1]; \ \left( \pi_i - \pi_i^* \right) \cdot \left( \phi_i^\top \theta - \phi_i^\top \theta^* \right) \ge 0 \right)
\end{align}
Therefore we have,
\begin{align}
\label{eq:non_uniform_lojasiewicz_glm_sigmoid_realizable_intermediate_2}
\MoveEqLeft
    \frac{1}{N} \cdot \sum_{i=1}^{N}{ \left( \pi_i - \pi_i^* \right)^2 } \le \frac{1}{4N} \cdot \sum_{i=1}^{N}{ \left( \pi_i - \pi_i^* \right) \cdot \left( \phi_i^\top \theta - \phi_i^\top \theta^* \right) } \qquad \left( \text{by \cref{eq:non_uniform_lojasiewicz_glm_sigmoid_realizable_intermediate_1}} \right) \\
    &= \frac{1}{4N} \cdot \sum_{i=1}^{N}{ \frac{1}{\pi_i \cdot \left( 1 - \pi_i \right)} \cdot \pi_i \cdot \left( 1 - \pi_i \right) \cdot \left( \pi_i - \pi_i^* \right) \cdot \left( \phi_i^\top \theta - \phi_i^\top \theta^* \right) } \\
    &\le \frac{1}{4N} \cdot \frac{1}{ \min_{i}{ \pi_i \cdot \left( 1 - \pi_i \right) } } \cdot \sum_{i=1}^{N}{ \pi_i \cdot \left( 1 - \pi_i \right) \cdot \left( \pi_i - \pi_i^* \right) \cdot \left( \phi_i^\top \theta - \phi_i^\top \theta^* \right) } \qquad \left( \left( \pi_i - \pi_i^* \right) \cdot \left( \phi_i^\top \theta - \phi_i^\top \theta^* \right) \ge 0 \right) \\
    &= \frac{1}{8} \cdot \frac{1}{ \min_{i}{ \pi_i \cdot \left( 1 - \pi_i \right) } } \cdot \left( \frac{2}{N} \cdot \sum_{i=1}^{N}{ \pi_i \cdot \left( 1 - \pi_i \right) \cdot \left( \pi_i - \pi_i^* \right) \cdot \phi_i } \right)^\top \left( \theta - \theta^* - c \cdot v_{\phi, \bot} \right) \\
    &= \frac{1}{8} \cdot \frac{1}{ \min_{i}{ \pi_i \cdot \left( 1 - \pi_i \right) } } \cdot \left( \frac{\partial \gL(\theta)}{\partial \theta} \right)^\top \left( \theta - \theta^* - c \cdot v_{\phi, \bot} \right) \qquad \left( \frac{\partial \gL(\theta)}{\partial \theta} = \frac{2}{N} \cdot \sum_{i=1}^{N}{ \pi_i \cdot \left( 1 - \pi_i \right) \cdot \left( \pi_i - \pi_i^* \right) \cdot \phi_i } \right) \\
    &\le \frac{1}{8} \cdot \frac{1}{ \min_{i}{ \pi_i \cdot \left( 1 - \pi_i \right) } } \cdot \left\| \frac{\partial \gL(\theta)}{\partial \theta} \right\|_2 \cdot \left\| \theta - \theta^* - c \cdot v_{\phi, \bot} \right\|_2 \qquad \left( \text{by Cauchy-Schwarz} \right) \\
    &= \frac{1}{8} \cdot \frac{1}{ u(\theta) } \cdot \left\| \frac{\partial \gL(\theta)}{\partial \theta} \right\|_2 \cdot \left\| \theta - \theta^* - c \cdot v_{\phi, \bot} \right\|_2, \qquad \left( u(\theta) \coloneqq \min_{i}{ \left\{ \pi_i \cdot \left( 1 - \pi_i \right) \right\} } \right)
\end{align}
where $v_{\phi, \bot}$ is orthogonal to the space $\text{Span}\left\{ \phi_1, \phi_2, \dots, \phi_N \right\}$, and $\theta - \theta^* - c \cdot v_{\phi, \bot}$ refers to the vector after cutting off all the components $v_{\phi, \bot}$ from $\theta - \theta^*$, such that $\theta - \theta^* - c \cdot v_{\phi, \bot} \in \text{Span}\left\{ \phi_1, \phi_2, \dots, \phi_N \right\}$. Next, we have, 
\begin{align}
\MoveEqLeft
\label{eq:non_uniform_lojasiewicz_glm_sigmoid_realizable_intermediate_3}
    \frac{1}{N} \cdot \sum_{i=1}^{N}{ \left( \pi_i - \pi_i^* \right)^2 } = \frac{1}{N} \cdot \sum_{i=1}^{N}{ \left( \frac{ d  \sigmoid( z_i^\prime ) }{ d z_i^\prime } \right)^2 \cdot \left( \phi_i^\top \theta - \phi_i^\top \theta^* \right)^2 } \qquad \left( \text{by the mean value theorem} \right) \\
    &= \frac{1}{N} \cdot \sum_{i=1}^{N}{ \left( \pi_i^\prime \right)^2 \cdot \left( 1 - \pi_i^\prime \right)^2 \cdot \left( \phi_i^\top \theta - \phi_i^\top \theta^* \right)^2 } \qquad \left( \text{by \cref{eq:non_uniform_lojasiewicz_glm_sigmoid_realizable_intermediate_1}} \right) \\
    &\ge \min_{i}\left\{ \left( \pi_i^\prime \right)^2 \cdot \left( 1 - \pi_i^\prime \right)^2 \right\} \cdot \frac{1}{N} \cdot \sum_{i=1}^{N}{ \left( \phi_i^\top \theta - \phi_i^\top \theta^* \right)^2 } \\
    &= \min_{i}\left\{ \left( \pi_i^\prime \right)^2 \cdot \left( 1 - \pi_i^\prime \right)^2 \right\} \cdot \left( \theta - \theta^* \right)^\top \left( \frac{1}{N} \cdot \sum_{i=1}^{N}{ \phi_i \phi_i^\top} \right) \left( \theta - \theta^* \right) \\
    &= \min_{i}\left\{ \left( \pi_i^\prime \right)^2 \cdot \left( 1 - \pi_i^\prime \right)^2 \right\} \cdot \left( \theta - \theta^* - c \cdot v_{\phi, \bot} \right)^\top \left( \frac{1}{N} \cdot \sum_{i=1}^{N}{ \phi_i \phi_i^\top} \right) \left( \theta - \theta^* - c \cdot v_{\phi, \bot} \right) \\
    &\ge \min\left\{ u(\theta)^2, v^2 \right\} \cdot \left( \theta - \theta^* - c \cdot v_{\phi, \bot} \right)^\top \left( \frac{1}{N} \cdot \sum_{i=1}^{N}{ \phi_i \phi_i^\top} \right) \left( \theta - \theta^* - c \cdot v_{\phi, \bot} \right) \qquad \left( v \coloneqq \min_{i}{ \left\{ \pi_i^*\cdot \left( 1 - \pi_i^* \right) \right\} } \right) \\
    &\ge \min\left\{ u(\theta)^2, v^2 \right\} \cdot \lambda_{\phi} \cdot \left\| \theta - \theta^* - c \cdot v_{\phi, \bot} \right\|_2^2,
\end{align}
where $\lambda_{\phi}$ is the smallest positive eigenvalue of $\frac{1}{N} \cdot \sum_{i=1}^{N}{ \phi_i \phi_i^\top}$. Therefore, we have,
\begin{align}
\MoveEqLeft
    \frac{1}{N} \cdot \sum_{i=1}^{N}{ \left( \pi_i - \pi_i^* \right)^2 } \le \frac{1}{8} \cdot \frac{1}{ u(\theta) } \cdot \left\| \frac{\partial \gL(\theta)}{\partial \theta} \right\|_2 \cdot \left\| \theta - \theta^* - c \cdot v_{\phi, \bot}  \right\|_2 \qquad \left( \text{by \cref{eq:non_uniform_lojasiewicz_glm_sigmoid_realizable_intermediate_2}} \right) \\
    &\le \frac{1}{8} \cdot \frac{1}{ u(\theta) } \cdot \left\| \frac{\partial \gL(\theta)}{\partial \theta} \right\|_2 \cdot \frac{1}{ \min\left\{ u(\theta), v \right\} } \cdot \frac{1}{ \sqrt{\lambda_\phi} } \cdot \left[ \frac{1}{N} \cdot \sum_{i=1}^{N}{ \left( \pi_i - \pi_i^* \right)^2 } \right]^{\frac{1}{2}}, \qquad \left( \text{by \cref{eq:non_uniform_lojasiewicz_glm_sigmoid_realizable_intermediate_3}} \right)
\end{align}
which implies,
\begin{equation*}
    \left\| \frac{\partial \gL(\theta)}{\partial \theta} \right\|_2 \ge 8 \cdot u(\theta) \cdot \min\left\{ u(\theta), v \right\} \cdot \sqrt{\lambda_\phi} \cdot \left[ \frac{1}{N} \cdot \sum_{i=1}^{N}{ \left( \pi_i - \pi_i^* \right)^2 } \right]^{\frac{1}{2}}. \qedhere
\end{equation*}
\end{proof}

\textbf{\cref{lem:non_uniform_smoothness_glm_sigmoid_realizable}.}
Denote $u(\theta) \coloneqq \min_{i}{ \left\{ \pi_i \cdot \left( 1 - \pi_i \right) \right\} }$, $v \coloneqq \min_{i}{ \left\{  \pi_i^*\cdot \left( 1 - \pi_i^* \right) \right\} }$, and $\lambda_\phi$ is the smallest positive eigenvalue of $\frac{1}{N} \cdot \sum_{i=1}^{N}{ \phi_i \phi_i^\top}$. We have, $\gL(\theta)$ satisfies $\beta$ smoothness with 
\begin{align}
    \beta = \frac{3}{8} \cdot \max_{i \in [N]}{ \left\| \phi_i \right\|_2^2 },
\end{align}
and $\beta(\theta)$ NS with
\begin{align}
    \beta(\theta) = L_1 \cdot \left\| \frac{\partial \gL(\theta)}{\partial \theta} \right\|_2 + L_0 \cdot \left( \left\| \frac{\partial \gL(\theta)}{\partial \theta} \right\|_2^2 \Big/ \gL(\theta) \right),
\end{align}
where
\begin{align}
    L_1 = \frac{ \max_{i}{ \left\| \phi_i \right\|_2^2 } }{32\cdot(\min\{u(\theta), v\}\cdot\sqrt{\lambda_\phi})^{3/2}}, \text{ and }
    L_0 = \frac{17\cdot \max_{i}{ \left\| \phi_i \right\|_2^2 } }{512\cdot u(\theta)^2 \cdot \min\left\{ u(\theta)^2, v^2 \right\} \cdot \lambda_\phi}.
\end{align}
\begin{proof}
Note that the gradient of $\gL(\theta)$ is,
\begin{align}
    \frac{\partial \gL(\theta)}{\partial \theta} = \frac{2}{N} \cdot \sum_{i=1}^{N}{ \pi_i \cdot \left( 1 - \pi_i \right) \cdot \left( \pi_i - \pi_i^* \right) \cdot \phi_i } \in \sR^d.
\end{align} 
Denote the second order derivative (Hessian) of $\gL(\theta)$ as,
\begin{align}
    S(\theta) &\coloneqq \frac{\partial}{\partial \theta} \left\{ \frac{\partial \gL(\theta)}{\partial \theta} \right\} \in \sR^{d \times d}.
\end{align}
For all $j , k \in [d]$, we calculate the corresponding component value of $S(\theta)$ matrix as follows,
\begin{align}
\label{eq:non_uniform_smoothness_glm_sigmoid_realizable_intermediate_1}
    S_{(j, k)} &= \frac{d}{d \theta(k)} \left\{ \frac{2}{N} \cdot \sum_{i=1}^{N}{ \pi_i \cdot \left( 1 - \pi_i \right) \cdot \left( \pi_i - \pi_i^* \right) \cdot \phi_i(j) } \right\} \\
    &= \frac{2}{N} \cdot \sum_{i=1}^{N}{ \frac{d \left\{ \pi_i \cdot \left( 1 - \pi_i \right) \cdot \left( \pi_i - \pi_i^* \right) \right\} }{d \theta(k)} \cdot \phi_i(j) } \\
    &= \frac{2}{N} \cdot \sum_{i=1}^{N}{ \frac{d \left\{ \pi_i \cdot \left( 1 - \pi_i \right) \cdot \left( \pi_i - \pi_i^* \right) \right\} }{d \left\{ \phi_i^\top \theta \right\} } \cdot \frac{d \left\{ \phi_i^\top \theta \right\}}{ d \theta(k)} \cdot \phi_i(j) } \\
    &= \frac{2}{N} \cdot \sum_{i=1}^{N}{ \left[ \pi_i \cdot \left( 1 - \pi_i \right)^2 \cdot \left( \pi_i - \pi_i^* \right) - \pi_i^2 \cdot \left( 1 - \pi_i \right) \cdot \left( \pi_i - \pi_i^* \right) + \pi_i^2 \cdot \left( 1 - \pi_i \right)^2 \right] \cdot \phi_i(k) \cdot \phi_i(j) } \\
    &= \frac{2}{N} \cdot \sum_{i=1}^{N}{ \left[ \pi_i \cdot \left( 1 - \pi_i \right) \cdot \left( 1 - 2 \pi_i \right) \cdot \left( \pi_i - \pi_i^* \right) + \pi_i^2 \cdot \left( 1 - \pi_i \right)^2 \right] \cdot \phi_i(k) \cdot \phi_i(j) }.
\end{align}
To calculate the smoothness coefficient, take a vector $z \in \sR^d$. We have,
\begin{align}
\label{eq:non_uniform_smoothness_glm_sigmoid_realizable_intermediate_2}
\MoveEqLeft
    \left| z^\top S(\theta) z \right| = \left| \sum_{j = 1}^{d} \sum_{k = 1}^{d}{ S_{(j, k)} \cdot z(j) \cdot z(k) } \right| \\
    &= \left| \frac{2}{N} \cdot \sum_{i=1}^{N}{ \left[ \pi_i \cdot \left( 1 - \pi_i \right) \cdot \left( 1 - 2 \pi_i \right) \cdot \left( \pi_i - \pi_i^* \right) + \pi_i^2 \cdot \left( 1 - \pi_i \right)^2 \right] \cdot \left( \phi_i^\top z \right)^2 } \right| \qquad \left( \text{by \cref{eq:non_uniform_smoothness_glm_sigmoid_realizable_intermediate_1}} \right) \\
    &\le \frac{2}{N} \cdot \max_{i}{ \left( \phi_i^\top z \right)^2 } \cdot \sum_{i=1}^{N}{ \left| \pi_i \cdot \left( 1 - \pi_i \right) \cdot \left( 1 - 2 \pi_i \right) \cdot \left( \pi_i - \pi_i^* \right) + \pi_i^2 \cdot \left( 1 - \pi_i \right)^2 \right| } \qquad \left( \text{by H{\" o}lder's inequality} \right) \\
    & \le \frac{2}{N} \cdot \max_{i}{ \left( \phi_i^\top z \right)^2 } \cdot \sum_{i=1}^{N}{ \left[ \pi_i \cdot \left( 1 - \pi_i \right) \cdot \left| 1 - 2 \pi_i \right| \cdot \left| \pi_i - \pi_i^* \right| + \pi_i^2 \cdot \left( 1 - \pi_i \right)^2 \right] } \qquad \left( \text{by triangle inequality} \right) \\
    &\le \frac{2}{N} \cdot \max_{i}{ \left( \phi_i^\top z \right)^2 } \cdot \sum_{i=1}^{N}{ \left[ \frac{1}{8} + \frac{1}{16} \right] } \qquad \left( x \cdot (1 - x) \le 1/4, \text{ and } x \cdot (1 - x) \cdot | 1 - 2x | \le 1/8, \ \forall x \in [0, 1] \right) \\
    &= \frac{3}{8} \cdot \max_{i}{ \left[ \phi_i^\top \left( \frac{z}{\left\| z \right\|_2 } \right) \right]^2 } \cdot \left\| z \right\|_2^2 \\
    &\le \frac{3}{8} \cdot \max_{i}{ \left\| \phi_i \right\|_2^2 } \cdot \left\| z \right\|_2^2.
\end{align}
Therefore, $\gL(\theta)$ satisfies $\beta$ (uniform) smoothness with $\beta = \frac{3}{8} \cdot \max_{i}{ \left\| \phi_i \right\|_2^2 }$. Next, we calculate the NS. We have,
\begin{align}
\MoveEqLeft
    \sum_{i=1}^{N}\pi_i^2\cdot(1-\pi_i)^2 \cdot \gL(\theta) = \sum_{i=1}^{N}\pi_i^2\cdot(1-\pi_i)^2 \cdot \frac{1}{N} \cdot \sum_{j=1}^{N}(\pi_j - \pi_j^*)^2 \\
    &\le \frac{N}{16}\cdot \frac{1}{N} \cdot \sum_{j=1}^{N}(\pi_j - \pi_j^*)^2 \\
    &\le \frac{N}{16}\cdot \frac{1}{64 \cdot u(\theta)^2 \cdot \min\left\{ u(\theta)^2, v^2 \right\} \cdot \lambda_\phi} \cdot\left\|\frac{\partial\gL(\theta)}{\partial \theta}\right\|_2^2, \qquad \left( \text{by \cref{lem:non_uniform_lojasiewicz_glm_sigmoid_realizable}} \right)
\end{align}
which implies, 
\begin{align}
\label{eq:non_uniform_smoothness_glm_sigmoid_realizable_intermediate_3}
    \sum_{i=1}^{N}\pi_i^2\cdot(1-\pi_i)^2 &\le \frac{N}{2} \cdot \frac{1}{512 \cdot u(\theta)^2 \cdot \min\left\{ u(\theta)^2, v^2 \right\} \cdot \lambda_\phi} \cdot \left\|\frac{\partial\gL(\theta)}{\partial \theta}\right\|_2^2\Big/\gL(\theta).
\end{align}
According to \cref{eq:non_uniform_lojasiewicz_glm_sigmoid_realizable_intermediate_3}, we have
\begin{align}
    \sum_{i=1}^{N}\frac{(\pi_i - \pi_i^* )^2}{\sqrt{\gL(\theta)} \cdot \left\|\theta-\theta^* - c \cdot v_{\phi, \bot} \right\|_2^{3/2}} &\ge  \sum_{i=1}^{N}\frac{(\pi_i - \pi_i^*  )^2}{\sqrt{\gL(\theta)}}\cdot(\min\{u(\theta)^2, v^2\}\cdot\lambda_\phi)^{3/4}\cdot\frac{1}{\gL(\theta)^{3/4}}\\
    &= (\min\{u(\theta)^2, v^2\}\cdot\lambda_\phi)^{3/4}\cdot\sum_{i=1}^{N}\frac{(\pi_i - \pi_i^* )^2}{\gL(\theta)^{5/4}}\\
    &= N\cdot(\min\{u(\theta)^2, v^2\}\cdot\lambda_\phi)^{3/4}\cdot\frac{\gL(\theta)}{\gL(\theta)^{5/4}}\\
    &\ge N\cdot(\min\{u(\theta), v\}\cdot \sqrt{\lambda_\phi})^{3/2}. \qquad \left( \gL(\theta) \in (0, 1] \right) 
\end{align}
Therefore we have, 
\begin{align}
\label{eq:non_uniform_smoothness_glm_sigmoid_realizable_intermediate_5}
\MoveEqLeft
    \sum_{i=1}^{N}\pi_i\cdot(1-\pi_i)\cdot \left| 1 - 2 \pi_i \right| \cdot \left|\pi_i - \pi_i^*\right| \le \sum_{i=1}^{N}\pi_i\cdot(1-\pi_i) \cdot \left|\pi_i - \pi_i^*\right|\\ &\le\left(\sum_{i=1}^{N}\pi_i\cdot(1-\pi_i)\cdot\left|\pi_i - \pi_i^*\right|\right)\cdot\left(\sum_{i=1}^{N}\frac{(\pi_i - \pi_i^*)^2}{\sqrt{\gL(\theta)} \cdot \left\|\theta-\theta^* - c \cdot v_{\phi, \bot} \right\|_2^{3/2}}\right)\cdot\frac{1}{N\cdot (\min\{u(\theta), v\}\cdot\sqrt{\lambda_\phi})^{3/2}}\\
    &= \frac{1}{N\cdot(\min\{u(\theta), v\}\cdot\sqrt{\lambda_\phi})^{3/2}}\cdot\left(\sum_{i=1}^{N}{\frac{\pi_i\cdot(1-\pi_i)\cdot\left|\pi_i - \pi_i^*\right|}{\sqrt{\left\|\theta-\theta^* - c \cdot v_{\phi, \bot} \right\|_2}}}\right)\cdot\left(\sum_{i=1}^{N}\frac{(\pi_i - \pi_i^*)^2}{\sqrt{\gL(\theta)} \cdot \left\|\theta-\theta^* - c \cdot v_{\phi, \bot} \right\|_2}\right)\\
    &\le \frac{1}{(\min\{u(\theta), v\}\cdot\sqrt{\lambda_\phi})^{3/2}}\cdot\left(\sum_{i=1}^{N}{\frac{\pi_i^2\cdot(1-\pi_i)^2\cdot(\pi_i-\pi_i^*)^2}{ 2 \cdot \left\|\theta-\theta^* - c \cdot v_{\phi, \bot} \right\|_2} + \frac{(\pi_i-\pi_i^*)^4}{ 2 \cdot \gL(\theta)\cdot\left\|\theta-\theta^* - c \cdot v_{\phi, \bot} \right\|_2^2} }\right)\\
    &\le \frac{1}{(\min\{u(\theta), v\}\cdot\sqrt{\lambda_\phi})^{3/2}}\cdot\left(\frac{1}{32}\cdot\sum_{i=1}^{N}\frac{\pi_i\cdot(1-\pi_i)\cdot(\pi_i-\pi_i^*) \cdot (\phi_i^{\top}\theta-\phi_i^{\top}\theta^*)}{\left\|\theta-\theta^* - c \cdot v_{\phi, \bot} \right\|_2} \right)\\
    &\qquad + \frac{1}{(\min\{u(\theta), v\}\cdot\sqrt{\lambda_\phi})^{3/2}}\cdot\left(\frac{1}{32\cdot u(\theta)^2}\cdot\sum_{i=1}^{N}\frac{\pi_i^2\cdot(1-\pi_i)^2\cdot(\pi_i-\pi_i^*)^2 \cdot (\phi_i^{\top}\theta-\phi_i^{\top}\theta^*)^2}{\gL(\theta)\cdot\left\|\theta-\theta^* - c \cdot v_{\phi, \bot} \right\|^2}\right)\\
    &\le \frac{N}{64 \cdot(\min\{u(\theta), v\}\cdot\sqrt{\lambda_\phi})^{3/2}}\cdot\left\|\frac{\partial\gL(\theta)}{\partial\theta}\right\|_2 + \frac{N}{64\cdot u(\theta)^2\cdot(\min\{u(\theta), v\}\cdot\sqrt{\lambda_\phi})^{3/2}}\cdot\left\|\frac{\partial\gL(\theta)}{\partial\theta}\right\|_2^2\Big/\gL(\theta),
\end{align}
where the second inequality is according to,
\begin{align}
    \left( \sum_{i=1}^{N}{a_i} \right) \cdot \left( \sum_{i=1}^{N}{b_i} \right) &= \sum_{i=1}^{N} \sum_{j=1}^{N}{ a_i \cdot b_j } \le \frac{1}{2} \cdot \sum_{i=1}^{N} \sum_{j=1}^{N}{ \left( a_i^2 + b_j^2 \right) } = \frac{N}{2} \cdot \sum_{i=1}^{N}{ \left( a_i^2 + b_i^2 \right) },
\end{align} 
and the last inequality is from the intermediate results in \cref{eq:non_uniform_lojasiewicz_glm_sigmoid_realizable_intermediate_2}.
Combining \cref{eq:non_uniform_smoothness_glm_sigmoid_realizable_intermediate_2,eq:non_uniform_smoothness_glm_sigmoid_realizable_intermediate_3,eq:non_uniform_smoothness_glm_sigmoid_realizable_intermediate_5}, we have \begin{align}
\MoveEqLeft
    \left| z^\top S(\theta) z \right| 
     \le \frac{2}{N} \cdot \max_{i}{ \left( \phi_i^\top z \right)^2 }\cdot \left[\sum_{i=1}^{N}\pi_i\cdot (1-\pi_i )\cdot \left|\pi_i - \pi_i^* \right| + \sum_{i=1}^{N}\pi_i^2\cdot (1-\pi_i )^2\right]\\
     &\le  \max_{i}{ \left( \phi_i^\top z \right)^2 } \cdot\left(\frac{1}{32\cdot (\min\{u(\theta), v\}\cdot\sqrt{\lambda_\phi})^{3/2} }\cdot\left\|\frac{\partial\gL(\theta)}{\partial\theta}\right\|_2 +\frac{17}{512 \cdot u(\theta)^2 \cdot \min\left\{ u(\theta)^2, v^2 \right\} \cdot \lambda_\phi}\cdot\left\|\frac{\partial\gL(\theta)}{\partial\theta}\right\|_2^2\Big/\gL(\theta) \right) \\
     &\le \max_{i}{ \left\| \phi_i \right\|_2^2 } \cdot \left\| z \right\|_2^2 \cdot\left(\frac{1}{32\cdot(\min\{u(\theta), v\}\cdot\sqrt{\lambda_\phi})^{3/2} }\cdot\left\|\frac{\partial\gL(\theta)}{\partial\theta}\right\|_2 +\frac{17}{512 \cdot u(\theta)^2 \cdot \min\left\{ u(\theta)^2, v^2 \right\} \cdot \lambda_\phi}\cdot\left\|\frac{\partial\gL(\theta)}{\partial\theta}\right\|_2^2\Big/\gL(\theta) \right). 
\end{align}
Therefore, $\gL(\theta)$ satisfies $\beta(\theta)$ NS with
\begin{align}
    \beta(\theta) = L_1 \cdot \left\| \frac{\partial \gL(\theta)}{\partial \theta} \right\|_2 + L_0 \cdot \left( \left\| \frac{\partial \gL(\theta)}{\partial \theta} \right\|_2^2 \Big/ \gL(\theta) \right),
\end{align}
where
\begin{equation*}
    L_1 = \frac{ \max_{i}{ \left\| \phi_i \right\|_2^2 } }{32\cdot(\min\{u(\theta), v\}\cdot\sqrt{\lambda_\phi})^{3/2}}, \text{ and }
    L_0 = \frac{17\cdot \max_{i}{ \left\| \phi_i \right\|_2^2 } }{512\cdot u(\theta)^2 \cdot \min\left\{ u(\theta)^2, v^2 \right\} \cdot \lambda_\phi}. \qedhere
\end{equation*}
\end{proof}

\textbf{\cref{thm:final_rates_normalized_glm_sigmoid_realizable}.}
With $\eta = 1/ \beta $, GD update satisfies for all $t \ge 1$, $\gL(\theta_t) \le \gL(\theta_1) \cdot e^{- C^2 \cdot (t-1)}$.
With $\eta \in \Theta(1)$, GNGD update satisfies for all $t \ge 1$, $\gL(\theta_t) \le \gL(\theta_1) \cdot e^{- C \cdot (t-1)}$, where $C \in (0,1)$, i.e., GNGD is strictly faster than GD.
\begin{proof}
Combining \cref{lem:non_uniform_lojasiewicz_glm_sigmoid_realizable,lem:non_uniform_smoothness_glm_sigmoid_realizable}, and the second part of (2b) in \cref{thm:general_optimization_main_result_1}, we have the results for GD. Using the fourth part of (2b) in \cref{thm:general_optimization_main_result_1}, we have the results for GNGD.
\end{proof}

\section{Miscellaneous Extra Supporting Results}
\label{sec:supporting_lemmas}

\begin{lemma}[Descent lemma for smooth function]
\label{lem:descent_lemma_smooth_function}
Let $f:\R^d \to \R$ be a $\beta$-smooth function, $\theta\in \R^d$ 
and $\theta^\prime = \theta - \eta \cdot \frac{\partial f(\theta)}{\partial \theta}$. We have, for any $0 < \eta < 2 / \beta$,
\begin{align}
    f(\theta^\prime) \le f(\theta).
\end{align}
In particular, for $\eta = \frac{1}{ \beta }$, we have,
\begin{align}
\label{eq:descent_lemma_smooth_function_claim_1}
    f(\theta^\prime) \le f(\theta) - \frac{1}{ 2 \beta } \cdot \left\| \frac{\partial f(\theta)}{\partial \theta} \right\|_2^2.
\end{align}
\end{lemma}
\begin{proof}
According to \cref{def:smoothness}, we have,
\begin{align}
    \left| f(\theta^\prime) - f(\theta) - \Big\langle \frac{\partial f(\theta)}{\partial \theta}, \theta^\prime - \theta \Big\rangle \right| \le \frac{\beta}{2} \cdot \| \theta^\prime - \theta \|_2^2,
\end{align}
which implies,
\begin{align}
    f(\theta^\prime) - f(\theta) &\le \Big\langle \frac{\partial f(\theta)}{\partial \theta}, \theta^\prime - \theta \Big\rangle + \frac{\beta}{2} \cdot \| \theta^\prime - \theta \|_2^2 \\
\label{eq:descent_lemma_smooth_function_intermediate_1}
    &= \eta \cdot \left( - 1 + \frac{\beta}{2} \cdot \eta \right) \cdot \left\| \frac{\partial f(\theta)}{\partial \theta} \right\|_2^2 \qquad \left( \theta^\prime = \theta - \eta \cdot \frac{\partial f(\theta)}{\partial \theta} \right) \\
    &\le 0. \qquad \left( 0 < \eta < \frac{2}{ \beta } \right)
\end{align}
Let $\eta = \frac{1}{ \beta }$ in \cref{eq:descent_lemma_smooth_function_intermediate_1}, we have \cref{eq:descent_lemma_smooth_function_claim_1}.
\end{proof}

\begin{lemma}[Descent lemma for NS function]
\label{lem:descent_lemma_NS_function}
Let $f:\R^d \to \R$ be a function that satisfies NS with $\beta(\theta) > 0$, for all $\theta\in \R^d$ 
and $\theta^\prime = \theta - \frac{1}{ \beta(\theta) } \cdot \frac{\partial f(\theta)}{\partial \theta}$. We have,
\begin{align}
    f(\theta^\prime) \le f(\theta) - \frac{1}{ 2 \cdot \beta(\theta) } \cdot \left\| \frac{\partial f(\theta)}{\partial \theta} \right\|_2^2.
\end{align}
\end{lemma}
\begin{proof}
According to \cref{def:non_uniform_smoothness}, we have,
\begin{align}
    f(\theta^\prime) - f(\theta) &\le \Big\langle \frac{\partial f(\theta)}{\partial \theta}, \theta^\prime - \theta \Big\rangle + \frac{\beta(\theta)}{2} \cdot \| \theta^\prime - \theta \|_2^2 \\
    &= - \frac{ 1 }{ \beta(\theta) } \cdot \left\| \frac{\partial f(\theta)}{\partial \theta} \right\|_2^2 + \frac{ 1 }{ 2 \cdot \beta(\theta) } \cdot \left\| \frac{\partial f(\theta)}{\partial \theta} \right\|_2^2 \qquad \left( \theta^\prime = \theta - \frac{1}{ \beta(\theta) } \cdot \frac{\partial f(\theta)}{\partial \theta} \right) \\
    &= - \frac{ 1 }{ 2 \cdot \beta(\theta) } \cdot \left\| \frac{\partial f(\theta)}{\partial \theta} \right\|_2^2. \qedhere
\end{align}
\end{proof}

\begin{lemma}
\label{lem:auxiliary_lemma_1}
Given any $\alpha > 0$, we have, for all $x \in [0, 1]$,
\begin{align}
    \frac{1}{\alpha} \cdot (1 - x^\alpha) \ge x^\alpha \cdot \left( 1 - x \right).
\end{align}
\end{lemma}
\begin{proof}
Define $f: x \mapsto \frac{1}{\alpha} \cdot (1 - x^\alpha) - x^\alpha \cdot \left( 1 - x \right)$. We show that $f(x) \ge 0$ for all $x \in [0, 1]$. Note that,
\begin{align}
    f(0) = \frac{1}{ \alpha } > 0, \text{ and } f(1) = 0.
\end{align}
On the other hand,
\begin{align}
    f^\prime(x) &= - x^{\alpha - 1} - \alpha \cdot x^{ \alpha - 1} \cdot \left( 1 - x \right) + x^\alpha \\
    &= - x^{\alpha - 1} \cdot \left[ 1 +  \alpha \cdot \left( 1 - x \right) - x \right] \\
    &= - x^{\alpha - 1} \cdot \left( 1 + \alpha \right) \cdot \left( 1 - x \right) \\
    &\le 0, \qquad \left( \alpha > 0, \text{ and } x \in [0, 1] \right)
\end{align}
which means $f$ is monotonically decreasing over $[0, 1]$. Therefore $f(x) \ge 0$ for all $x \in [0, 1]$, finishing the proof.
\end{proof}

\begin{lemma}
\label{lem:auxiliary_lemma_2}
Given any $\alpha > 0$, we have, for all $x \in \left[ \frac{2 \alpha + 1}{2 \alpha + 2}, 1 \right]$,
\begin{align}
    \frac{1}{2 \alpha} \cdot (1 - x^\alpha) \le x^\alpha \cdot \left( 1 - x \right).
\end{align}
\end{lemma}
\begin{proof}
Define $g: x \mapsto x^\alpha \cdot \left( 1 - x \right) - \frac{1}{2 \alpha} \cdot (1 - x^\alpha)$. The derivative of $g$ is,
\begin{align}
    g^\prime(x) &= \alpha \cdot x^{ \alpha - 1} \cdot \left( 1 - x \right) - x^\alpha + (1/2) \cdot x^{\alpha - 1} \\
    &= x^{\alpha - 1} \cdot \left[ \alpha \cdot \left( 1 - x \right) - x + 1/2 \right] \\
    &= x^{\alpha - 1} \cdot \left[ \left( 1 + \alpha \right) \cdot \left( 1 - x \right) - 1/2 \right].
\end{align}
Then we have,
\begin{align}
	g^\prime(x) &> 0 \text{ for all } x \in \left[ 0, (2 \alpha + 1) / (2 \alpha + 2) \right), \text{ and} \\
	g^\prime(x) &\le 0 \text{ for all } x \in \left[ (2 \alpha + 1) / (2 \alpha + 2), 1 \right],
\end{align}
which means $g$ is monotonically increasing over $\left[ 0, (2 \alpha + 1) / (2 \alpha + 2) \right)$ and decreasing over $\left[ (2 \alpha + 1) / (2 \alpha + 2), 1 \right]$. On the other hand,
\begin{align}
	g( (2 \alpha + 1) / (2 \alpha + 2) ) &= \left( \frac{2 \alpha + 1}{2 \alpha + 2} \right)^{\alpha} \cdot \left( 1 - \frac{2 \alpha + 1}{2 \alpha + 2} \right) - \frac{1}{ 2 \alpha} \cdot \left[ 1 - \left( \frac{2 \alpha + 1}{2 \alpha + 2} \right)^{\alpha}  \right] \\
	&= \frac{1}{ 2 \alpha} \cdot \left[ \left( \frac{2 \alpha + 1}{2 \alpha + 2} \right)^{\alpha} \cdot \frac{2 \alpha + 1}{ \alpha + 1 } - 1 \right] \\
	&= \frac{1}{ 2 \alpha} \cdot \left[ \exp\left\{ \log{ \left( \frac{2 \alpha + 1}{ \alpha + 1 } \right) } - \alpha \cdot \log{ \left( 1 + \frac{1}{ 2 \alpha + 1} \right) } \right\} - 1 \right] \\
	&\ge  \frac{1}{ 2 \alpha} \cdot \left[ \exp\left\{ \log{ \left( \frac{2 \alpha + 1}{ \alpha + 1 } \right) } - \frac{ \alpha }{ 2 \alpha + 1} \right\} - 1 \right] \qquad \left( 1 + x \le e^x \right) \\
	&\ge  \frac{1}{ 2 \alpha} \cdot \left[ \exp\left\{ \frac{ \alpha }{ 2 \alpha + 1} - \frac{ \alpha }{ 2 \alpha + 1} \right\} - 1 \right] \qquad \left( \log(x) \ge 1 - 1/x \text{ for } x > 0 \right) \\
	&= 0.
\end{align}
Also note that $g(1) = 0$. Therefore we have $g(x) \ge 0$ for all $x \in \left[ (2 \alpha + 1) / (2 \alpha + 2), 1 \right]$, finishing the proof.
\end{proof}

\begin{lemma}
\label{lem:policy_gradient_softmax}
Denote $H(\pi) \coloneqq \diagonalmatrix(\pi) - \pi \pi^\top$. Softmax policy gradient w.r.t. $\theta$ is
\begin{align}
    \frac{\partial V^{\pi_\theta}(\mu)}{\partial \theta(s, \cdot)} = \frac{1}{1-\gamma} \cdot d_{\mu}^{\pi_\theta}(s) \cdot H(\pi_\theta(\cdot | s)) Q^{\pi_\theta}(s,\cdot), \quad \forall s \in \gS.
\end{align}
\end{lemma}
\begin{proof}
See the proof in \citep[Lemma 1]{mei2020global}. We include a proof for completeness.

According to the policy gradient theorem \citep{sutton2000policy},
\begin{align}
    \frac{\partial V^{\pi_\theta}(\mu)}{\partial \theta} = \frac{1}{1-\gamma} \expectation_{s^\prime \sim d_{\mu}^{\pi_\theta} } { \left[ \sum_{a}  \frac{\partial \pi_\theta(a | s^\prime)}{\partial \theta} \cdot Q^{\pi_\theta}(s^\prime,a) \right] }.
\end{align}
For $s^\prime \not= s$, $\frac{\partial \pi_\theta(a | s^\prime)}{\partial \theta(s, \cdot)} = \rvzero$ since $\pi_\theta(a | s^\prime)$ does not depend on $\theta(s, \cdot)$. Therefore, we have,
\begin{align}
\MoveEqLeft
    \frac{\partial V^{\pi_\theta}(\mu)}{\partial \theta(s, \cdot)} = \frac{1}{1-\gamma} \sum_{s^\prime}{ d_{\mu}^{\pi_\theta}(s^\prime) \cdot \left[ \sum_{a}  \frac{\partial \pi_\theta(a | s^\prime)}{\partial \theta(s, \cdot) } \cdot Q^{\pi_\theta}(s^\prime,a) \right] } \\
    &= \frac{1}{1-\gamma} \cdot d_{\mu}^{\pi_\theta}(s) \cdot { \left[ \sum_{a} \frac{\partial \pi_\theta(a | s)}{\partial \theta(s, \cdot)} \cdot Q^{\pi_\theta}(s,a) \right] } \qquad \left( \frac{\partial \pi_\theta(a | s^\prime)}{\partial \theta(s, \cdot)} = \rvzero, \ \forall s^\prime \not= s \right) \\
    &= \frac{1}{1-\gamma} \cdot d_{\mu}^{\pi_\theta}(s) \cdot \left( \frac{d \pi(\cdot | s)}{d \theta(s, \cdot)} \right)^\top Q^{\pi_\theta}(s,\cdot) \\
    &= \frac{1}{1-\gamma} \cdot d_{\mu}^{\pi_\theta}(s) \cdot H(\pi_\theta(\cdot | s)) Q^{\pi_\theta}(s,\cdot). \qquad \left( H(\pi_\theta) \text{ is the Jacobian of } \theta \mapsto \softmax(\theta) \right) 
\end{align}
Note that in one-state MDPs, we have,
\begin{align}
    \frac{d \pi_\theta^\top r}{d \theta} &= \left( \frac{d \pi_\theta }{d \theta} \right)^\top r = H(\pi_\theta) r. \qedhere
\end{align}
\end{proof}

\begin{lemma}
\label{lem:policy_gradient_norm_softmax}
Softmax policy gradient norm is
\begin{align}
    \left\| \frac{\partial V^{\pi_\theta}(\mu)}{\partial \theta }\right\|_2 = \frac{1}{1-\gamma} \cdot \left[ \sum_{s}{d_{\mu}^{\pi_\theta}(s)^2 \cdot \left\| H(\pi_\theta(\cdot | s)) Q^{\pi_\theta}(s,\cdot) \right\|_2^2 } \right]^\frac{1}{2}.
\end{align}
\end{lemma}
\begin{proof}
We have,
\begin{align}
    \left\| \frac{\partial V^{\pi_\theta}(\mu)}{\partial \theta }\right\|_2 &= \left[ \sum_{s,a} \left( \frac{\partial V^{\pi_\theta}(\mu)}{\partial \theta(s,a)} \right)^2 \right]^{\frac{1}{2}} \\
    &= \left[ \sum_{s}{\left\| \frac{\partial V^{{\pi_\theta}}(\mu) }{\partial \theta(s, \cdot)} \right\|_2^2 } \right]^\frac{1}{2} \\
    &= \frac{1}{1-\gamma} \cdot \left[ \sum_{s}{d_{\mu}^{\pi_\theta}(s)^2 \cdot \left\| H(\pi_\theta(\cdot | s)) Q^{\pi_\theta}(s,\cdot) \right\|_2^2 } \right]^\frac{1}{2}. \qquad \left( \text{by \cref{lem:policy_gradient_softmax}} \right) \qedhere
\end{align}
\end{proof}

\begin{lemma}[Performance difference lemma \citep{kakade2002approximately}]
\label{lem:performance_difference_general}
For any policies $\pi$ and $\pi^\prime$,
\begin{align}
    V^{\pi^\prime}(\rho) - V^{\pi}(\rho) &= \frac{1}{1 - \gamma} \sum_{s}{ d_\rho^{\pi^\prime}(s) \sum_{a}{ \left( \pi^\prime(a | s) - \pi(a | s) \right) \cdot Q^{\pi}(s,a) } }\\
    &= \frac{1}{1 - \gamma} \sum_{s}{ d_{\rho}^{\pi^\prime}(s)  \sum_{a}{ \pi^\prime(a | s) \cdot A^{\pi}(s, a) } }.
\end{align}
\end{lemma}
\begin{proof}
According to the definition of value function,
\begin{align}
    V^{\pi^\prime}(s) - V^{\pi}(s) &= \sum_{a}{ \pi^\prime(a | s) \cdot Q^{\pi^\prime}(s,a) } - \sum_{a}{ \pi(a | s) \cdot Q^{\pi}(s,a) } \\
    &= \sum_{a}{ \pi^\prime(a | s) \cdot \left( Q^{\pi^\prime}(s,a) - Q^{\pi}(s,a) \right) } + \sum_{a}{ \left( \pi^\prime(a | s) - \pi(a | s) \right) \cdot Q^{\pi}(s,a) } \\
    &= \sum_{a}{ \left( \pi^\prime(a | s) - \pi(a | s) \right) \cdot Q^{\pi}(s,a) } + \gamma \sum_{a}{ \pi^\prime(a | s) \sum_{s^\prime}{  \gP( s^\prime | s, a) \cdot \left[ V^{\pi^\prime}(s^\prime) -  V^{\pi}(s^\prime)  \right] } } \\
    &= \frac{1}{1 - \gamma} \sum_{s^\prime}{ d_{s}^{\pi^\prime}(s^\prime) \sum_{a^\prime}{ \left( \pi^\prime(a^\prime | s^\prime) - \pi(a^\prime | s^\prime) \right) \cdot Q^{\pi}(s^\prime, a^\prime) }  } \\
    &= \frac{1}{1 - \gamma} \sum_{s^\prime}{ d_{s}^{\pi^\prime}(s^\prime) \sum_{a^\prime}{ \pi^\prime(a^\prime | s^\prime) \cdot \left( Q^{\pi}(s^\prime, a^\prime) - V^{\pi}(s^\prime) \right) }  } \\
    &= \frac{1}{1 - \gamma} \sum_{s^\prime}{ d_{s}^{\pi^\prime}(s^\prime)  \sum_{a^\prime}{ \pi^\prime(a^\prime | s^\prime) \cdot A^{\pi}(s^\prime, a^\prime) } }. \qedhere
\end{align}
\end{proof}

\begin{lemma}[Value sub-optimality lemma]
\label{lem:value_suboptimality}
For any policy $\pi$,
\begin{align}
    V^*(\rho) - V^{\pi}(\rho) = \frac{1}{1 - \gamma} \sum_{s}{ d_\rho^{\pi}(s) \sum_{a}{ \left( \pi^*(a | s) - \pi(a | s) \right) \cdot Q^*(s,a) } }.
\end{align}
\end{lemma}
\begin{proof}
See the proof in \citep[Lemma 21]{mei2020global}. We include a proof for completeness.

We denote $V^*(s) \coloneqq V^{\pi^*}(s)$ and $Q^*(s,a) \coloneqq Q^{\pi^*}(s,a)$ for conciseness. We have, for any policy $\pi$,
\begin{align}
    V^*(s) - V^{\pi}(s) &= \sum_{a}{ \pi^*(a | s) \cdot Q^*(s,a) } - \sum_{a}{ \pi(a | s) \cdot Q^{\pi}(s,a) } \\
    &= \sum_{a}{ \left( \pi^*(a | s) - \pi(a | s) \right) \cdot Q^*(s,a) } + \sum_{a}{ \pi(a | s) \cdot \left(  Q^*(s,a) - Q^{\pi}(s,a) \right) } \\
    &= \sum_{a}{ \left( \pi^*(a | s) - \pi(a | s) \right) \cdot Q^*(s,a) } + \gamma \sum_{a}{ \pi(a | s) \sum_{s^\prime}{  \gP( s^\prime | s, a) \cdot \left[ V^{\pi^*}(s^\prime) -  V^{\pi}(s^\prime)  \right] } } \\
    &= \frac{1}{1 - \gamma} \sum_{s^\prime}{ d_{s}^{\pi}(s^\prime) \sum_{a^\prime}{ \left( \pi^*(a^\prime | s^\prime) - \pi(a^\prime | s^\prime) \right) \cdot Q^*(s^\prime, a^\prime) }  }. \qedhere
\end{align}
\end{proof}

\section{Non-convex (Non-concave) Examples for N\L{} Inequality}
\label{sec:examples_non_convex_nl_inequality}

We list some non-convex (or non-concave in maximization problems) functions which satisfy N\L{} inequalities here from literature. See corresponding references for details.

\paragraph{Expected reward, softmax parameterization.} 
As shown in \cref{lem:non_uniform_lojasiewicz_softmax_special} and \citet[Lemma 3]{mei2020global},
\begin{align}
    \left\| \frac{d \pi_\theta^\top r}{d \theta} \right\|_2 \ge \pi_\theta(a^*) \cdot ( \pi^* - \pi_\theta )^\top r.
\end{align}

\paragraph{Value function, softmax parameterization.} As shown in \cref{lem:non_uniform_lojasiewicz_softmax_general} and \citet[Lemma 8]{mei2020global},
\begin{align}
    \left\| \frac{\partial V^{\pi_\theta}(\mu)}{\partial \theta }\right\|_2 \ge \frac{ \min_s{ \pi_\theta(a^*(s)|s) } }{ \sqrt{S} \cdot  \left\| d_{\rho}^{\pi^*} / d_{\mu}^{\pi_\theta} \right\|_\infty } \cdot \left[ V^*(\rho) - V^{\pi_\theta}(\rho) \right].
\end{align}

\paragraph{Entropy regularized expected reward, softmax parameterization.}  As shown in \citet[Proposition 5]{mei2020global},
\begin{align}
    \left\| \frac{ d \{ \pi_\theta^\top ( r - \tau \log{\pi_\theta}) \} }{d \theta} \right\|_2 \ge \sqrt{2 \tau} \cdot \min_{a}{ \pi_\theta(a) } \cdot \left[ { \pi_\tau^* }^\top \left( r - \tau \log{ \pi_\tau^* } \right)  - \pi_{\theta}^\top \left( r - \tau \log{ \pi_{\theta} } \right) \right]^{\frac{1}{2}}.
\end{align}

\paragraph{Entropy regularized value function, softmax parameterization.} As shown in \citet[Lemma 15]{mei2020global},
\begin{align}
    \left\| \frac{\partial \tilde{V}^{{\pi_\theta}}(\mu) }{\partial \theta} \right\|_2 \ge \frac{\sqrt{2 \tau}}{\sqrt{S}} \cdot \min_{s}{\sqrt{ \mu(s) } } \cdot \min_{s,a}{ \pi_\theta(a | s)  } \cdot \left\| \frac{d_{\rho}^{\pi_\tau^*} }{ d_{\mu}^{\pi_\theta}} \right\|_\infty^{-\frac{1}{2}} \cdot \left[ \tilde{V}^{\pi_\tau^*}(\rho) - \tilde{V}^{{\pi_\theta}}(\rho) \right]^{\frac{1}{2}}.
\end{align}

\paragraph{Expected reward, escort parameterization.}
As shown in \citet[Lemma 3]{mei2020escaping},
\begin{align}
    \left\| \frac{d \pi_\theta^\top r}{d \theta } \right\|_2 \ge \frac{p}{\| \theta \|_p} \cdot \pi_\theta(a^*)^{1 - 1/p} \cdot ( \pi^* - \pi_\theta )^\top r.
\end{align}

\paragraph{Value function, escort parameterization.}
As shown in \citet[Lemma 7]{mei2020escaping},
\begin{align}
    \left\| \frac{\partial V^{\pi_\theta}(\mu)}{\partial \theta }\right\|_2 \ge \frac{p}{\sqrt{S}} \cdot \left\| \frac{d_{\rho}^{\pi^*}}{d_{\mu}^{\pi_\theta}} \right\|_\infty^{-1} \cdot \frac{ \min_s{ \pi_\theta(a^*(s)|s)^{1-1/p} } }{ \max_{s}  \left\| \theta(s, \cdot) \right\|_p } \cdot \left[ V^*(\rho) - V^{\pi_\theta}(\rho) \right].
\end{align}

\paragraph{Entropy regularized value function, escort parameterization.} 
As shown in \citet[Lemma 12]{mei2020escaping},
\begin{align}
    \left\| \frac{\partial \tilde{V}^{{\pi_\theta}}(\mu) }{\partial \theta} \right\|_2 \ge \frac{p \cdot \sqrt{2 \tau}}{\sqrt{S}} \cdot \min_{s}{\sqrt{ \mu(s) } } \cdot  \frac{ \min_{s,a}{\pi_\theta(a | s)^{1-1/p}} }{ \max_{s} \| \theta(s, \cdot ) \|_p} \cdot \left\| \frac{d_{\rho}^{\pi_\tau^*} }{ d_{\mu}^{\pi_\theta}} \right\|_\infty^{-\frac{1}{2}} \cdot \left[ \tilde{V}^{\pi_\tau^*}(\rho) - \tilde{V}^{{\pi_\theta}}(\rho) \right]^\frac{1}{2}.
\end{align}

\paragraph{Cross entropy, escort parameterization.} As shown in \citet[Lemma 17]{mei2020escaping},
\begin{align}
    \left\| \frac{d \{ \KL(y \| \pi_\theta) \} }{d \theta} \right\|_2 \ge \frac{p}{\| \theta \|_p} \cdot \min_{a}{\pi_\theta(a)}^{\frac{1}{2}-\frac{1}{p}} \cdot \KL(y \| \pi_\theta)^\frac{1}{2}.
\end{align}

\paragraph{Generalized linear models, sigmoid activation, mean squared error.} As shown in \cref{lem:non_uniform_lojasiewicz_glm_sigmoid_realizable},
\begin{align}
    \left\| \frac{\partial \gL(\theta)}{\partial \theta} \right\|_2 \ge 8 \cdot u(\theta) \cdot \min\left\{ u(\theta), v \right\} \cdot \sqrt{\lambda_\phi} \cdot \left[ \frac{1}{N} \cdot \sum_{i=1}^{N}{ \left( \pi_i - \pi_i^* \right)^2 } \right]^{\frac{1}{2}}.
\end{align}

\section{Additional Simulation Results}
\label{sec:additional_simulations_experiments}

\subsection{ $f: x \mapsto | x |^p,$ $p \in (1, 2)$}

As shown in \cref{prop:absulte_power_p}, with $p \in (1, 2)$, $f: x \mapsto | x |^p$ satisfies N\L{} inequality with $\xi = 1/p \in (1/2, 1)$, which is the case (3) in \cref{thm:general_optimization_main_result_1}. The function $f$ is differentiable, and the Hessian $\left|f^{\prime\prime}(x) \right| = p \cdot (p-1)\cdot \left|  x \right|^{p-2} \to \infty$, as $x \to 0$, which indicates GD with $\eta \in \Theta(1)$ does not converge.
\begin{figure*}[ht]
\centering
\includegraphics[width=0.8\linewidth]{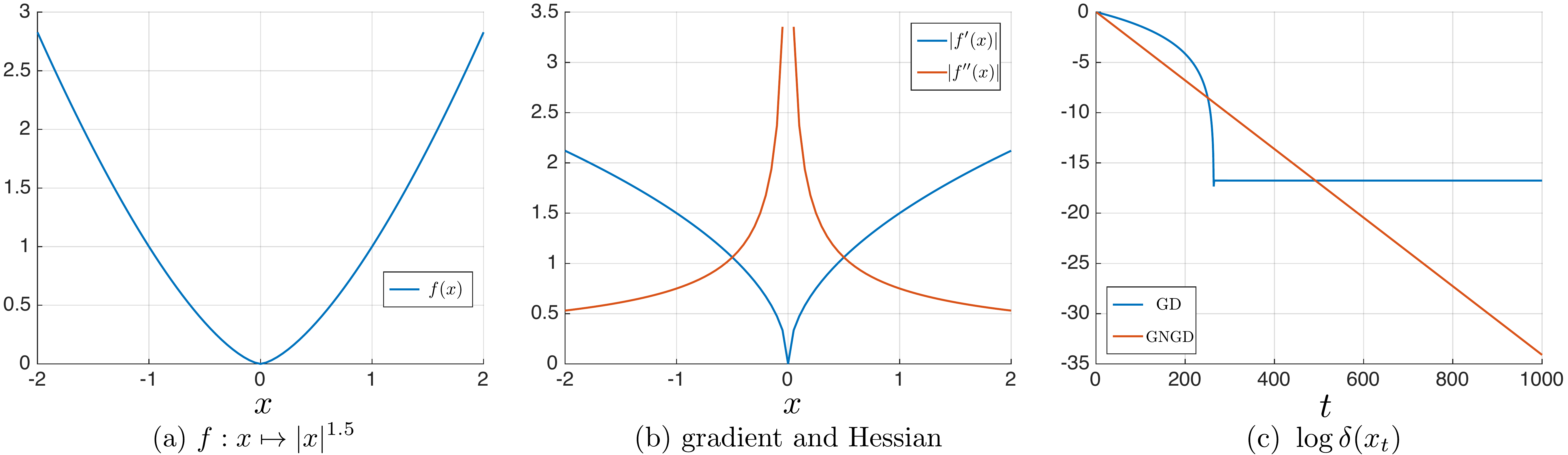}
\caption{GD and GNGD on $f: x \mapsto | x |^p$, $p = 1.5$.
} 
\label{fig:example_power_1p5}
\end{figure*}

\cref{fig:example_power_1p5}(a) shows the image of $f: x \mapsto | x |^{1.5}$. As shown in subfigure (b), the gradient of $f$ exists at $x = 0$, and the Hessian $\left| f^{\prime\prime}(x) \right| \to \infty$ as $x \to 0$. The results of GD with $\eta = 0.005$ and GNGD are presented in subfigure (c). The sub-optimality of GD update decreased for some time, and then it increased later. This is due to the Hessian is unbounded near $x = 0$, and thus constant learning rates cannot guarantee monotonic progresses for GD. On the other hand, GNGD with $\eta = 0.01$ enjoys $O(e^{- c \cdot t})$ convergence rate, verifying the results in the case (3) in \cref{thm:general_optimization_main_result_1}.

\subsection{$f: x \mapsto | x |^p,$ $p > 2$}

As shown in \cref{prop:absulte_power_p}, with $p \in (1, 2)$, $f: x \mapsto | x |^p$ satisfies N\L{} inequality with $\xi = 1/p \in (0, 1/2)$. As shown in \cref{fig:example_power_4}(a), the spectral radius of Hessian approaches $0$ as $x \to 0$, which is the case (1) in \cref{thm:general_optimization_main_result_1}.
\begin{figure}[ht]
\centering
\includegraphics[width=0.8\linewidth]{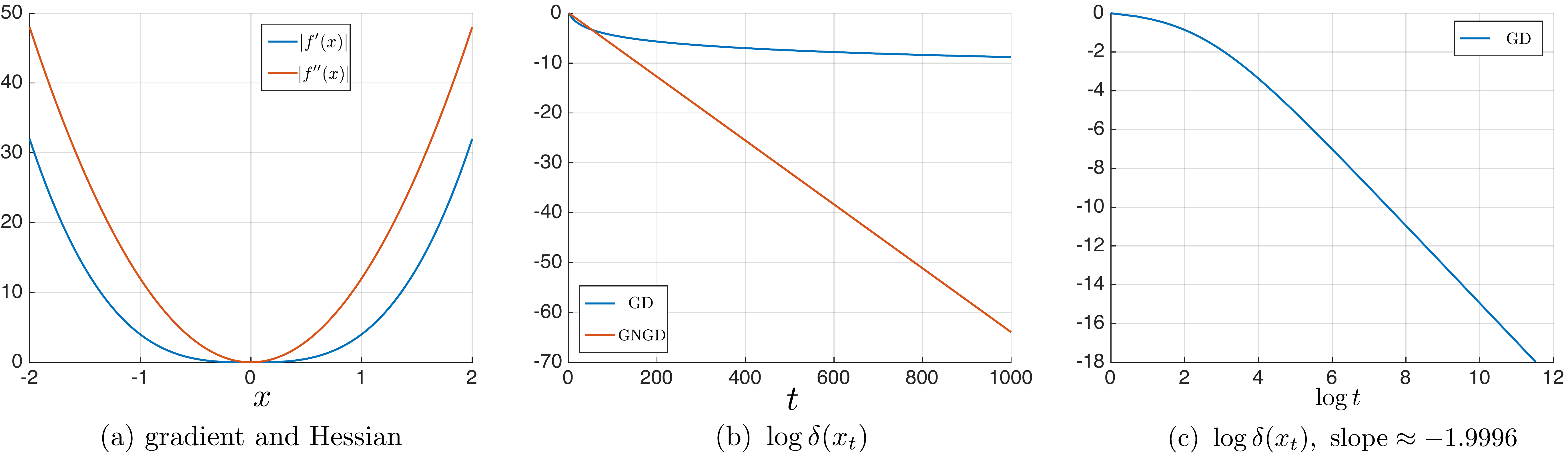}
\caption{GD and GNGD on $f: x \mapsto | x |^p$, $p = 4$.} 
\label{fig:example_power_4}
\end{figure}

Subfigure (c) shows that the standard GD with constant learning rate $\eta = 0.01$ achieves sublinear rate about $O(1/t^2)$, while subfigure (b) shows that GNGD with $\eta = 0.01$ enjoys linear rate $O(e^{-c \cdot t})$, verifying \cref{thm:general_optimization_main_result_1}.

\subsection{Convergence Rates on GLM}

\cref{thm:final_rates_normalized_glm_sigmoid_realizable} proves linear convergence rates $O(e^{-c \cdot t})$ for both GD and GNGD on GLM. We compare GD, NGD \citep{hazan2015beyond}, and GNGD on GLM, as shown in \cref{fig:glm_gd_ngd_gngd_convergence_rates}.
\begin{figure*}[ht]
\centering
\includegraphics[width=0.8\linewidth]{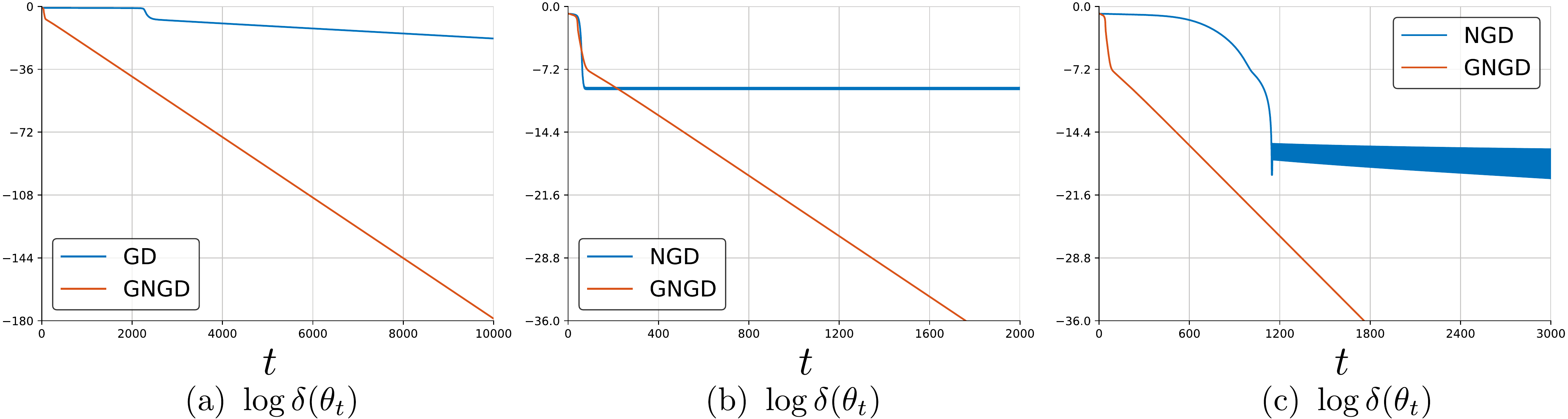}
\caption{Convergence rates for GD, NGD, and GNGD on GLM.
} 
\label{fig:glm_gd_ngd_gngd_convergence_rates}
\end{figure*}

Subfigure (a) presents the results of GD with $\eta = 0.09$ and GNGD with $\eta = 0.09$. Both GD and GNGD achieve linear $O(e^{-c \cdot t})$ rates, verifying \cref{thm:final_rates_normalized_glm_sigmoid_realizable}. GD suffers from the plateaus at the early-stage optimization, which is consistent with \cref{fig:example_glm_sigmoid} and the explanations after \cref{thm:final_rates_normalized_glm_sigmoid_realizable}. On the other hand, the slopes indicate that GNGD converges strictly faster than GD, which justifies the constant dependences ($C \ge C^2$) in \cref{thm:final_rates_normalized_glm_sigmoid_realizable}. Subfigure (b) shows that standard NGD \citep{hazan2015beyond} with constant learning rate $\eta  = 0.09$ does not converge. The NGD update keeps oscillating, which verifies our argument of using standard normalization for all $t \ge 1$ is not a good idea. Subfigure (c) presents the NGD using adaptive learning rate $\eta_t = \frac{0.09}{\sqrt{t}}$, which has faster convergence than NGD with constant $\eta$. However, GNGD still significantly outperforms NGD with $\eta_t = \frac{0.09}{\sqrt{t}}$, verifying the $O(e^{-c \cdot t})$ in \cref{thm:final_rates_normalized_glm_sigmoid_realizable} and $O(1/\sqrt{t})$ in \cref{thm:ngd_glm_sigmoid_convergence_rate}.

\subsection{Tree MDPs}

\cref{fig:synthetic_tree_pg_gnpg} shows the results for
PG and GNPG beyond one-state MDPs. The environment is a synthetic tree with height $h$ and branching
factor $b$. The total number of states is
\begin{align}
    S = \sum_{i=0}^{h-1}{b^i}.
\end{align}
The discount factor $\gamma = 0.99$, and we set $\mu = \rho$ (e.g., in \cref{alg:normalized_policy_gradient_softmax} and \cref{thm:final_rates_normalized_softmax_pg_general}), where $\rho(s_0) = 1$ for the root state $s_0$. For PG, in each iteration, we calculate the policy gradient (\cref{lem:policy_gradient_softmax}) to do one update. For GNPG, \cref{alg:normalized_policy_gradient_softmax} is used.

Subfigures (a) and (b) show the results for $h = b = 4$, and $S = 85$. The learning rate is $\eta = 0.02$ for PG and GNPG. Subfigures (c) and (d) show the results for $h = 5$ and $b = 4$, and $S = 341$. The learning rate is $\eta = 0.05$ for PG and GNPG. 
\begin{figure*}[ht]
\centering
\includegraphics[width=1.0\linewidth]{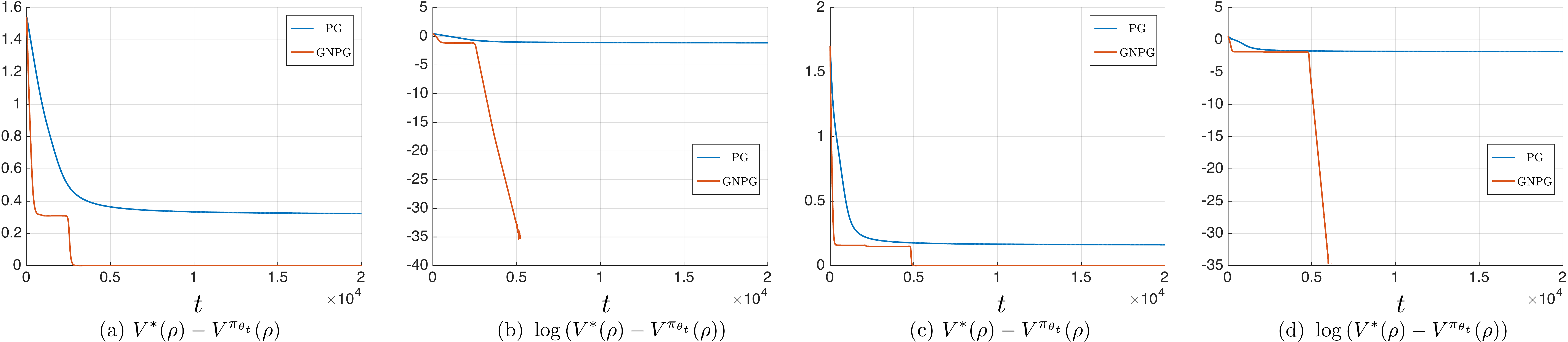}
\caption{Results for PG and GNPG on tree MDPs. In (a) and (b), $S = 85$. In (c) and (d), $S = 341$.
} 
\label{fig:synthetic_tree_pg_gnpg}
\end{figure*}

\end{document}